\documentclass[journal]{IEEEtran}

\usepackage[margin=1.18in]{geometry}

\usepackage{wrapfig}
\usepackage{xr-hyper}
\usepackage[utf8]{inputenc} 
\usepackage[T1]{fontenc}    
\usepackage{url}            
\usepackage{booktabs}       
\usepackage{amsfonts}       
\usepackage{nicefrac}       
\usepackage{microtype}      

\usepackage{comment}

\usepackage{algorithm}
\usepackage{listings}
\usepackage{algorithmic}
\usepackage{multirow}

\usepackage{amsmath,amsfonts,amssymb,mathtools,url,enumerate,verbatim}
\usepackage{amsthm}
\usepackage[german,french,british]{babel}
\usepackage{ucs}
\usepackage{times}
\usepackage{pgf,tikz}
\usepackage{float}
\usepackage{mathrsfs}
\usepackage{mathtools}
\usepackage{multirow}
\usepackage{yfonts}
\usepackage{caption}
\usepackage{subcaption}
\usetikzlibrary{arrows}

\definecolor{xdxdff}{rgb}{0.49019607843137253,0.49019607843137253,1.}
\definecolor{qqqqff}{rgb}{0.,0.,1.}
\definecolor{wwzzff}{rgb}{0.4,0.6,1.}
\definecolor{xdxdff}{rgb}{0.49019607843137253,0.49019607843137253,1.}

\DeclareMathOperator*{\argmin}{arg\,min}

\usepackage{xr}

\DeclareMathOperator{\Tr}{Tr}

\DeclareMathOperator{\rad}{\mathfrak{R}}
\DeclareMathOperator{\norm}{normalize}

\DeclareMathOperator{\Fr}{Fr}

\allowdisplaybreaks

\newtheorem{theorem}{Theorem}[section]
\newtheorem{thm}{Theorem}[section] \newtheorem{lemma}[thm]{Lemma}
\newtheorem{corollary}[thm]{Corollary} 
\newtheorem{proposition}[thm]{Proposition}

\newtheorem{definition}[thm]{Definition}
\newtheorem{remark}[thm]{Remark}

\newcommand\E{\mathbb{E}}

\newcommand\R{\mathbb{R}}
\renewcommand\P{\mathbb{P}}

    \newcommand\newdot{{\kern.8pt\cdot\kern.8pt}}
\newcommand\nbull{{\kern.8pt\raise1.5pt\hbox{\small\bf .}\kern.8pt}}
\newcommand\1{\hbox{\kern.375em\vrule height1.57ex depth-.1ex
		width.05em\kern-.375em \rm 1}}

\usepackage{comment}
\usepackage{amsthm}
\usepackage[german,french,british]{babel}
\usepackage{ucs}
\usepackage{times}
\usepackage{mathabx}

\usepackage{pgf,tikz}
\usepackage{float}
\usepackage{mathrsfs}
\usepackage{mathtools}
\usepackage{multirow}
\usepackage{yfonts}
\usepackage{caption}
\usepackage{subcaption}
\usepackage{xr}

\usetikzlibrary{arrows}

\usepackage{authblk}
\definecolor{xdxdff}{rgb}{0.49019607843137253,0.49019607843137253,1.}
\definecolor{qqqqff}{rgb}{0.,0.,1.}
\definecolor{wwzzff}{rgb}{0.4,0.6,1.}
\definecolor{xdxdff}{rgb}{0.49019607843137253,0.49019607843137253,1.}

\DeclareMathOperator{\old}{old}

\DeclareMathOperator{\new}{new}

\DeclareMathOperator{\minimize}{min}

\DeclareMathOperator{\rank}{rank}
\allowdisplaybreaks

\DeclareMathOperator{\spn}{span}

\DeclareMathOperator{\Id}{Id} 
 \DeclareMathOperator{\diag}{diag}

\DeclareMathOperator{\RMSE}{RMSE}
\DeclareMathOperator{\MBD}{MBD}
\DeclareMathOperator{\IBD}{IBD}
\DeclareMathOperator{\UBD}{UBD}
\DeclareMathOperator{\SPC}{SPC}

\usepackage{hyperref}

\usepackage{cleveref}
\usepackage{stfloats}

\usepackage{hyperref}

\ifCLASSINFOpdf
\else
\fi
\usepackage{comment}
\begin{document}
%

\title{Orthogonal Inductive Matrix Completion}

\author{Antoine Ledent*, Rodrigo Alves*, and Marius Kloft, \textit{Senior Member, IEEE}
	\thanks{*The first two authors contributed equally to this work}
	\thanks{The authors are with TU Kaiserslautern}
	\thanks{\{ledent,alves,kloft\}@cs.uni-kl.de}
	\thanks{Code repository: http://github.com/rasalves/OMIC}
	\thanks{Accepted for publication in Transactions of Neural Networks and Learning Systems (TNNLS), DOI: 10.1109/TNNLS.2021.3106155}}


\maketitle

\begin{abstract}
	We propose orthogonal inductive matrix completion (OMIC), an interpretable 
	approach to matrix completion based on a sum of multiple orthonormal side information terms, together with nuclear-norm regularization. 
	The approach allows us to inject prior knowledge about the singular vectors of the ground truth matrix. 
	We optimize the approach by a provably converging algorithm, which optimizes all components of the model simultaneously. 
We study the generalization capabilities of our method in both the distribution-free setting and in the case where the sampling distribution admits uniform marginals, yielding learning guarantees that improve with the quality of the injected knowledge in both cases. 
As particular cases of our framework, we present models which can incorporate user and item biases or community information in a joint and additive fashion. 
	We analyse the performance of OMIC on several synthetic and real datasets. 
	On synthetic datasets with a sliding scale of user bias relevance, we show that OMIC better adapts to different regimes than other methods. On real-life datasets containing user/items recommendations and relevant side information, we find that OMIC surpasses the state-of-the-art, with the added benefit of greater interpretability. 
\end{abstract}

\section{Introduction and related work}

Matrix completion, the problem of recovering the missing entries of a partially observed matrix,
has found application in a wide range of domains.
As examples consider the following.
(1) A streaming provider recommends movies to its users, based on an incomplete database of user-movie ratings. 
(2) A social network wants to find missing links in their friendship network.
(3) A chemical producer wants to predict interactions of chemical compounds from a subset of known pairwise interactions.
These examples---from the domains of recommender systems~\cite{recomenderMC,MCforRS}, social network analysis~\cite{SocialNet}, and chemical engineering~\cite{jirasek2020machine}---highlight the wide range of applications of matrix competition. 
For simplicity, we use movie recommendation as running example here, so the data consists of user-movie ratings.
It should be clear that, more generally, we can work with type1-type2 pairs, depending on the application, e.g.,
user-book, user-user, compound-compound, etc.

To recover missing entries of a matrix, it is necessary to make an assumption on the structure of the ground truth matrix. 
The most common assumption is that the matrix is of low rank.
However, optimally approximating the observed entries whilst minimizing the rank is NP-hard. The SoftImpute algorithm~\cite{softimpute} bypasses this difficulty by using the nuclear norm as a regularizer.
Not only does SoftImpute work well in practice, it also enjoys favorable theoretical guarantees: only a small number of known entries is required to recover the underlying low-rank matrix exactly~\cite{genius,CandesRecht} or approximately~\cite{noisy,Kolchinski} from noisy entries. 

In practical applications, the following refinements may help the  performance of classic recommender systems.
(1) \emph{Incorporating bias} \cite{book,netflixwin}. 
Some users may generally be more critical than others. This means they tend to give lower ratings than other users, regardless of the movie. Moreover, some movies are intrinsically better than others, so they receive more favorable ratings. Previous work incorporated user and movie bias in a pre-processing step, and then trained matrix completion on the residuals.
(2) \emph{Incorporating side information}. For movies, we find plenty of side information (genre, staring actors, director, reviews, etc.) on the web, and we might have access to user attributes (age, gender, etc.), from which we can derive clusters of users (community information). Inductive matrix completion (IMC) \cite{IMC} uses such side information 
to guide the prediction of the user-movie ratings. 
IMC, which is backed up by well-developed learning theory~\cite{IMC,IMC1,IMC2,IMC3,PIMC}, can be applied also to new users with no ratings, but for which side information is given.

\begin{figure*}[!b]
	\centering
	\includegraphics[width=.98\linewidth]{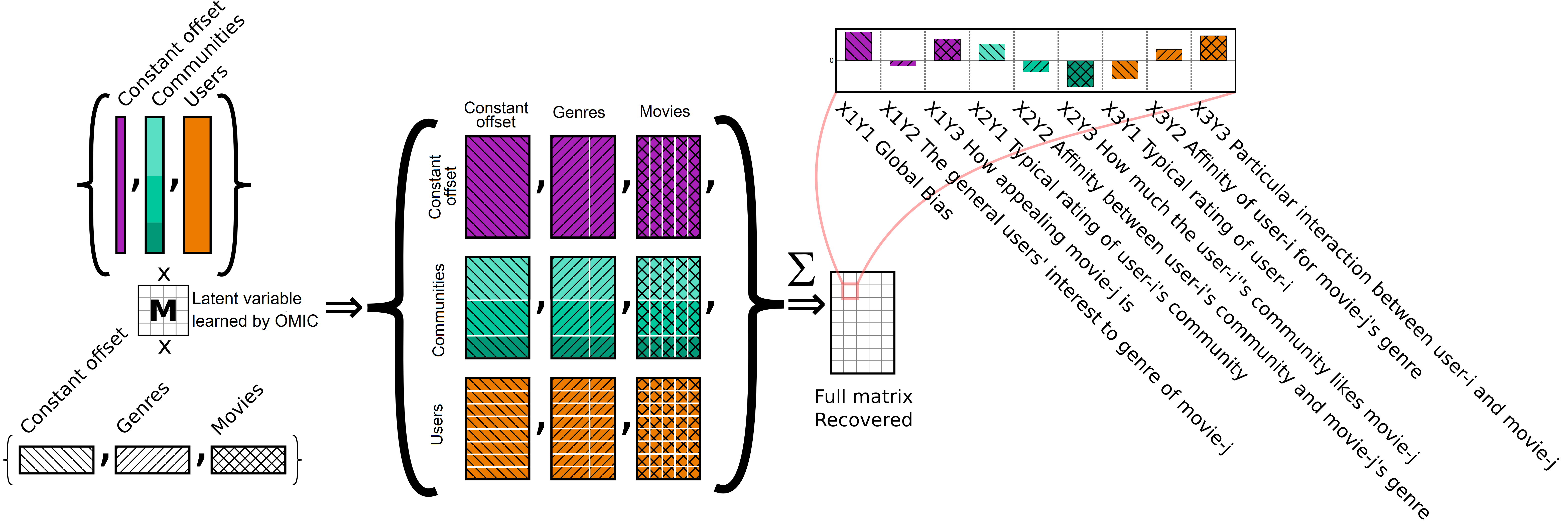}
	\caption{Visualization of orthogonal inductive matrix completion (specifically, the model choice BOMIC+ explained in Section~\ref{bomic+}). The model is a sum of matrix terms, each of the form $XMY^\top$. This means each combination of $X$ and $Y$ gives rise to a term in the sum. We interpret the magnitude of this term as its relevance to the prediction. }
	\label{fig:mainfig}
\end{figure*}

 Our aim in this work is to create a generic model that can incorporate all the improvements mentioned above into a single flexible framework with a well-principled optimization procedure. To the best of our knowledge, the only work which attempted to incorporate some of the above improvements into a single jointly trained model is~\cite{mostrelated}, which considers outputs of the form
 \begin{align}
\label{seventeen}
f_{i,j} = \mathbf{x}_i^\top M\mathbf{y}_j + z_{i,j},
\end{align}

\noindent with nuclear-norm regularization imposed on both $Z$ and $M$. The model is trained with gradient descent. The incorporation of both the standard low-rank term $Z$ and the IMC term $XMY^\top$ allows the model to capture both generic low-rank phenomena, as well as any behavior related to the side information. The hyperparameters involved in the nuclear-norm constraints can be optimized through cross-validation and allow the model to decide how relevant the side information is.
However, a single given matrix $f_{\nbull,\nbull}$ can correspond to several possible choices of $M$ and $Z$, thereby limiting the  interpretability of the model and the individual terms of the sum~\eqref{seventeen}. Furthermore, the model does not capture user and item biases. \\ Our model remedies these failings. Firstly, the corresponding predictors can take the following form as a particular case: 
\begin{align}
\label{simplebiasfirst}
f_{i,j} = c+ u_i + m_j + \mathbf{x}_i^\top M\mathbf{y}_j + z_{i,j},
\end{align}
where $c$ is an unknown constant (corresponding to a global bias of the model), ${u}_i$ and ${m}_j$ are the user bias and movie quality terms (i.e. free parameters which are trained based on the sole constraint that the user bias term ${u}_i$ can only depend on the user $i$ (but not on the item $j$) and vice versa), $\mathbf{x}_i$ and $\mathbf{y}_j$ are the known side information vectors of the $i$-th user and $j$-th movie, whilst $M$ and $Z=(z_{i,j})$ are parameter matrices to which \textit{nuclear norm regularization} is applied. The nuclear norm of a matrix is defined as the sum of its singular values. Regularizing the nuclear norm has a rank-sparsity inducing effect similar to the sparsity inducing effect of LASSO. Thus, our regularizer (introduced formally in equation~\eqref{theopt} below) \textit{both} indirectly encourages $M$, and therefore the term corresponding to be $\mathbf{x}_i^\top M\mathbf{y}_j$ to be low rank (whilst relying on the side information for prediction), \textit{and} encourages the residual term $Z$ to have low rank. In summary, we are able to model biases, side information terms, as well as residual generic low-rank effects in a single, jointly trained model.

Furthermore, we will impose orthogonality constraints that effectively require each term in the sum in \eqref{simplebiasfirst} to live in separate, mutually orthogonal subspaces. 
This has two advantages. 
First, training can be performed for all components simultaneously, as we will show. 
Second, the variables in \eqref{simplebiasfirst} admit interpretation. This is because any ground truth matrix can be represented \emph{uniquely} (thanks to the orthogonality conditions) through \eqref{simplebiasfirst}. We thus interpret the magnitude of the summands in \eqref{simplebiasfirst} as their relevance to the model. For instance, this implies that the user-movie match $Z$ is free of any behavior that could be interpreted as user bias or movie quality.

We note that several specific variations of our model are possible depending on whether or not side information is present, how many distinct types of side information are present, whether or not we want to include user/item biases, etc. Therefore, we rely on a \textit{unified formal framework} to describe all possible variations of these ideas as follows. 
Our general model's predictors have the following form:
\begin{align}
\label{thesume}
F=(f_{i,j})=\sum_{k=1}^K\sum_{l=1}^L X^{(k)}M^{(k,l)}(Y^{(l)})^\top.
\end{align}
Here $X^{(1)},\ldots,X^{(K)}$ and $Y^{(1)},\ldots,Y^{(L)}$ are so-called "auxiliary matrices", which define the specific model choice with respect to the incorporation (or lack thereof) of biases, side information etc. We now observe that the terms of equation~\eqref{simplebiasfirst} can all be written in the form $X^{(k)}M^{(k,l)}(Y^{(l)})^\top$ for some suitably defined $X^{(k)}, Y^{(l)}$. For instance, if we set $X$ as the identity matrix and $Y$ as a column matrix of all $1$'s, then the matrix $XMY^\top$ has the user biases as entries: $\mathbf{x}_iM\mathbf{y}_j^\top=u_{i}$ for all $i,j$. If we set both $X$  and $Y$ as identity matrices, we obtain the specific user-movie match $\mathbf{x}_iM\mathbf{y}_j^\top=z_{i,j}$ for all $i,j$. In Figure~\ref{fig:mainfig}, we illustrate an example with $K=L=3$ which takes into account user and item biases, as well as side information in the form of partitions of users and items into communities (with the movie communities being genres). This representative example is described in more technical detail in Subsection~\ref{bomic+}, where we also further explain how equation~\eqref{simplebiasfirst} can be fully incorporated as a particular case of~\eqref{thesume}. 
To ensure the uniqueness of the decomposition~\eqref{thesume} and thus enable interpretability, we require that the \textit{columns} of $(X^{(1)},\ldots,X^{(K)})$ and $(Y^{(1)},\ldots,Y^{(L)})$ form orthonormal bases of their respective spaces.  Hence, we refer to our model as OMIC (Orthogonal Inductive Matrix Completion). Note that the choice of matrices $X^{(k)}$s and $Y^{(l)}$s happens at the level of \textit{model selection} (from a preprocessing of available side information or available insight into the likely structure of the data). Thus, the orthogonality assumption is not an assumption on the ground truth matrix, but a restriction of \textit{model choice}.  For fixed $k,l$, the \textit{rows} of $X^{(k)}$ and $Y^{(l)}$ play a similar role to "side information vectors" in traditional IMC, and are not required to be orthogonal.
The choice of auxiliary matrices influences both interpretability and accuracy: on the one hand, the model's output will come in a form that is already divided into components corresponding to each pair $(X^{(k)},Y^{(l)})$. On the other hand, if the ground truth structure matches the choice of prior directions, accuracy will improve. For instance, as we will discuss further below, choosing $X^{(1)}=(\frac{1}{\sqrt{m}},\ldots, \frac{1}{\sqrt{m}})^\top$ and $X^{(2)}$ to be the complement of $X^{(1)}$ in $\R^{m}$ corresponds to the assumption that \textit{item biases} are important: this will both (1) \textit{increase accuracy} if item biases indeed play a role in the ground truth; and (2) in any case \textit{direct interpretability} in the direction of the disentanglement of item biases from other low-rank effects in the solution given by the algorithm. If we are given raw side information matrices $X,Y$, it also makes sense to create an instance of our model where $X^{(1)}$ and $Y^{(1)}$ are orthogonalized versions of $X$ and $Y$~\footnote{And $X^{(2)}$ and $Y^{(2)}$ are chosen to satisfy the orthogonality conditions}. In that case we obtain an IMC-type model where the solution consists of four terms depending on whether side information was used for users and for items.

Three specific choices of auxiliary matrices $X^{(k)}$s and $Y^{(l)}$s yield especially interesting models with favourable properties in terms optimization and interetability. We develop them in greater detail and we refer to the corresponding models as BOMIC, OMIC+ and BOMIC+.

Our contributions can be summarized as follows:

\begin{enumerate}
	\item We propose orthogonal inductive matrix completion (OMIC), a class of learning algorithms for matrix completion, which give rise to interpretable solutions, for instance, in terms of the amount of user bias, movie quality, and community effects in a learned matrix.
	\item We propose an efficient optimization algorithm. Furthermore, we prove its convergence and give upper bounds on its convergence rate. See Theorems~\ref{convergence} and~\ref{rate}.
	\item  We present in more detail \textit{three especially interesting models}, \textbf{BOMIC}, \textbf{OMIC+},  and \textbf{BOMIC+}, which are particular cases of our framework. 
For all three cases, we provide a scalable implementation (explained in Algorithms~\ref{OMICimpute} and~\ref{BOMICLUSTERS2}) of our algorithm that allows us to work on large datasets.
	\item For the three most relevant models mentioned above, we prove generalisation bounds both in the case of a sampling distribution with uniform marginals, and in the distribution-free case (where we make no assumption on the ground truth distribution). The better the model matches the ground truth, the tighter the bounds. 
	\item Our empirical analysis shows that OMIC exhibits better performance and flexibility across the whole spectrum of varying quality of side information. On a large array of real data, OMIC surpasses the state-of-the-art in terms of accuracy, with the added benefit of interpretability.
\end{enumerate}

\subsection{Related work}

The idea of training user biases, either as a preprocessing step or jointly with a model was frequent in the pre-SoftImpute days~\cite{book,netflixwin}. In~\cite{gaillard2015time}, time-dependent user and item biases were incorporated into a maximum margin matrix factorization framework\footnote{A close cousin of nuclear norm minimisation, cf. appendix} with both the biases and the low-rank residuals continuously retrained alternatingly to account for the time variations.

The idea of training a side information term $XMY^\top$ jointly with a residual term $Z$ was expressed in~\cite{mostrelated}, which is the most related paper to ours. We note that only this specific case was studied there, and that no orthogonality/interpretability constraint was imposed, and thus no imputation algorithm was developed (alternating optimization and gradient descent were used instead). Furthermore, the generalization bounds obtained present differences due to the special nature of our side information and auxiliary matrices. 

In~\cite{tassisa18} and~\cite{PIMC}, the authors study a model equivalent to a single inductive matrix completion term with non-orthogonal side information and prove bounds in the uniform sampling regime. None of these works contain either a sum of IMC terms, cross-term orthogonality constraints such as ours, user/item biases, or any distribution-free generalization bounds.

In~\cite{NICE}, the authors introduce an interesting model with similarities (and differences) to both~\cite{mostrelated} and the present work. First, the matrices $X$ and $Y$ are augmented by column vectors of ones resulting in the matrices $\bar{X}=[X,\textbf{1}]$ and $\bar{Y}=[Y,\textbf{1}]$. Predictors then take the form $E=\bar{X}M(\bar{Y})^\top+\Delta$, with nuclear norm regularisation imposed on $E$ and $L^1$ (or nuclear norm) regularisation on $M$, and Frobenius norm regularization imposed on $\Delta$, with the constraint that $P_{\Omega}(E)=R_\Omega$. In~\cite{Rev2.2}, the authors solve an explicit rank minimization problem under linear constraints on the matrix (this problem is now commonly referred to as 'Matrix Regression' (see also~\cite{Rechtmatrixequations})). In~\cite{Rev2.3}, the authors propose a very general optimization framework that encompasses both the matrix regression problem mentioned above and low rank matrix completion, as well one-bit matrix completion. 

\section{Description of the model and optimization procedure} 

We always assume that we have an $m\times n$ matrix $R$ whose entries are partially revealed to us. In this section we assume that the entries are observed without noise for notational simplicity\footnote{Algorithmic aspects are unchanged}, whilst in the theoretical results from the next section, we will deal with the more general case of noisy observations. The set of revealed entries is denoted by $\Omega\subset \{1,2,\ldots,m\}\times  \{1,2,\ldots,n\}$, and $\Omega^\perp$ denotes the complement $\{1,2,\ldots,m\}\times  \{1,2,\ldots,n\}\setminus \Omega$ (i.e., the set of unobserved entries). The projection on the set of matrices with all entries on $\Omega^{\perp}$ being zero is denoted by $P_{\Omega}$. $P_{\Omega^\perp}$ is defined similarly. We denote the matrix of observed entries by $R_\Omega$. 
As explained in the introduction, our model requires choosing some auxiliary matrices $X^{(k)}\in \mathbb{R}^{m\times d^{(1)}_k}$, $Y^{(l)}\in \mathbb{R}^{n\times d^{(2)}_k}$ representing prior knowledge about the problem.
The columns of the auxiliary matrices are assumed to form an orthonormal basis of their respective spaces, i.e. 
\begin{align}
\label{orthogonalityconditions}
    &\sum_{i=1}^m X^{(k_1)}_{i,j_1}X^{(k_2)}_{i,j_2}=\delta_{k_1,k_2}\delta_{j_1,j_2};\nonumber \\
    &\sum_{i=1}^n Y^{(l_1)}_{i,j_1}Y^{(l_2)}_{i,j_2}=\delta_{l_1,l_2}\delta_{j_1,j_2}; \nonumber  \\
    &\spn_{k,j}(X^{(k)}_{\nbull,j})=\mathbb{R}^m \quad \quad  \text{and}\nonumber \\
    &\spn_{l,j}(Y^{(l)}_{\nbull,j})=\mathbb{R}^n.
\end{align}
\vspace{-0,4cm}
\subsection{Orthogonal inductive matrix completion}

The general form of the optimization problem we consider is as follows: 

\small
\begin{align}
\label{theopt}
&\quad \quad \quad \quad \quad \quad \quad \minimize_M \quad \mathcal{L}(\Omega,M,\Lambda)\> \quad \quad \quad \text{with}\nonumber \\ 
&\mathcal{L}(\Omega,M,\Lambda)=\sum_{k=1}^K\sum_{l=1}^L \lambda_{k,l}\|M^{(k,l)}\|_{*}\\&+\frac{1}{2} \sum_{(i,j)\in \Omega } \ell\left[R_{i,j},\left(\sum_{k=1,l=1}^{K,L} X^{(k)}M^{(k,l)}(Y^{(l)})^{\top}\right)_{i,j}\right].\nonumber
\end{align}
\normalsize

Here the $\lambda_{k,l}$s are non-negative tunable hyperparameters. Note that the objective is \textit{convex}, but \textit{not strongly convex}. In particular, any stationary point is a solution, but the solution may not be unique (though all solutions correspond to the same (optimal) value of the objective function). In practical implementations such as the specific algorithm and implementation we present below, $\ell$ is set to the square loss\footnote{If this loss function is used and the matrix is \textit{fully observed}, the problem becomes strongly convex}: $\ell(x,x')=|x-x'|^2$.

We note that our orthogonality constraints offer the additional advantage of interpretability: it is easy to check that the spaces $\left\{  X^{(k)}M (Y^{(l)})^\top   \big| M\in \mathbb{R}^{d_1^{(k)}\times d_2^{(l)}}   \right\}$
are orthogonal with respect to the Frobenius inner product. Thus, each ground truth matrix $R$ has a unique representation 
$R=\sum_{k=1}^K\sum_{l=1}^LX^{(k)}R^{(k,l)}(Y^{(l)})^{\top}$. We provide the detailed proof of these results in Appendix~\ref{Uniqueness_proof} (see Proposition~\ref{prop:unique}). The norms $\left\|X^{(k)}R^{(k,l)}(Y^{(l)})^{\top}\right\|_{\Fr}=\left\|  R^{(k,l)}  \right\|_{\Fr}$ can be interpreted as the relative importance of the auxiliary pairs $X^{(k)},Y^{(l)}$.  Furthermore, the tuning of the coefficients $\lambda_{k,l}$ can be assisted by human knowledge. In particular we can dramatically reduce our cross-validation needs by tying many parameters with each other, as well as by setting other parameters corresponding to easy to learn quantities (such as user/community biases) to zero.

In the next three Sections~\ref{bias},~\ref{omic+} and~\ref{bomic+} below, we present the three most significant instances of our model class, which incorporate user/item biases and/or community side information.
\subsection{First example: jointly trained user/item biases (BOMIC)}
\label{bias}
One notable example of this setting provides a way to train a low-rank matrix completion model together with user biases, as discussed in the introduction: if we set $X^{(1)}=(\frac{1}{\sqrt{m}},\ldots,\frac{1}{\sqrt{m}})^\top$, $Y^{(1)}=(\frac{1}{\sqrt{n}},\frac{1}{\sqrt{n}},\ldots,\frac{1}{\sqrt{n}})^\top$ and set $X^{(2)}$ (resp. $Y^{(2)}$) to be the orthogonal complement of $X^{(1)}$ (resp. $Y^{(1)}$) in $\mathbb{R}^{m}$ (resp. $\mathbb{R}^{n}$), then our model~\eqref{theopt} is equivalent to optimizing a prediction function $f_{i,j}=c+u_i+m_j+Z_{i,j}$ where $Z$ is constrained to have columns and rows summing to zero, and the regularizer is $$\lambda_1|c|+\lambda_2 \|\mathbf{u}\|_{2}+\lambda_3\|\mathbf{m}\|_{2}+\lambda_4 \|Z\|_{*}.$$
One would typically set $\lambda_2=\lambda_3$ and $\lambda_1=0$. 

Here, the terms $X^{(1)}M^{(1,1)}(Y^{(1)})^\top$  and  $X^{(1)}M^{(1,2)}(Y^{(2)})^\top+X^{(2)}M^{(2,1)}(Y^{(1)})^\top$, correspond to a general, matrix-wise bias, and a combination of user/item specific biases respectively. Meanwhile, the term $X^{(2)}M^{(2,2)}(Y^{(2)})^\top$ represents purely bias-free low-rank effects: an entry of $X^{(2)}M^{(2,2)}(Y^{(2)})^\top$ will be large if the item and user are particularly well-fitted to each other, independently of the general behavior of either user or item. This can be especially interesting in terms of interpretability, or if each user must be paired with a single item. We refer to this particular case of our model as BOMIC (Bias-OMIC). 

\subsection{Second example: OMIC+}
\label{omic+}
Another highly representative example is the case where we are given side information about the users and items in the form of communities, where each user (resp. item) belongs to exactly one user (resp. item) community. In this situation, we construct the columns of $X^{(1)}$ (resp. $Y^{(1)}$) as normalized indicator functions of the user (resp. item) communities. The columns of $X^{(2)}$ (resp. $Y^{(2)}$) can then be chosen as any orthonormal basis of the orthogonal complement of $X^{(1)}$ (resp. $Y^{(1)}$) in $\R^{m}$ (resp. $\R^n$). In the case where the user (resp. item) communuties each contain a fixed number of members, the model becomes equivalent to the following optimization problem (up to multiplicatve constants in the regularising parameters): 
		\begin{align}
		\label{theoptother}
		&\min_{C,M,U,Z}\mathcal{L} \quad \text{with} \quad  \mathcal{L}= \\&\  \sum_{(i,j)\in \Omega} |C_{f(i),g(j)}+M_{i,g(j)}+U_{f(i),j}+Z_{i,j}-R_{i,j}|^2 \nonumber \\&  +\lambda_{1} \|C\|_{*}+\lambda_{12} \|M\|_{*}+\lambda_{21}\|U\|_{*}+\lambda_{22} \|Z\|_{*},  
		\end{align}
		subject to 
		\begin{align}
		\label{cond}
		   &\sum_{i\in f^{-1}(u)} M_{i,v}=0 \quad \forall u\leq d_1,v\leq d_2, \nonumber \\ &  \sum_{j\in g^{-1}(v)} U_{u,j}=0 \quad \forall u\leq d_1,v\leq d_2,\nonumber \\ &
		   \sum_{i\in f^{-1}(u)} Z_{i,j} =0 \quad \forall j\leq n,\nonumber  \\ & \text{and}
		   \sum_{j\in g^{-1}(v)} Z_{i,j}=0 \quad \forall i\leq m.
		\end{align}
		Here, the function $f:\{1,2,\ldots,m\}\rightarrow \{1,2,\ldots,d_1^{1}\}$ (resp. $g:\{1,2,\ldots,n\}\rightarrow \{1,2,\ldots,d_2^{1}\}$) assigns to each user (resp. item) its community. In terms of interpretability, note that in the predictors $C_{f(i),g(j)}+M_{i,g(j)}+U_{f(i),j}+Z_{i,j}$ above, the contribution from $C$ corresponds to effects that only depend on the community of the user and the community of the item, the contribution from $M$ (resp. $U$) corresponds to 'specific user-item community' (resp. 'specific item-user community') effects, whilst the contribution from $Z$ corresponds to effects that are purely specific to the user and the item (independently of their respective communities).

\subsection{Third example: BOMIC+}
\label{bomic+}
We note that several variations of the ideas for the construction of the auxiliary matrices $X^{(k)}$ and $Y^{(l)}$ as above are useful in practice. An important instance is BOMIC+, which combines both of the ideas above by incorporating \textit{both} user and item biases \textit{and} community side information. This is the specific model explained in Figure~\ref{fig:mainfig}, and is further empirically investigated in the experiments section.  Given community side information, we define our auxiliary matrices $X^{(k)}$ and $Y^{(l)}$ as follows:
\begin{itemize}
	\item $X^{(1)}$ and $Y^{(1)}$ are constructed as in the case of BOMIC (\ref{bias}), i.e. $X^{(1)}=(\frac{1}{\sqrt{m}},\ldots,\frac{1}{\sqrt{m}})^\top$, $Y^{(1)}=(\frac{1}{\sqrt{n}},\frac{1}{\sqrt{n}},\ldots,\frac{1}{\sqrt{n}})^\top$;
	\item Let $X$ (resp. $Y$) denote a matrix whose columns are indicator functions of the user (resp. item) communities. The columns of $X^{(2)}$ (resp. $Y^{(2)}$)  form an orthonormal basis of the space $\{v\in \spn(X):\langle v, X^{(1)}\rangle=0\}$ (resp. $\{v\in \spn(Y):\langle v, Y^{(1)}\rangle=0\}$), where $\spn(X)$ (resp. $\spn(Y)$) denotes the span of the columns of $X$ (resp. $Y$). 
	\item Finally, the columns of $X^{(3)}$ (resp. $Y^{(3)}$) form an orthonormal basis of the orthogonal complement of the columns of $(X^{(1)},X^{(2)})$ (resp. $(Y^{(1)},Y^{(2)})$) in $\mathbb{R}^m$ (resp. $\mathbb{R}^n$).
\end{itemize}

Thus, this model corresponds to equation~\eqref{simplebiasfirst}, together with some orthogonality constraints. Indeed, $X^{(1)} M^{(1,1)}(Y^{(1)})^\top$ is a constant and corresponds to $c$. Furthermore, all the rows of $X^{(1)} M^{(1,3)}(Y^{(3)})^\top$  (resp. columns of $X^{(3)} M^{(3,1)}(Y^{(1)})^\top$) are equal, so that
the term $X^{(1)} M^{(1,3)}(Y^{(3)})^\top$  (resp. $X^{(3)} M^{(3,1)}(Y^{(1)})^\top$) corresponds to $\mathbf{u}$ (resp. $\mathbf{m}$). $X^{(2)}M^{(2,2)}(Y^{(2)})^\top$ is the side information term corresponding to $\mathbf{x}_iM\mathbf{y}_j^\top$ in~\eqref{simplebiasfirst}. Meanwhile, the remaining terms $X^{(1)}M^{(1,2)}(Y^{(2)})^\top+X^{(2)}M^{(2,1)}(Y^{(1)})^\top+X^{(3)}M^{(3,2)}(Y^{(2)})^\top+X^{(2)}M^{(2,3)}(Y^{(3)})^\top+X^{(3)}M^{(3,3)}(Y^{(3)})^\top$ correspond to the term $Z_{i,j}$ 
from equation~\eqref{simplebiasfirst}, further refined into specific components distinguishing effects involving (1) only the side information of the users but not that of the items (or vice versa), (2) interactions between user bias and item side information (or vice versa) or (3) no side information or biases whatsoever.

\subsection{Optimization algorithm}
In this subsection, we propose an iterative imputation algorithm to solve the problem~\eqref{theopt} with the square loss. The first step is to develop a method to solve~\eqref{theopt} in the case where $\Omega=\{1,2,\ldots,m\}\times \{1,2,\ldots,n\}$, the so-called “fully known case”. The final solution can then be obtained by iteratively using this method.

\subsubsection{The fully known case}

Recall the singular value thresholding operator $S_{\lambda}$ from~\cite{softimpute} and~\cite{donoho}, which is defined by $S_{\lambda}(Z)=\sum_{i=1}^r (\lambda_i-\lambda)_{+}v_iw_i^\top$, where $Z=\sum_{i=1}^r \lambda_iv_iw_i^\top$ is the singular value decomposition (SVD) of $Z$.

In our case, we now introduce the \textit{Generalized singular value thresholding operator} $S_{\Lambda}$, which, for any set of parameters $\Lambda=(\lambda_{k,l})_{k\leq K\atop l\leq L}$, and given a set of auxiliary matrices $X^{(k)},Y^{(l)}$ (satisfying the orthogonality conditions~\eqref{orthogonalityconditions}), is defined by 
\begin{align}
\label{SingThresh}
S_{\Lambda}(Z)=\sum_{k=1}^K\sum_{l=1}^L X^{(k)}S_{\lambda_{k,l}}(M^{(k,l)})(Y^{(l)})^{\top},
\end{align}
where $M^{(k,l)}=(X^{(k)})^\top Z(Y^{(l)})$, which ensures  $Z=\sum_{k=1}^K\sum_{l=1}^L X^{(k)}M^{(k,l)}(Y^{(l)})^{\top}$. 

Note that $S_{\Lambda}(Z)$ is well-defined since the spaces $\mathcal{S}_{k,l}=\left\{X^{(k)}M(Y^{(l)})^\top\right\}\subset \mathbb{R}^{m\times n}$ are orthogonal with respect to the Frobenius inner product, and in particular, linearly independent. 

\begin{proposition}
	\label{fullyknown}
	The definition in equation~\eqref{SingThresh} is equivalent to the following: 
	\begin{align}
	\label{SingThreshNew}
	S_{\Lambda}(Z)=\sum_{k=1}^K\sum_{l=1}^L X^{(k)}S_{\lambda_{k,l}}\left( (X^{(k)})^\top Z Y^{(l)}  \right)(Y^{{l}})^\top.
	\end{align}
	Furthermore, $\tilde{Z}=S_{\Lambda}(Z)$ is the solution to the following optimization problem: 
	\begin{align}
	\label{thenewprob1}
	\minimize \frac{1}{2}\|\tilde{Z}-Z\|_{\Fr}^2+ \sum_{k=1}^K\sum_{l=1}^L\lambda_{k,l}\left\|(X^{(k)})^\top Z Y^{(l)}  \right\|_{*},
	\end{align}
	or equivalently 
	\begin{align}
	\label{thenewprob}
	\minimize &\frac{1}{2}\|\tilde{Z}-Z\|_{\Fr}^2+ \sum_{k=1}^K\sum_{l=1}^L\lambda_{k,l}\left\|M^{(k,l)} \right\|_{*},\\
	\label{decomp}
	\text{subject to} \quad \tilde{Z}&=\sum_{k=1}^K\sum_{l=1}^L X^{(k)}M^{(k,l)}(Y^{(l)})^\top.
	\end{align}
	
\end{proposition}

\subsubsection{The OMIC algorithm}
\label{subomicc}
For any fixed set of hyperparameters $\Lambda$ and auxiliary matrices $X^{(k)},Y^{(l)}$ (for $k\leq K,l\leq L$), the final solution to the optimization problem~\eqref{theopt} is obtained iteratively: at each step, a target matrix is constructed by setting the entries of $\Omega$ to the observed values and imputing the values of the previous iteration to the entries of $\Omega^\perp$. We then apply the fully known case~\eqref{fullyknown} to this target matrix to reach the next iterate. This algorithm converges for any initial (0th iteration) matrix. 

However, if several values of $\Lambda$ must be calculated, we can use warm starts to improve efficiency. 
Algorithm~\ref{OMICimpute} below does this in the case where the range of values for $\Lambda$ is a product $\mathcal{V}=\prod_{k,l} \mathcal{V}_{k,l}$ where the $\mathcal{V}_{k,l}$ are finite sets of candidate values for $\lambda_{k,l}$ (ordered in increasing or decreasing order): initial estimates of $M_{k,l}$ for each value of $\lambda_{k,l}$ are calculated by setting each $\lambda_{k',l'}$ to infinity for $k'\neq k,l'\neq l$ and solving the full problem~\eqref{theopt} in this case\footnote{"Setting $\lambda_{k',l'}$ to infinity" amounts to not including the $(k',l')$ term in the sum~\eqref{thesume} which defines our predictors. Thus our warm starts are computed by training each term in that sum independently.}. Furthermore, each of those sets of $M_{k,l}$ are calculated using warm starts along the sequence of $\lambda_{k,l}\in \mathcal{V}_{k,l}$. For any real number $\lambda$, $p_{k,l}(\lambda)$ denotes the set of hyperparameters $\Lambda$ with $\Lambda_{k,l}=\lambda$ and $\Lambda_{k',l'}=\infty$ otherwise. Note also that this algorithm depends on the auxiliary matrices $X^{(k)}$ and $Y^{(l)}$ through the generalized singular value thresholding operator $S_{\Lambda}$, defined in equation~\eqref{SingThresh} above. 

Note that $M^{(k,l)}$ is of dimension $d^{k}_1\times d^{l}_2$. Whenever one of $d^{k}_1$ or $d^{l}_2$ is small, the computation of SVDs is trivial. For instance, in the specific case of BOMIC, matrices $M^{(k,l)}$ are vectors whenever $k=1$ or $l=1$, thus calculating the SVD simply amounts to normalizing a vector. Therefore, although our algorithm theoretically requires $KL$ SVD operations at each iteration, in fact, most of them are trivial. We refer the reader to Appendix~\ref{complexity}, for a more thorough explanation and an empirical evaluation  of the lack of any extra computational burden from the trivial SVDs. Furthermore, in practice, a rank-restricted version of the SVD operation can be employed, together with a suitable warm-start strategy. In Subsection~\ref{scalableAlg}, we develop novel techniques to optimize the computational and memory burden of the algorithm in the most interesting and practically relevant particular cases. 
\begin{algorithm}[h!]
	\caption{OMIC \\
		INPUT: $R_{\Omega}$, set of regularizing parameters $\mathcal{V}=\prod_{k,l} \mathcal{V}_{k,l}$  \\
		OUTPUT: Set of recovered matrices $Z^{\Lambda}$ for all $\Lambda\in \mathcal{V}$}
		\label{OMICimpute}
	\begin{algorithmic}[1] 
		\FOR{$k\in \{1,2,\ldots,K\}$}
		\FOR{$l\in\{1,2,\ldots,L\}$}
		\STATE Initialize $Z^{\new}\leftarrow 0$
		\FOR{$\lambda\in \mathcal{V}_{k,l}$}
		\REPEAT
		\STATE $Z^{\old}\leftarrow Z^{\new}$
		\STATE $Z^{\new}\leftarrow S_{p_{k,l}(\lambda)}\left(R_{\Omega}+P_{\Omega^\perp}(Z^{\old})  \right)$
	\UNTIL converged
		\STATE $Z^{k,l,\lambda}\leftarrow Z^{\new}$
		\ENDFOR
		\ENDFOR
		\ENDFOR
		\FOR{$\Lambda\in \mathcal{V}$}
		\STATE $Z^{\new}\leftarrow \sum_{k,l=1}^{K,L} Z^{k,l,\lambda_{k,l}} $
		\REPEAT
	    \STATE $Z^{\old}\leftarrow Z^{\new}$
		\STATE $Z^{\new}\leftarrow S_{\Lambda}\left(R_{\Omega}+P_{\Omega^\perp}(Z^{\old})  \right)$
		\UNTIL converged
		\STATE $Z^{\Lambda} \leftarrow  Z^{\new}$
		\ENDFOR 
				\RETURN $Z^{\Lambda}$  for $\Lambda\in \mathcal{V}$
	\end{algorithmic}
\end{algorithm}

\subsubsection{Convergence guarantees}

Our algorithm enjoys convergence guarantees, which we present here. Here, we fix a $\Lambda$ and assume that we perform the iterative imputation procedure in the algorithm above starting from $Z^{0}=0$, with (for each $i\geq 0$)
\begin{align}
\label{defseq}
Z^{i+1}= S_{\Lambda}\left( P_{\Omega}(R)+P_{\Omega^\perp}(Z^i) \right).
\end{align} 

We have the following two results, whose proofs are left to the Appendix.

\begin{theorem}
	\label{convergence}
	Consider our general setting with auxiliary matrices satisfying the conditions~\eqref{orthogonalityconditions} and the operator $S_\Lambda$ defined accordingly. The sequence $Z^{i}$ defined in~\eqref{defseq} converges to a solution $Z^\infty$ of the optimization problem~\eqref{theopt} with the squared loss. In particular, the loss $\mathcal{L}$ converges to the minimum  $\mathcal{L}^*$ of optimization problem~\eqref{theopt}.
\end{theorem}

\begin{theorem}
	\label{rate}
	Let $\mathcal{L}^*$ be the minimum value of the loss $\mathcal{L}$ from problem~\eqref{theopt}. For every fixed $\Lambda$, the sequence $Z^i$ has the following worst-case asymptotic convergence: 
	\begin{align}
	\label{BBB}
	\mathcal{L}(Z^i)-\mathcal{L}(Z^\infty)=	\mathcal{L}(Z^i)-\mathcal{L}^*\leq \frac{2\|Z^0-Z^\infty\|_{\Fr}^2}{i+1},
	\end{align}
	where $Z^\infty$ is the limit of the sequence $Z^i$.
\end{theorem}

\textbf{Remark:}
Note that although the RHS depends on the limit $Z^\infty$, it can be further bounded above as follows
\begin{align*}
    \frac{2\|Z^0-Z^\infty\|_{\Fr}^2}{i+1}\leq  \frac{ 2\max_{Z^*\in \mathcal{S}}\|Z^0-Z^*\|_{\Fr}^2}{i+1}
\end{align*}
where $\mathcal{S}=\argmin_Z (\mathcal{L}(Z))$ is the \textit{set} of solutions of Problem~\eqref{theopt}\footnote{Note the solution is not necessarily unique, as we discuss in greater detail at the end of Appendix~\ref{proofconv}.}. Thus, the rate $1/(i+1)$ holds.

\subsection{A scalable algorithm for BOMIC+ } 
\label{scalableAlg}

The main computational step at each iteration of the algorithm we propose to solve~\eqref{theopt} is the calculation of the solution to an instance of the fully known case, which can be obtained via our \textit{generalized singular value thresholding operator} $S_{\Lambda}$. 
Observe that the sum of the numbers of columns of all $X^{(k)}$ (resp. $Y^{(l)}$) is $m$ (resp. $n$). Thus, if $m$ and $n$ are large (which is often the case in RSs contexts) it is infeasible to even store all the auxiliary matrices in memory, let alone perform operations on them directly. The same problem can occur with some of the latent matrices $M^{(k,l)}$. It is easy to see that the largest $M^{(k,l)}$ has at least $m/K \times n/L$ entries which is also very large.

In this subsection, we show how to circumvent this difficulty in the specific case of BOMIC+ where $K=L=3$ and the auxiliary matrices are defined in Section~\ref{bomic+} assuming the side information $X,Y$ consists of indicator functions of communities. For instance, the columns of $Y$ could represent sets of movies belonging to a specific genre, whilst the columns of $X$ could represent countries, genders or professions for users. Our strategy heavily relies on both the “sparse-plus-low-rank” structure present in the target matrices of the "fully known problem" solved at each iteration, as well as the specific structure of community side information.

First, we define some matrices $\widetilde{X}^{(1)}$, $\widetilde{X}^{(2)}$, $\widetilde{X}^{(3)}$, $\widetilde{Y}^{(1)}$, $\widetilde{Y}^{(2)}$, $\widetilde{Y}^{(3)}$ as follows: $\widetilde{X}^{(1)}=X^{(1)}$, $\widetilde{Y}^{(1)}=Y^{(1)}$,  $\widetilde{X}^{(2)}=\norm(X)$, $\widetilde{Y}^{(2)}=\norm(Y)$, $\widetilde{X}^{(3)}=\Id$ and  $\widetilde{Y}^{(3)}=\Id$. Here, $\norm$ denotes the operation of normalizing each column. Note that $X$ (resp. $Y$) is a sparse matrix composed of the indicator functions of user (resp. item) communities. Therefore, the matrices $\widetilde{X}^{(1)}$, $\widetilde{X}^{(2)}$,$\widetilde{X}^{(3)}$,$\widetilde{Y}^{(1)}$ ,$\widetilde{Y}^{(2)}$and $\widetilde{Y}^{(3)}$ can easily be stored in memory, and it is easy to multiply them by arbitrary vectors on either side.

To see how these matrices will help us execute BOMIC+, observe first that due to the rotational invariance of SVDs, the operator $S_{\Lambda}$ can be rewritten as 
\begin{align}
\label{SingThreshSimple}
S_{\Lambda}(Z)=\sum_{k=1}^3\sum_{l=1}^3 S_{\lambda_{k,l}}\left(X^{(k)}(X^{(k)})^\top ZY^{(l)}(Y^{(l)})^{\top}\right).
\end{align}

Thus, for a given $Z$, calculating $S_{\Lambda}(Z)$ boils down to computing the SVD of the matrices $H^{(k,l)} = X^{(k)}(X^{(k)})^\top ZY^{(l)}(Y^{(l)})^{\top}$, which are the projections of $Z$ on the spaces corresponding to each pair of auxiliary matrices. We perform the SVD computation through a rank-restricted alternating least squares algorithm (ALS, see Algorithm~\ref{SVDCalc}). This requires an efficient strategy to compute $H^{(k,l)}W_1$ and $W_2H^{(k,l)}$, where $W_1$ and $W_2$ are any real conformable matrices.

We now observe that for any orthogonal matrix $U$, if $W_1\in \mathbb{R}^{n\times r}$, $UU^\top W_1$ is the projection of $W_1$ on the span of the columns of $U$. Crucially, if $V$ is an orthogonal matrix with $\spn(V) \subset \spn(U)$, then for any $W_1\in \mathbb{R}^{n \times r}$, $UU^\top W_1 -VV^\top W_1$ is the projection of $W_1$ on the orthogonal complement of $V$ in $U$.
Applying this to the matrices $Y^{(l)}$ and $\widetilde{Y}^{(l)}$, we obtain, for all $l\leq 3$: \begin{align}
\label{simplifythingsright}
&&Y^{(l)}(Y^{(l)})^{\top}W_1\nonumber \\&=\widetilde{Y}^{(l)}(\widetilde{Y}^{(l)})^\top W_1&-\widetilde{Y}^{(l-1)}(\widetilde{Y}^{(l-1)})^\top W_1.
\end{align}

Similarly, for any $W_2 \in \mathbb{R}^{m \times r}$ and $k\leq 3$,
\begin{align}
\label{simplifythingsleft}
& &X^{(k)}(X^{(k)})^\top W_2\nonumber \\&=\widetilde{X}^{(k)}(\widetilde{X}^{(k)})^\top W_2 &-\widetilde{X}^{(k-1)}(\widetilde{X}^{(k-1)})^\top W_2.
\end{align}

\begin{remark}
	The operation $W_1\leftarrow Y^{(3)}(Y^{(3)})^{\top}W_1$ performed as above can be visualized intuitively: it corresponds to removing from each component of each  column ${W_1}_{\nbull,i}$ the average of the components in the same community (in the same column). $$\forall i\leq m, \quad v_i\leftarrow v_i-\frac{\sum_{j\in c_i}v_j}{\#(c_i)}.$$
\end{remark}

With the above techniques in our tool kit, we can now present our algorithm for calculating $H^{(k,l)}W_1$. We will divide the task into three steps. 

First, we evaluate $\widetilde{W_1} = Y^{(l)}[(Y^{(l)})^{\top}W_1]$ using~\eqref{simplifythingsright}. 

Next, observe that as a byproduct of the iterative imputation procedure, $Z$ can be decomposed as the sum of a sparse matrix $Z_{Sp}$ and a low-rank matrix $Z_{LR}$ as follows:
\begin{align}
\label{decompZ}
Z= Z_{Sp}  + Z_{LR} =  Z_{Sp} + U_{LR}\left[D_{LR}{V_{LR}}^\top\right].
\end{align}
Using this decomposition, it is easy to calculate the quantity $\widetilde{\widetilde{W_1}} = Z\widetilde{W_1}$ as follows:
\begin{align}
\label{w_tilde_tilde}
\widetilde{\widetilde{W_1}} = Z\widetilde{W_1} = Z_{Sp}\widetilde{W_1}  + U_{LR}D_{LR}({V_{LR}}^\top \widetilde{W_1}).
\end{align}

Finally, we calculate $H^{(k,l)}W_1= \big( X^{(k)}[(X^{(k)})^\top\widetilde{\widetilde{W_1}} ] \big)\top$ using~\eqref{simplifythingsleft}. Symmetrically and with the same arguments we can compute $W_2H^{(k,l)}$.

Alg.~\ref{SVDCalc} (based on~\cite{softimputeALS}) describes how to compute $S_{\Lambda}(Z)$ for a fixed hyperparameter set $\Lambda$ and Alg.~\ref{BOMICLUSTERS2} is our fully scalable implementation of  Alg.~\ref{OMICimpute} for the BOMIC+ case with community side information. In practice, we can further use warm start strategies such as the one presented in Alg.~\ref{OMICimpute} to speed up convergence.

\textbf{Remark:} The convergence of Alg.~\ref{SVDCalc} follows from that of Alg. 2.1 in~\cite{softimputeALS}: for each combination $(k,l)$~\footnote{leaving aside the extra implementation details required for our specific model}, of which there are finitely many, the while loop from lines 11 to 22 essentially runs Algorithm 2.1 from ~\cite{softimputeALS} on the component matrix $R^{(k,l)}=(X^{(k)})^\top R Y^{(l)}$, thus the full algorithm converges. The convergence of Alg.~\eqref{BOMICLUSTERS2} then follows from (1) the convergence of Alg.~\ref{SVDCalc} (which is used inside the while loop between lines 3 and 6) together with (2) the convergence of Alg.~\ref{omic+}, which is established in Theorem~\ref{convergence}

\textbf{Remark:} To avoid manipulating or storing large matrices, one must perform the operations in the correct order, which is defined by the brackets. This remark applies in particular to Algorithms~\ref{SVDCalc} and~\ref{BOMICLUSTERS2}.

\begin{algorithm}
	\caption{\textbf{($S_{\Lambda}$)-ALS}: Generalized restricted truncated singular value thresholding via alternating least squares\\ 
		\textbf{INPUT:}  $Z\in \mathbb{R}^{m\times n}$ decomposed as in~\eqref{decompZ}; thresholding parameters $\Lambda$, maximum rank $r$
	\\
	\textbf{OUTPUT}: $S_{\Lambda}(Z)$ represented in storable low-rank format as $(\mathfrak{U},\diag(\Sigma),\mathfrak{V})$ such that $S_{\Lambda}(Z)=\mathfrak{U}\diag(\Sigma)\mathfrak{V}^\top$}
		\label{SVDCalc}
	\begin{algorithmic}[1]
	    \STATE \textbf{Procedure} Projection(${E}^{(1)}$,${E}^{(0)}$,$\Theta$) 
        \RETURN $ {E}^{(1)}[({E}^{(1)})^\top \Theta] -{E}^{(0)}[({E}^{(0)})^\top \Theta]$
        \STATE \textbf{end Procedure}
	    \STATE $\mathfrak{U} \leftarrow \Sigma\leftarrow \mathfrak{V}\leftarrow NULL$
	    \STATE $\widetilde{X}^{(0)} \leftarrow 0, \widetilde{Y}^{(0)} \leftarrow 0$
	    \FOR{k in (1..3)}
		\FOR{l in (1..3)}
		\STATE $U\leftarrow$  random orthogonal $m\times r$ matrix
		\STATE $D\leftarrow \Id_{r\times r}$
		\STATE $A \leftarrow UD$
		\WHILE{$AB^\top$ not converged}
		\STATE $\Theta \leftarrow UD(D^2+\lambda_{k,l} I)^{-1}$
		\STATE $\widetilde{\Theta} \leftarrow$ Projection($\widetilde{X}^{(k)},\widetilde{X}^{(k-1)},\Theta$) \\
		\STATE $\widetilde{\widetilde{\Theta}} = \widetilde{\Theta}^\top Z_{Sp}  + [(\widetilde{\Theta}^\top U_{LR})D_{LR}]{V_{LR}}^\top$
		\STATE $\tilde{B}  \leftarrow$ Projection($\widetilde{Y}^{(l)},\widetilde{Y}^{(l-1)},\widetilde{\widetilde{\Theta}^\top}$)
		\STATE Compute the SVD of $\tilde{B}D=\tilde{V}\tilde{D}^2R^\top$ and attribute $V\leftarrow \tilde{V}$, $D\leftarrow \tilde{D}$ and $B\leftarrow VD$\\
		\STATE $\Theta \leftarrow V D(D^2+\lambda_{k,l} I)^{-1}$
		\STATE $\widetilde{\Theta} \leftarrow$ Projection($\widetilde{Y}^{(l)},\widetilde{Y}^{(l-1)},\Theta$)

		\STATE $\widetilde{\widetilde{\Theta}} =  Z_{Sp}\widetilde{\Theta}  + U_{LR}[D_{LR}({V_{LR}}^\top \widetilde{\Theta})]$
		\STATE $\tilde{A}  \leftarrow$ Projection($\widetilde{X}^{(k)},\widetilde{X}^{(k-1)},\widetilde{\widetilde{\Theta}}$)\\
		\STATE Compute the SVD of $\tilde{A}D=\tilde{U}\tilde{D}^2\tilde{R}^\top$ and attribute $U\leftarrow \tilde{U}$, $D\leftarrow \tilde{D}$ and $A\leftarrow UD$\\
		\ENDWHILE
		
		\STATE $\widetilde{\Theta} \leftarrow$ Projection($\widetilde{Y}^{(l)},\widetilde{Y}^{(l-1)},V$)
		\STATE $\widetilde{\widetilde{\Theta}} =  Z_{Sp}\widetilde{\Theta}  + U_{LR}[D_{LR}({V_{LR}}^\top \widetilde{\Theta})]$
		\STATE $M  \leftarrow$ Projection($\widetilde{X}^{(k)},\widetilde{X}^{(k-1)},\widetilde{\widetilde{\Theta}}$)\\
		\STATE Compute the SVD of $M$: $M=\bar{U}D_{\sigma}\bar{R}^\top$.

		\STATE $\mathfrak{U} \leftarrow \text{CONCAT}(\mathfrak{U},\bar{U})$ 
		\STATE $\Sigma \leftarrow \text{CONCAT}(\Sigma,\big((\sigma_1-\lambda_{k,l})_{+},(\sigma_2-\lambda_{k,l})_{+},\ldots,(\sigma_r-\lambda_{k,l})_{+}\big))$ 
		\STATE $\mathfrak{V} \leftarrow \text{CONCAT}(\mathfrak{V},V\bar{R})$
		\ENDFOR
		\ENDFOR
			\RETURN $\mathfrak{U};\diag(\Sigma);\mathfrak{V}$

		\end{algorithmic}
\end{algorithm}

\begin{algorithm}
	\caption{\textbf{Scalable BOMIC+} \\ \textbf{INPUT} $R_{\Omega}$,  regularizing parameters $\Lambda\in  \mathbb{R}^{3\times 3}$, maximum rank $r$
		\\
	\textbf{OUTPUT}: Recovered matrix $Z$
	}
		\label{BOMICLUSTERS2}
	\begin{algorithmic}[1] 
		\STATE $Z_s \leftarrow R_{\Omega}$ 
		\STATE $Z_{LR} \leftarrow \{U_{LR} \leftarrow 0, D_{LR} \leftarrow 0,V_{LR}\leftarrow 0 \}$
		\WHILE{Not converged}
		\STATE $Z_s \leftarrow R_{\Omega} - P_{\Omega}(Z_{LR})$ 
		\STATE $Z_{LR}$ $\leftarrow$ ($S_{\lambda}$)-ALS($Z=\{Z_s,Z_{LR}\}$)
		\ENDWHILE
		\RETURN $Z\leftarrow Z_{LR}$
	\end{algorithmic}
\end{algorithm}

\textbf{Complexity analysis:}
Treating $K$ and $L$ as constants, our algorithm can achieve $O(|\Omega|r+(m+n)r^2)$ complexity all three important cases OMIC+, BOMIC and BOMIC+. We note that, as explained above, although in principle the computational burden is multiplied by $KL$ (which can be as much as 9 in the case of BOMIC+), in practice, further refinements can help the algorithm perform at the same speed as SoftImpute, as is hinted at in Subsection~\ref{subomicc}. We refer the reader to Appendix~\ref{complexity} for more details.

\section{Generalization bounds}

In this section, we present some generalization bounds for our model in the three relevant cases BOMIC, OMIC+ and BOMIC+.

\subsection{Distribution-free bounds}
We begin with a distribution-free approach, meaning that we do not assume anything about the sampling distribution. In particular, the bounds in this subsection behave worse than the corresponding bounds under uniform sampling assumptions such as the ones in the celebrated works~\cite{CandesRecht,genius}.

We first focus on generalization bounds which apply when the side information corresponds to user/item biases and/or community side information (OMIC+).

We will write $C_{k,l}$ for an upper bound on the entries of the ground truth component $X^{(k)}R^{(k,l)}(Y^{(l)})^\top$  where $R^{(k,l)}=(X^{(k)})^\top R Y^{(l)}$.  Similarly, we will write $r_{k,l}$ for the rank of $R^{(k,l)}$.

Thus, e.g., if $C_{1,2}$ is large, one concludes that in the ground truth matrix, the specific affinities of items in $\{1,2,\ldots n\}$ to whole communities of users is a significant factor in determining the value of each entry. If $C_{2,2}$ is large, the individual affinities between users and items, independently of their respective communities, is a strong factor. 

The following follows from the inequality $\sqrt{x}+\sqrt{y}+\sqrt{z}+\sqrt{t}\leq 2\sqrt{x+y+z+t}$ and Theorem~\ref{CommunityPrecise2} in the Appendix. 

\begin{corollary}
	\label{corspecific}
	Consider the community setting above and assume the user (resp. item) communities are of sizes within a ratio of $O(1)$, as well as that (wlog) $b\geq a,$  and $m\geq n$. For any $\epsilon>0$, the required number of entries to recover the ground truth matrix within $\epsilon$ expected loss (with probability $\geq 1-\delta$) is  \begin{align*}
O\Bigg( (1/\epsilon^2)\bigg(C_{1,1}^2b\sqrt{ar_{1,1}}+C_{1,2}^2 n\sqrt{ar_{1,2}}+\\  C_{2,1}^2m \sqrt{br_{2,1}}+C_{2,2}^2 m\sqrt{r_{2,2}n}  +\log(1/\delta)    \bigg)  \Bigg).
	\end{align*}
	Alternatively, using the nuclear norms of the component matrices, we have the following sample complexity bound:
	\begin{align*}
O\Bigg( (1/\epsilon^2)\bigg(\sqrt{b}t_{1,1}+t_{1,2} \sqrt{n}+\\  t_{2,1} \sqrt{m}+t_{2,2} \sqrt{m}  +\log(1/\delta)    \bigg)  \Bigg).
	\end{align*}
\end{corollary}
Note that the above bounds improve with the quality of the side information: the better the ground truth matrix can be approximated by community behaviour, the closer the bound behaves to the bound one would obtain for an $a\times b$ matrix with each user and item being assimilated to its community. Furthermore, the result applies in particular to the BOMIC model from Section~\ref{bias} by grouping all users (resp. items) into a single community. This yields a bound of the order $(1/\epsilon^2)(C_{1,1}^2+nC_{1,2}^2+mC_{2,1}+(\sqrt{n}+\sqrt{m})\sqrt{mn r}C_{2,2}^2+\log(1/\delta))$ where (e.g.) $C_{2,2}$ is a bound on the bias-free part of the ground truth matrix. Similar bounds hold for the BOMIC+ model~\ref{bomic+} (cf. equation~\eqref{bomic+bound} in the Appendix).

\subsection{Bounds assuming uniform marginals}

In this subsection, we present some bounds under the assumption that the sampling distribution has uniform marginals (i.e., whenever an entry is sampled, the probability of choosing an entry in any given row (resp. column) is $1/m$ (resp. $1/n$)). In this case, only the terms containing side information for both users and items present an improvement. 

\begin{corollary}
	\label{corspecific2}
	Consider the community setting above (OMIC+) and assume: (1) each user (resp. item) community is of equal size; (2) that (wlog) $b\geq a,$  and $m\geq n$ and (3) the marginals of the sampling distribution are uniform. For any $\epsilon>0$, (the required number of entries to recover the ground truth matrix within $\epsilon$ expected loss  is  \begin{align*}
O\Bigg( (1/\epsilon^2)\bigg(C_{1,1}^2br_{1,1}\log(b)+C_{1,2}^2 n r_{1,2}\log(n)+\\  C_{2,1}^2m r_{2,1}\log(m)+C_{2,2}^2 m r_{2,2}\log(m)\bigg)  \Bigg).
	\end{align*}
\end{corollary}
\begin{proof}
Follows from Theorem~\ref{thm:omic+bound} in Appendix~\ref{unimarg}.

\end{proof}
In the supplementary (Section~\ref{sec:experiments_uniform}), we experimentally validate the bound on some synthetic data, and observe a good match between the bound and the de facto sample complexities we encounter in practice.

\textbf{Remark:} An interesting observation from the bounds is that the knowledge of the explicit side information vectors (i.e. communities or biases) helps achieve a faster sample complexity (as opposed to the knowledge of the  equivalent low-rank constraint) whenever either:
\begin{itemize}
    \item the first order term $R^{(1,1)}$ is significant; or
    \item the distribution is arbitrary (doesn't have uniform marginals).
\end{itemize}

This is explained in detail in Subsection~\ref{absorb} of the appendix. 
As our synthetic data experiments (cf. Subection~\ref{synthhh}) demonstrate, we can still gain something in practice for moderate-size matrices, even when it comes to cross-terms in the uniform sampling case, but such an improvement cannot be captured at the level of asymptotic results. 
\textbf{Remark:} Note that although the proofs are relatively straightforward, even the generalisation bounds corresponding to a single term do \textit{not} follow from standard results on inductive matrix completion. This is also explained in detail in Section~\ref{absorb} in the appendix.

\normalsize

\section{Experiments}
\label{experiments}

To compare OMIC with the baselines we conducted two experiment strands: synthetic data simulations and real-world applications. 

In the first case, we propose a matrix generation procedure that allows us to evaluate how the performance of BOMIC varies in different ground truth regimes: we generated sparsely observed matrices composed of the sum of user and item biases (generic behaviour) and a non-inductive term (specific behaviour). The proportion of each term, the number of observed entries and the distribution of known entries were varied.

In the latter case, we validated our model on five real recommender systems datasets: the Amazon recommender system dataset, the Douban movie database, the Goodreads spoiler dataset, and two versions of MovieLens. We show that our methods exhibit state-of-the-art performance in all cases.

All the hyperparameter tuning was done through cross-validation. The range of the parameters was adjusted according to each model's needs.

\subsection{Baselines}

Our model is a fundamental tool that relies only on the incomplete matrix and some high-level side information and has the benefit of interpretability. We compare our model with other similarly fundamental models such as IMC~\cite{IMC,mostrelated,IMC1,PIMC,tassisa18} and Softimpute~\cite{softimpute}, with the understanding that the basic ideas could be refined and incorporated into more sophisticated recommender systems.

\begin{itemize}
	
	\item \textbf{SoftImpute (SI)} is a matrix completion method that uses nuclear-norm regularization. This is a standard baseline for non-inductive matrix completion~\cite{softimpute}.
	\item \textbf{Biased SoftImpute (B-SI)} is a popular approach that consists in first, training user and item biases, and then training the SoftImpute model on the residuals.
	\item \textbf{Inductive Matrix Completion with Noisy Features (IMCNF) } In this model~\cite{mostrelated} we train a sum of an IMC term and a residual SoftImpute term jointly (via alternating optimization). This model requires side information, and was therefore only considered in real data experiments.  
\end{itemize}

\subsection{Metrics} \label{metrics}Let $R \in \R^{m\times n}$ denote the ground truth matrix, $\hat{R}^{(k)}$ the  matrix predicted by method $k$ and let $\bar{\Omega}$ be the test set (a subset of entries). Let $\bar{R}^{(k)}$, $\hat B_U^{(k)}$ and $ \hat B_I^{(k)}$ be respectively the zero-order term ($X^{(1)}M^{(1,1)}Y^{(1)}$ in BOMIC), the vector of user biases and the vector of items biases estimated by method $k$. In the SoftImpute case  we need an extra post-processing step to estimate the biases: $\bar{R}^{(SI)} = \sum_{ij}\hat{R}^{SI}_{ij}/mn$, $\hat B_{U_i}^{(SI)}=\sum_j  (\hat{R}^{(SI)}_{ij} - \bar R^{(SI)})/n $ and $\hat B_{I_j}^{(SI)}=\sum_i  (\hat{R}^{(SI)}_{ij} - \bar R^{(SI)})/m$. To assess the methods we used the metrics presented in the list bellow:

\begin{itemize}
	\item \textbf{[RMSE] Root-mean-square error}: $\RMSE = \sqrt{ \sum_{i,j \in  \bar{\Omega}} (\tilde{R}_{ij} - {R}_{ij})^2 /|\bar{\Omega}|}$ 
	\item \textbf{[MBD] Matrix bias deviation}: $\MBD =\left|\bar{R} -  \bar{R}^{(k)}\right|$
	\item \textbf{[UBD] (resp. IBD): User (resp. item) bias deviation}: $\UBD = \| {B_U} - \hat B_U^{(k)}\|_{\Fr}$  (resp. $\IBD = \| {B_I}- \hat B_I^{(k)}\|_{\Fr}$ )
	\item \textbf{[SPC] Spearman correlation}: 
	\small $\SPC=\rho_S \left(R_{\bar{\Omega}}, \hat{R}_{\bar{\Omega}}^{(k)}\right)$\normalsize, the Spearman correlation between two vectors composed of the entries of $R$ and $\hat{R}^{(k)}$ on the test set.
\end{itemize}

Since calculating the metrics MBD,UBD and IBD requires knowledge of all the entries of $T$, we only calculated them for the synthetic experiments. 
Note that lower values of RMSE, MBD and UBD and higher values of the Spearman correlation correspond to better performance.

\subsection{Synthetic data simulations}
\label{synthhh}
For synthetic data simulations, we evaluated the ability of our model BOMIC to detect and adapt to different regimes in terms of the importance of user and item biases. We constructed two fixed matrices $G$ and $S$, with the former made up purely of user/item biases, and the latter free of any implicit or explicit user or item biases. Then, we considered combinations $R(\alpha)=\alpha G+(1-\alpha)S$, observed either uniformly (which we describe with $\gamma=1$) or in a biased way designed to fool a naive bias method into miscalculating the user and item biases ($\gamma=4$). The proportion of observed entries $p_{\Omega}$ was also varied.

\subsubsection{Results} For each combination of $\alpha$, $\gamma$ and $p_{\Omega}$ we generated 50 different samples of $R(\alpha)$. Given a sampled matrix, we recovered the unknown entries through  \textbf{BOMIC}, \textbf{B-SI} and \textbf{SI}. Figure~\ref{fig:ressynt}  summarizes the results of the performed simulations. We observe that our method consistently outperforms B-SI and SI, in terms of RMSE, UBD and IBD, and performs as well as SI in the MBD case. In addition, OMIC's ability to recover the correct user and item biases (UBD and IBD in Figure~\ref{fig:ressynt}) is particularly marked in the case of non uniformly sampled entries, as might be expected, in line with Corollary~\ref{corspecific}.

\begin{figure*}
	\centering
	\includegraphics[width=.98\linewidth]{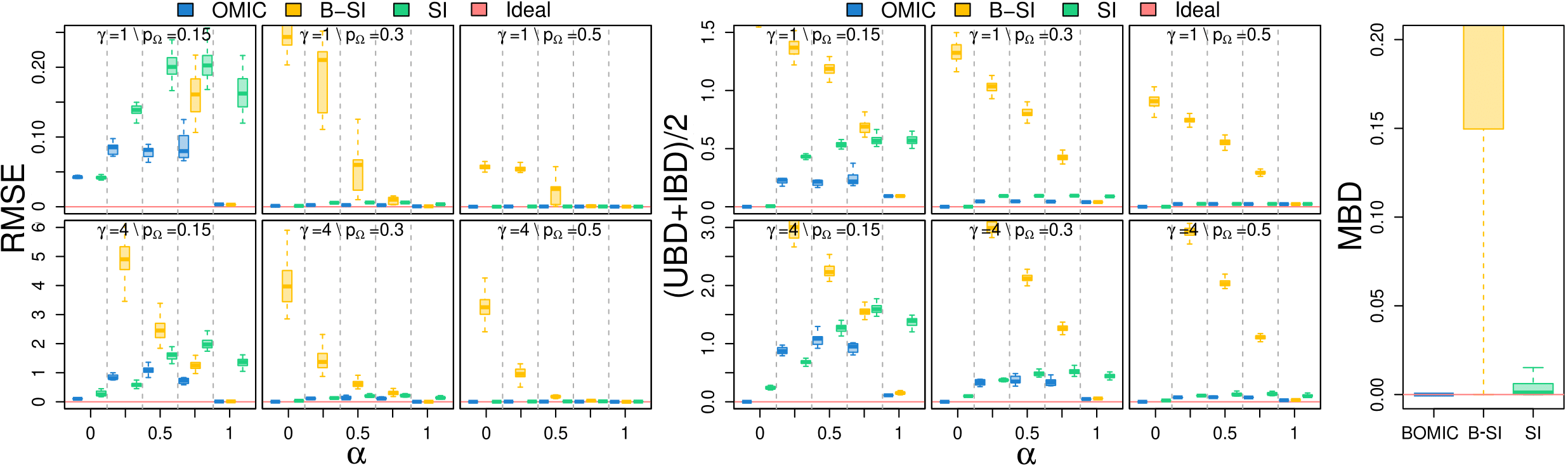}
	\caption{Summary of synthetic data simulations results. The first graph shows the relationship between all combinations of  the parameters ($\alpha, \gamma, p_\Omega$) and the RMSE. The second one shows how ($\alpha, \gamma, p_\Omega$) influences the correct recovery of user and item biases ((UBD+IBD)/2). Each box plot in the first two graphs is obtained from 50 simulations. The third graph displays the distribution of the MBD over all of the simulations.}
	\label{fig:ressynt}
\end{figure*}

\subsection{OMIC as a recommender system}

In this subsection, we present our results on real data from the field of recommender systems. We begin by introducing the datasets, then provide our results, and conclude with a practical exploration of the added interpretability benefits of our method. 
\subsubsection{Datasets}
In this paper, we worked with the following datasets:

\begin{itemize}
    \item \textbf{Amazon} ($R \in \mathbb{R}^{164383 \times 101364}$): Amazon is a multinational technology company which mainly focuses on e-commerce. We used the "Electronics" dataset, which we obtained through~\cite{he2016ups}. Users are Amazon's clients and items are electronic products (e.g. smartphones). The rating range is from 1 to 5 and the entry $(i,j)$ refers to the rating given by client $i$ to product $j$. Since no systematic side information was provided, we only investigated the performance of BOMIC and SoftImpute for this dataset.
	\item \textbf{Douban} ($R \in \mathbb{R}^{4988 \times 4903}$): Douban is a social network where users can produce content related to movies, music, and events. The ratings matrix was obtained through~\cite{DVN/JGH1HA_2019}. Douban users are members of the social network and Douban items are a subset of popular movies. The rating range is $\{1,2,\ldots,5\}$ and the entry $(i,j)$ represents the rating given by user $i$ to movie $j$. Feature vectors were collected by the authors and can be divided into two distinct parts: general features (e.g year of production, genres and movie duration) and the embedding of the description of the movie given by the pre-trained neural network Bert~\cite{devlin2018bert}.
	\item \textbf{Goodreads spoiler dataset (GRS)}  ($R \in \mathbb{R}^{ 4199 \times 3278}$): This dataset was released by~\cite{wan2019fine} and it is available online. Goodreads is a social cataloging website that allows individuals to freely search its database of books, annotations, and reviews. In this case, an entry $(i,j)$ represents the rating of the user $i$ for the book $j$ on a scale from $0$ to $5$. For each user-book pair, in addition to the rating score, the review text is also available. Each sentence of the review was annotated with respect to whether or not spoilers were present. We generated 89 features such as the length of the review and which percentage of the text contains spoilers. 
		\item \textbf{MovieLens}: We consider the MovieLens~1M ($R \in \mathbb{R}^{ 6040 \times 3706}$) and MovieLens~20M ($R~\in~\mathbb{R}^{138493 \times 27032}$) datasets, which are broadly used and stable benchmark datasets. MovieLens is a non-commercial website for movie recommendations. 
In MovieLens 1M, we chose movies' genres (resp. age-gender combinations) as item (resp. user) communities.
\end{itemize}

\textbf{Train-test setup:}
For each dataset, we split the set of observed entries (uniformly at random) into a training set (85 \%), a validation set (10\%) and a test set (5\%).

\subsubsection{Results}
Table~\ref{resultstable} summarizes the results of the real-world data experiments. We evaluated the performance of BOMIC, BOMIC+, SI and IMCNF on all datasets above. For BOMIC+, instead of using the side information directly, we performed unsupervised clustering of the corresponding features to translate them into community information, which we then used as the $X,Y$ in the BOMIC+ algorithm in Section~\ref{bomic+}.
Observe that BOMIC+ has the lowest RMSE on all datasets and largest SPC in two datasets, whilst BOMIC has the best SPC on the MovieLens dataset. It is important to highlight that the standard BOMIC also beat the baselines.  One interesting aspect of using BOMIC+ is that the unsupervised clustering step reduces the dimensionality of the side information, which can have a positive regularizing effect.

\subsubsection{Illustration of OMIC's interpretability}

\begin{figure*}
	\centering
	\includegraphics[width=.97\linewidth]{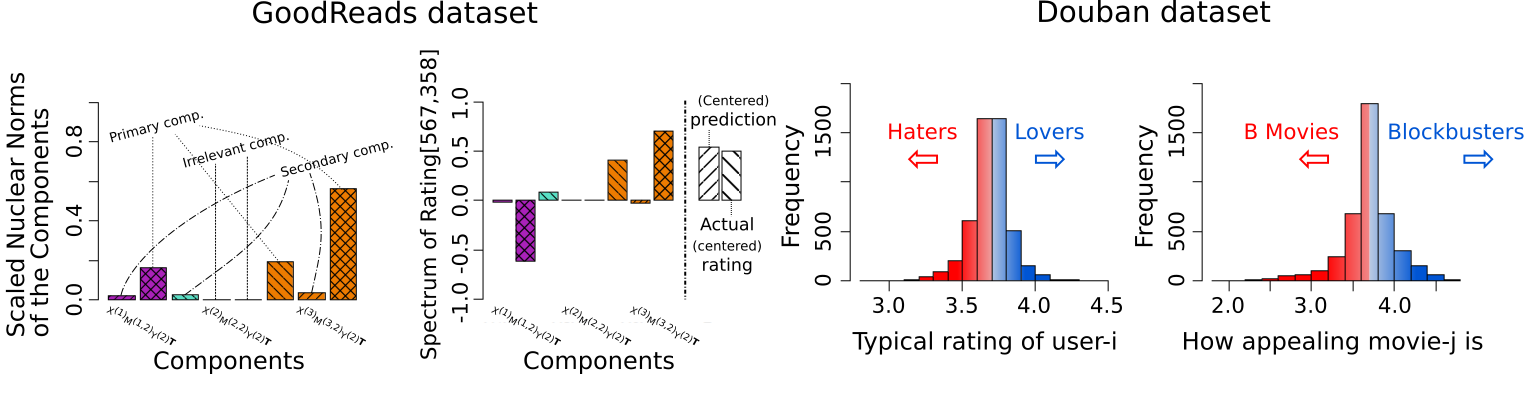}
	\caption{
		The first two graphs show the relative influence of the BOMIC+ components on the predictions of the whole matrix and one individual entry respectively.
		The last two graphs show the distribution of the users and item biases obtained by BOMIC on the Douban dataset. }
	\label{interpretability}
\end{figure*}

As explained above, one advantage of our method is that it can provide partial explanations for its predictions: each prediction is a sum of terms coming from each of the components of the model. Furthermore, this sum is uniquely determined, since the components of the model live in mutually orthogonal spaces and correspond to distinct intuitive phenomena. For example, if some auxiliary vectors are constructed from user community side information, the algorithm can disentangle the users' particular tastes from those of their respective community. In particular, OMIC can discover facts about community-wide behavior.

We illustrate those effects in Figures~\ref{interpretability} and \ref{fig:mlenspic}. On the left of Figure~\ref{interpretability}, 
we show the norms of each of the components of the recovered matrix: our recovered matrix takes the form  $R=\sum_{k,l\leq 3}X^{(k)}\hat{M}^{(k,l)}(Y^{(l)})^\top$ where the  $\hat{M}^{(k,l)}$ are obtained as the solution to our optimization problem~\eqref{theopt}, and each component in $X^{(k)}\hat{M}^{(k,l)}(Y^{(l)})^\top$ in the sum corresponds to an interpretable concept. For instance, $X^{(2)}\hat{M}^{(2,1)}(Y^{(1)})^\top$ correspond to user community biases. The norms of each component can give us an idea of how important each component is globally. 
Thus we see that over the whole GoodReads dataset, the most important components (excluding the global bias) are: (1) the specific match between the user and the book, (2) user generosity and (3) the quality of each book. These results are intuitively natural and expected. 

The second picture presents an explanation for an individual prediction. In other words, we chose one entry of $R$ (say, $R_{i,j}$) and represented the corresponding entry of each of the above mentioned components: for instance, the orange bar to the right of the graph represents the entry $(X^{(3)}\hat{M}^{(3,3)}(Y^{(3)})^\top)_{i,j}$, which corresponds to the same rating. This number represents the part of the rating $(i,j)$ which is attributable to a specific preference of the user for the specific movie (discounting the parts of this preference which are shared by the other members of that user's community or other movies of the same genre).

Here, the book is not generally considered good by the users (cf. large negative component corresponding to the purple bar), but the individual is usually generous (first orange bar), and the specific book and user are a good match for each other (cf. large orange component corresponding to the rightmost bar). 

Note that what is interesting here is that our model was specifically trained in a way that treats each of the components as a separate entity, with its own cross-validated hyperparmeter, so that the decomposition along those components is and intertwined with the optimization process (rather than collected as a statistic after applying a standard matrix completion method).

In the last two graphs, we show the distribution of user biases and movie quality in the Douban dataset. The distribution is similar to a normal distribution (squished at the boundaries), allowing us to characterize the users (resp. movies) on a spectrum between haters and lovers (resp. B-movies and blockbusters).

Figure~\ref{fig:mlenspic} shows bar charts illustrating the affinities between user communities (gender-age combinations) and four movie genres in the MovieLens dataset. Note that these affinity scores are part of our model and can be directly read from the component $X^{(2)}M^{(2,2)}({Y^{(2)}})^\top$ in BOMIC+. We observe that BOMIC+ is able to detect noteworthy human behaviour. For instance, female users tend to prefer drama and romance while male users enjoy comedies and thrillers instead. Note also that the biases vary with users' ages: older male users like romance movies more than their younger counterparts.

\begin{table*}[]
	\caption{Performance comparison of our methods vs baselines on the real datasets}
	\label{resultstable}
	\centering
	\resizebox{\textwidth}{!}{
		\begin{tabular}{|c|c|c|c|c|c|c|c|c|c|}
			\hline
			\multirow{2}{*}{Dataset} & \multirow{2}{*}{$P_\Omega$} & \multicolumn{4}{c|}{RMSE}                                   & \multicolumn{4}{c|}{SPC}                                    \\ \cline{3-10} 
			&                     & BOMIC  & BOMIC+                           & SI     & IMVNF  & BOMIC  & BOMIC+                           & SI     & IMVNF  \\ \hline
			Amazon                   & 0.0001              & \textbf{1.0406} & - & 1.0625 & - & \textbf{0.4121} & - & 0.4110 & - \\ \hline
			Douban                   & 0.0195              & 0.7886 & \textbf{0.7510} & 0.8797 & 0.8034 & 0.6280 & \textbf{0.6457} & 0.5760 & 0.6017 \\ \hline
			GoodReads                & 0.0331              & 1.0736 & \textbf{1.0540} & 1.0991  & 1.0770 & 0.5113 & \textbf{0.5120} & 0.4857 & 0.5052 \\ \hline
			MovieLens1M                & 0.0446              & 0.8534 & \textbf{0.8455} & 0.8838  & 0.8559  & \textbf{0.6368} & 0.6336 & 0.6289 & 0.6321 \\ \hline
			MovieLens20M                  & 0.0051              & \textbf{0.7803} & - & 0.8025 & - & \textbf{0.6697} & - & 0.6521 & - \\ \hline
		\end{tabular}
	}
\end{table*}

\normalsize

\begin{figure*}
	\centering
	\includegraphics[width=.97\linewidth]{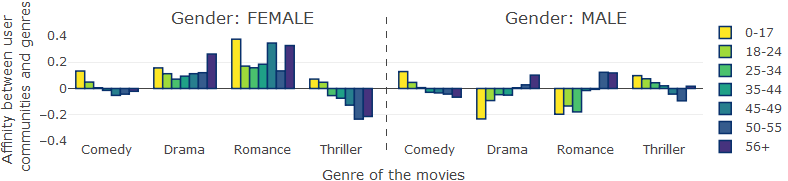}
	\caption{
		Affinity between user communities (grouped  by gender and age) and movie genres for MovieLens. These biases can be directly read from the component $X^{(2)}M^{(2,2)}({Y^{(2)}})^\top$ in BOMIC+.}
	\label{fig:mlenspic}
\end{figure*}

\section{Conclusion}
In this paper, we introduced OMIC, a matrix completion framework which relies on orthogonal auxiliary matrices to guide the model in privileged directions corresponding to prior knowledge. This simultaneously allows us to recover interpretable information about the predicted behavior. Our algorithm includes, as  particular cases, three models (BOMIC, OMIC+ and BOMIC+) which can train user and item biases (and/or a community component) jointly with a nuclear norm minimization strategy. Finally, synthetic and real-data  experiments demonstrate our algorithm's superior ability to adapt to and interpret different qualitative behaviors in the data.

\section*{Acknowledgements}

We thank the reviewers for their helpful comments. We warmly thank Luís Augusto Weber Mercado for substantial help in designing the figures of the present paper.  We also thank Rob Vandermeulen for helpful comments.  The authors acknowledge support by the German Research Foundation (DFG) award KL 2698/2-1 and by the Federal Ministry of Science and Education (BMBF) awards 031L0023A, 01IS18051A, and 031B0770E. The simulations were partly executed on the high performance cluster “Elwetritsch II” at the TU Kaiserslautern (TUK) which is part of the “Alliance for High Performance Computing Rheinland-Pfalz”(AHRP). We kindly acknowledge the support of the RHRK especially when using their DGX-2

\bibliographystyle{IEEEtran}

\bibliography{TNNLS-2020-P-15336.R2-bibliography}

\appendices

\renewcommand{\theequation}{S.\arabic{equation}}

\section{An alternative formulation of the optimization problem }
\label{maxmargin}
Instead of a nuclear-norm minimization algorithm, our optimization problem~\eqref{theopt} can be equivalently formulated as below.
\begin{theorem}
	\label{equivalentprob}
	The optimization problem~\eqref{theopt} is equivalent to the following optimization problem: 
	\begin{align}
	\minimize \quad &\mathcal{L}\left(R_{\Omega},\Lambda,\{U^{(k,l)},V^{(k,l)}\}_{k\leq K,l\leq L}\right)\>  \text{with} \nonumber  \\
	&\mathcal{L}(R_{\Omega},M,\Lambda)\nonumber \\&=\left\|P_{\Omega}(R)-P_{\Omega}\left(\sum_{k=1,l=1}^{K,L} U^{(k,l)}(V^{(k,l)})^\top\right)\right\|_{\Fr}^2\nonumber \\ &  \quad + \sum_{k,l} \lambda_{k,l} \left(\|U^{(k,l)}\|_{\Fr}^2+\|V^{(k,l)}\|_{\Fr}^2\right),  \nonumber\end{align}
	subject to ($\forall k,l$):
	\small
	\begin{align}
	& \spn((U^{(k,l)})_{\nbull,i} :    i\leq d^{(k)}_1)\subset \spn((X^{(k)})_{\nbull,i}: i\leq d^{(k)}_1); \nonumber  \\
	& \spn((V^{(k,l)})_{\nbull,i} :    i\leq d^{(l)}_2)\subset \spn((Y^{(l)})_{\nbull,i}: i\leq d^{(l)}_2). \quad\nonumber 
	\end{align}
\end{theorem}

For the proof, we will need the following lemma (lemma 6 from~\cite{softimpute}, see also~\cite{fazel,Srebro}):
\begin{lemma}
	\label{nuclearfro}
	For any matrix $Z$, the following holds: 
	\begin{align}
	\|Z\|_{*}=\min_{U,V;\atop UV^\top=Z}\|U\|_{\Fr}\|V\|_{\Fr} \\ =\min_{U,V;\atop UV^\top=Z} \frac{1}{2}\left(\|U\|_{\Fr}^2+\|V\|_{\Fr}^2\right)
	\end{align}
\end{lemma}
\begin{proof}
	By Lemma~\ref{nuclearfro}, we have that the optimization problem~\eqref{theopt} is equivalent to the following: 
	\begin{align}
	&\minimize  \quad  \mathcal{L}\left(R_{\Omega},\Lambda,\{U^{(k,l)},V^{(k,l)}\}_{k\leq K,l\leq L}\right)\>  \text{with} \nonumber  \\
	&\mathcal{L}(R_{\Omega},M,\Lambda)\nonumber \\&=\left\|P_{\Omega}(R)-P_{\Omega}\left(\sum_{k=1,l=1}^{K,L}X^{(k)}M^{(k,l)}(Y^{(l)})^{\top} \right)\right\|_{\Fr}^2\nonumber \\ &  \quad + \sum_{k,l} \lambda_{k,l} \left(\|U^{(k,l)}\|_{\Fr}^2+\|V^{(k,l)}\|_{\Fr}^2\right),  \quad  \nonumber \\
	&\text{subject to}\quad  M^{(k,l)}=U^{(k,l)}(V^{(k,l)})^\top \quad \forall k,l.
	\end{align}
	Now, note that if for any $(k,l)$,  $M^{(k,l)}=U^{(k,l)}(V^{(k,l)})^\top $ and $Z^{(k,l)}=X^{(k)}M^{(k,l)}(Y^{(l)})^{\top} $, then $Z=(X^{(k)} U^{(k,l)})(Y^{(l)}V^{(k,l)})^\top =\tilde{U}^{(k,l)}(\tilde{V}^{(k,l)})^\top $, where $\tilde{U}^{(k,l)}=(X^{(k)} U^{(k,l)})$ and $\tilde{V}^{(k,l)}=(Y^{(l)}V^{(k,l)})$. Furthermore, for any matrix $A\in \mathbb{R}^{n_1\times n_2}$ and any orthogonal matrix $B\in \mathbb{R}^{n_0\times n_1}$ (resp. $C\in \mathbb{R}^{n_2\times n_3}$), 
	\begin{align}
	\label{Rotate}
	\|A\|_{\Fr}= \|BA\|_{\Fr}=\|AC\|_{\Fr}=\|BAC\|_{\Fr}.
	\end{align}
	
	Thus we have $\|\tilde{U}^{(k,l)}\|_{\Fr}=\|\tilde{X}^{\top}\tilde{U}^{(k,l)}\|_{\Fr}=\|U^{(k,l)}\|_{\Fr}$, where $\tilde{X}$ is a matrix whose first $d_1^{k}$ columns are identical to those of $X^{(k)}$, and whose columns form an orthonormal basis of $\mathbb{R}^{m}$. Similarly,  $\|\tilde{V}^{(k,l)}\|_{\Fr}=\|V^{(k,l)}\|_{\Fr}$. Furthermore, conversely, if we can write a matrix $Z$ as $\tilde{U}^{(k,l)}(\tilde{V}^{(k,l)})^\top $ for some $\tilde{U}^{(k,l)}$ and $\tilde{V}^{(k,l)}$ whose columns are in the span of the columns of $X^{(k)}$ and $Y^{(l)}$ respectively, then we can write $Z=(X^{(k)} U^{(k,l)})(Y^{(l)}V^{(k,l)})^\top =\tilde{U}^{(k,l)}(\tilde{V}^{(k,l)})^\top $ where \\$[U^{(k,l)}]_{i,j}=[(\tilde{X}^{(k)})^\top\tilde{U}^{(k,l)}]_{i,j} \quad \forall i,j \quad \text{s.t.} \quad i\leq d_1^{(k)}$.
	The theorem follows. 
\end{proof}
\section{Proof of uniqueness of decomposition}
\label{Uniqueness_proof}

\begin{proposition}
	\label{prop:unique}
	Let $\mathcal{F}_{k,l}=\{R:\exists M\in \mathbb{R}^{d^{1}_k\times d^2_{l}}: R=X^{(k)}M(Y^{(l)})^\top\}$ denote the $KL$ subspaces corresponding to each pair of auxiliary matrices $(X^{(k)},Y^{(l)})$.
	Those vector spaces $\mathcal{F}_{k,l}$ are orthogonal (w.r.t. the Frobenius inner product) and their direct sum is the whole of $\mathbb{R}^{m\times n}$: 
	\begin{align}
	\label{itsasum}
	\bigoplus \mathcal{F}_{k,l}=\mathbb{R}^{m\times n}.
	\end{align}
	In particular, for any $R\in \mathbb{R}^{m\times n}$, there exist a unique collection of matrices $R^{(k,l)}\in \mathcal{F}_{k,l}$ such that $R=\sum_{k,l} R^{(k,l)}$. 
	In fact, 
	\begin{align}
	\label{eq:Rsum}
	R^{(k,l)}=(X^{(k)})^\top R Y^{(l)}.
	\end{align}
\end{proposition}

\begin{proof}
	We divide the proof into two parts: the proof that the subspaces are mutually orthogonal, and the proof that their direct sum is the whole of the matrix space.
	
	\textbf{The subspaces $\mathcal{F}^{k,l}$ are mutually orthogonal.}
	Let $A\in \mathcal{F}^{k,l}$ and $B\in \mathcal{F}^{k',l'}$ where either $k\neq k'$ or $l\neq l$. 
	By definition of the subspaces in question, there exist $M\in \R^{d^1_k\times d^2_l}$ and $N\in \R^{d^1_{k'}\times d^2_{l'}}$ such that $A=X^{(k)}M(Y^{(l)})^\top$ and $B=X^{(k')}N(Y^{(l')})^\top$. 
	
	We now compute the (Frobenius) inner product between A and B in both cases.

	\textbf{Case 1:} $l\neq l'$, in which case by the assumption on $Y^{(1)},\ldots Y^{(l)}$, we have $(Y^{(l)})^\top Y^{(l')}=\mathbf{0}\in \mathbb{R}^{d^2_{l}\times d^2_{l'}}$. Then, 
	\begin{align*}
	\langle A, B\rangle & = \Tr\left(X^{(k)}M(Y^{(l)})^\top  (X^{(k')}N(Y^{(l')})^\top)^\top\right) \\
	&= \Tr\left(X^{(k)}M(Y^{(l)})^\top Y^{(l')}N^\top (X^{(k')})^\top\right)\\
	&= \Tr\left(X^{(k)}M \mathbf{0} N^\top (X^{(k')})^\top\right) \\
	&=0
	\end{align*}
	\textbf{Case 2:} $k\neq k'$
	\begin{align*}
	\langle A, B\rangle &= \Tr\left((X^{(k')}M(Y^{(l)})^\top)^\top X^{(k)}N(Y^{(l)})^\top\right) \\
	&=\Tr\left(Y^{(l)}M^\top (X^{(k')})^\top X^{(k)}M(Y^{(l)})^\top\right)\\
	&= \Tr\left(Y^{(l)}M^\top \mathbf{0} M(Y^{(l)})^\top\right)=0,
	\end{align*}
	as expected.

	\textbf{The direct sum is the whole of $\mathbb{R}^{m\times n}$: $\bigoplus_{k,l}\mathcal{F}^{k,l}=\mathbb{R}^{m\times n}$.}
	
	Let $R\in\R^{m\times n}$. For each column vector $v\in \R^{m}$ we have immediately $v=\sum_{k}X^{(k)}(X^{(k)})^\top v$ by our assumption on the $X$'s. Applying this to each column of $R$, we have 
	$R=\sum_{k}X^{(k)}(X^{(k)})^\top R$. Similarly, $R=\sum_{l} RY^{(l)}(Y^{(l)})^\top$. Plugging the second equation into the first one, we obtain \begin{align*}
	R&=\sum_{k}X^{(k)}(X^{(k)})^\top R \\  R&=\sum_{k}X^{(k)}(X^{(k)})^\top \sum_{l} R(Y^{(l)}(Y^{(l)})^\top)\\
	&=\sum_{k,l} X^{(k)}\left[(X^{(k)})^\top RY^{(l)}\right](Y^{(l)})^\top\\&\in \bigoplus_{k,l} \mathcal{F}_{k,l},
	\end{align*}
	as expected (this also proves equality~\eqref{eq:Rsum}).
\end{proof}

\section{Proof of convergence of our OMIC algorithm}
\label{proofconv}

In this section, we prove Theorems~\ref{convergence} and~\ref{rate}. The proofs rely mostly on adaptations of the techniques from~\cite{softimpute}, together with extensive use of the rotational invariance of the Frobenius and nuclear norms, as well as the linear independence of the spaces corresponding to each side information pairs.

Recall the optimization algorithm we propose to solve is the following one (cf. equations~\eqref{theopt})

\begin{align}
&\minimize  \quad \mathcal{L}(R_{\Omega},M,\Lambda)\> \quad \quad \quad \text{with} \\
&\mathcal{L}(R_{\Omega},M,\Lambda)=\sum_{k,l} \lambda_{k,l}\|M\|_{*}\nonumber \\
&+\frac{1}{2} \left\|P_{\Omega}(R)-P_{\Omega}\left(\sum_{k=1,l=1}^{K,L} X^{(k)}M^{(k,l)}(Y^{(l)})^{\top}\right)\right\|_{\Fr},\nonumber 
\end{align}

where $P_{\Omega}$ is the projection operator on the set of observed entries: i.e., if an entry is not observed, it is set to zero;
if an entry is observed $p$ times, any Frobenius norm counts that entry $p$ times. Here, the output is $(M^{(k,l)})_{k\leq K,l\leq L}$ and $Z=\sum_{k=1,l=1}^{K,L} X^{(k)}M^{(k,l)}(Y^{(l)})^{\top}$.

First, let us finish the proof of the fully-known case: 
\begin{proof}[Proof of Proposition~\ref{fullyknown}]
	Equation~\eqref{SingThreshNew} follows from the fact that $M^{(k,l)}$ in the decomposition is unique and determined by the formula $M^{(k,l)}=(X^{(k)})^\top Z Y^{(l)} $. This itself follows from the orthogonality of the side information matrices after multiplying each side of equation~\eqref{decomp} by $(X^{(k)})^\top$ on the left and $Y^{(l)} $ on the right. The equivalence between the next two problems also follows. 
	
	As to the fact that $S_{\Lambda}(Z)$ is the solution to problem~\eqref{thenewprob}, let us first note that the case $K=L=1$ with identity side information is just lemma 1 in~\cite{softimpute}. 
	
	Now, note that 
	\begin{align}
	&\|\tilde{Z}-Z\|_{\Fr}^2\nonumber \\&=\sum_{k,l} \|X^{(k)}M^{(k,l)}(Y^{(l)})^\top -X^{(k)}\tilde{M}^{(k,l)}(Y^{(l)})^\top\|_{\Fr}^2   \nonumber \\
	&=\sum_{k,l} \|M^{(k,l)}-\tilde{M}^{(k,l)}\|_{\Fr}^2,
	\end{align}
	where at the first equality, we have used the orthogonality of the terms of the sum with respect to the Frobenius inner product,  at the second equality, we have used the rotational invariance of the Frobenius norm. Here  $\tilde{M}^{(k,l)}=(X^{(k)})^\top Z Y^{(l)}$, so that $Z=\sum_{k,l} X^{(k)}\tilde{M}^{(k,l)}(Y^{(l)})^\top$. 

	Using this, we can reformulate the problem~\eqref{thenewprob} as follows: 
	\begin{align}
	\label{decomposable}
	\minimize \quad  &\sum_{k,l}\frac{1}{2}\|M^{k,l}-(X^{(k)})^\top Z Y^{(l)} \|_{\Fr}^2\nonumber \\&\quad \quad \quad \quad + \sum_{k=1}^K\sum_{l=1}^L\lambda_{k,l} \left\|M^{(k,l)} \right\|_{*},
	\end{align}
	which can be solved as $KL$ independent optimization problems, with the solution corresponding to index $(k,l)$ being given by $M^{(k,l)}=S_{\lambda_{k,l}}((X^{(k)})^\top Z Y^{(l)})$, by an application of lemma 1 from~\cite{softimpute}. 
	The theorem follows.
	
\end{proof}

Then, let us dispose with the following straightforward observation: 

\begin{lemma}
	\label{contraction}
	The generalized singular value thresholding operator $S_{\Lambda}$ satisfies, for any two matrices $Z_1,Z_2\in \mathbb{R}^{m\times n}$, 
	\begin{align}
	\left\| S_{\Lambda}\left(Z_1\right)  -S_{\Lambda}\left(Z_2\right)        \right\|_{\Fr}\leq   \left\|   Z_1-Z_2 \right\|_{\Fr},
	\end{align}
	and in particular, $S_{\Lambda}(\nbull)$ is a continuous map.
\end{lemma}
\begin{proof}
	This follows from the corresponding lemma 3 in~\cite{softimpute}, together with the definition of the operator $S_{\Lambda}$: 
	\begin{align}
	&\left\|	S_{\Lambda}(Z_1)-	S_{\Lambda}(Z_2)\right\|_{\Fr}^2\nonumber \\
	&=\bigg\|  \sum_{k=1}^K\sum_{l=1}^L X^{(k)}S_{\lambda_{k,l}}\left( (X^{(k)})^\top Z_1 Y^{(l)}  \right)(Y^{(l)})^\top\nonumber \\&\quad\quad -\sum_{k=1}^K\sum_{l=1}^L X^{(k)}S_{\lambda_{k,l}}\left( (X^{(k)})^\top Z_2 Y^{(l)}  \right)(Y^{(l)})^\top           \bigg\|_{\Fr}^2\nonumber  \\
	&=\left\|  \sum_{k=1}^K\sum_{l=1}^L X^{(k)}S_{\lambda_{k,l}}\left( (X^{(k)})^\top (Z_1-Z_2) Y^{(l)}  \right)(Y^{(l)})^\top               \right\|_{\Fr}^2\nonumber \\
	&=\sum_{k=1}^K\sum_{l=1}^L\left\|S_{\lambda_{k,l}}\left( (X^{(k)})^\top (Z_1-Z_2) Y^{(l)}  \right)\right\|_{\Fr}^2\nonumber\\&\leq \sum_{k=1}^K\sum_{l=1}^L\left\| (X^{(k)})^\top (Z_1-Z_2) Y^{(l)} \right\|_{\Fr}^2\nonumber \\
	&=\|Z_1-Z_2\|_{\Fr}^2,
	\end{align}
	where at the fourth line, we have used Lemma 3 from~\cite{softimpute}.
	
\end{proof}

Now, let us define the quantity \begin{align*}
Q(A|B)&=\frac{1}{2}\|P_{\Omega}(R)+P_{\Omega^\perp}(B)-A\|_{\Fr}^2\nonumber \\&\quad \quad \quad \quad \quad +\sum_{k,l}\lambda_{k,l}\|(X^{(k)})^\top A Y^{(l)}\|_{*}.
\end{align*}

We have that the loss $\mathcal{L}(Z)$ corresponding to a matrix $Z$ can be written $Q(Z|Z)$. Furthermore, let us define $Z^{i+1}=\argmin_{Z}Q(Z|Z^{i})$ (since this is an instance of the fully known case, the solution is unique and given by the operator $S_\Lambda$ above).
We now have the following lemma, which shows that the loss decreases monotonically with $i$: 

\begin{lemma}
	\label{monotone}
	Define the sequence $Z^{i}$ by $Z^{i+1}=\argmin_{Z}Q(Z,Z^{i})$ (with any starting point, for instance $Z^0=0$), which is equivalent to definition~\eqref{defseq}. 
	We have 
	\begin{align}
	\label{monotoneq}
	\mathcal{L}(Z^{i+1})\leq Q(Z^{i+1}|Z^{k})\leq 	\mathcal{L}(Z^{i}).
	\end{align}
\end{lemma}
\begin{proof}
	The proof is almost the same as the proof of Lemma 2 in~\cite{softimpute}. 
	We have 
	\begin{align}
	&	\mathcal{L}(Z^{i})\nonumber \\&=Q(Z^{i}|Z^{i}) \nonumber \\
	&=\frac{1}{2}\|R_\Omega+P_{\Omega^\perp}(Z^{i})-Z^{i}\|_{\Fr}^2\nonumber \\& \quad \quad \quad \quad \quad \quad\quad \quad \quad\quad  +\sum_{k,l}\lambda_{k,l}\|(X^{(k)})^\top Z^{i}Y^{(l)}\|_{*}\nonumber \\
	&\geq\min_{Z}\frac{1}{2}\|R_{\Omega}+P_{\Omega^\perp}(Z^{i})-Z\|_{\Fr}^2\nonumber \\& \quad \quad \quad \quad \quad \quad\quad \quad \quad\quad+\sum_{k,l}\lambda_{k,l}\|(X^{(k)})^\top ZY^{(l)}\|_{*}\nonumber \\
	&=Q(Z^{i+1}|Z^{i})\nonumber \\
	&= \frac{1}{2} \|(R_{\Omega}-P_{\Omega}(Z^{i+1}) \nonumber \\& \quad \quad \quad \quad \quad \quad \quad\quad +(P_{\Omega^\perp}(Z^i)-P_{\Omega^\perp}(Z^{i+1}))   \|_{\Fr}^2\nonumber \\& +\sum_{k,l}\lambda_{k,l}\|(X^{(k)})^\top Z^{i+1}Y^{(l)}\|_{*}\nonumber \\
	&= \frac{1}{2} \left\|(R_{\Omega}-P_{\Omega}(Z^{i+1})\right\|_{\Fr}^2\nonumber \\& \quad \quad \quad \quad \quad +\frac{1}{2}\left\|(P_{\Omega^\perp}(Z^i)-P_{\Omega^\perp}(Z^{i+1}))   \right\|_{\Fr}^2 \nonumber \\
	&\quad\quad \quad \quad \quad \quad \quad +\sum_{k,l}\lambda_{k,l}\|(X^{(k)})^\top Z^{i+1}Y^{(l)}\|_{*}\nonumber \\
	&\geq \frac{1}{2} \left\|(R_{\Omega}-P_{\Omega}(Z^{i+1})\right\|_{\Fr}^2\nonumber \\& \quad \quad \quad \quad \quad \quad\quad \quad +\sum_{k,l}\lambda_{k,l}\|(X^{(k)})^\top Z^{i+1}Y^{(l)}\|_{*}\nonumber \\
	&=Q(Z^{i+1},Z^{i+1})=\mathcal{L}(Z^{i+1}).
	\end{align}
\end{proof}
Next, we have the following lemma: 

\begin{lemma}
	\label{decreasefrob1}
	The sequence $\|Z^{i}-Z^{i-1}\|_{\Fr}$ is monotone decreasing: 
	\begin{align}
	\label{decreasefrob}
	\|Z^{i}-Z^{i+1}\|_{\Fr}\leq \|Z^{i}-Z^{i-1}\|_{\Fr}.
	\end{align}
	Furthermore, 
	\begin{align}
	\label{secondstatement}
	Z^{i}-Z^{i+1} \rightarrow 0 \quad \text{as} \quad i\rightarrow \infty.
	\end{align}
\end{lemma}

\begin{proof}
	We have 
	\begin{align}
	&\|Z^{i}-Z^{i+1}\|_{\Fr}^2\nonumber \\&=\| S_{\Lambda}\left(P_{\Omega^\perp}(Z^{i-1})+R_{\Omega}\right)   -S_{\Lambda}\left(P_{\Omega^\perp}(Z^{i})+R_{\Omega}\right)    \|_{\Fr}^2\nonumber\\
	&\leq \|\left(P_{\Omega^\perp}(Z^{i-1})+R_{\Omega}\right)   -\left(P_{\Omega^\perp}(Z^{i})+R_{\Omega}\right)    \|_{\Fr}^2\nonumber\\
	& \quad \quad\quad \text{By Lemma~\ref{contraction}}\nonumber\\
	&= \|P_{\Omega^\perp}(Z^{i-1}) -P_{\Omega^\perp}(Z^{i}) \|_{\Fr}^2 \label{middle}\\
	&\leq \|Z^{i}-Z^{i-1}\|_{\Fr}^2,\label{last}
	\end{align}
	which proves the first statement~\eqref{decreasefrob}.
	
	As for the second statement~\eqref{secondstatement}, it will follow from the following two claims: 
	
	\textit{Claim 1:} $P_{\Omega}(Z^{i}-Z^{i+1})\rightarrow 0$.
	\textit{Claim 2:} $P_{\Omega^\perp}(Z^{i}-Z^{i+1})\rightarrow 0$.
	
	\textit{Proof of Claim 1:}
	Note that by inequality~\eqref{decreasefrob}, the sequence $\|Z^{i}-Z^{i+1}\|_{\Fr}$ must converge. In particular, $\|Z^{i}-Z^{i+1}\|_{\Fr}-\|Z^{i}-Z^{i-1}\|_{\Fr}\rightarrow 0$, and by inequalites~\eqref{middle} and~\eqref{last}, $$\|P_{\Omega^\perp}(Z^{i-1}) -P_{\Omega^\perp}(Z^{i}) \|_{\Fr}- \|Z^{i}-Z^{i-1}\|_{\Fr}\rightarrow 0,$$
	from which we conclude that $$\|P_{\Omega}(Z^{i})-P_{\Omega}(Z^{i+1})\|_{\Fr}^2\rightarrow 0.$$ Claim 1 follows. 
	
	\textit{Proof of Claim 2:} 
	We know by inequality~\eqref{monotoneq} that $\mathcal{L}(Z^{i})$ must converge, and thus $\mathcal{L}(Z^{i})-\mathcal{L}(Z^{i+1})\rightarrow 0$, from which it follows that \begin{align}
	\label{bla}
	Q(Z^{i+1}|Z^i)  - Q(Z^{i+1}|Z^{i+1})\rightarrow 0 .
	\end{align} Now, 
	\begin{align}
	&Q(Z^{i+1}|Z^i)  - Q(Z^{i+1}|Z^{i+1})\nonumber\\
	&=\frac{1}{2}\|R_{\Omega}+P_{\Omega^\perp}(Z^i)-Z^{i+1}\|_{\Fr}^2
	\nonumber \\& \quad \quad \quad \quad \quad \quad\quad \quad +\sum_{k,l}\lambda_{k,l}\|(X^{(k)})^\top Z^{i+1} Y^{(l)}\|_{*}\nonumber\\&-\frac{1}{2}\|R_{\Omega}+P_{\Omega^\perp}(Z^{i+1})-Z^{i+1}\|_{\Fr}^2
	\nonumber \\& \quad \quad \quad \quad \quad \quad \quad\quad-\sum_{k,l}\lambda_{k,l}\|(X^{(k)})^\top Z^{i+1} Y^{(l)}\|_{*}\nonumber\\
	&	=\frac{1}{2}\|P_{\Omega^\perp}(Z^{i+1})-P_{\Omega^\perp}(Z^{i})\|_{\Fr}^2, 
	\end{align}
	which, together with~\eqref{bla}, implies claim 2.
\end{proof}

The next step is to prove that each limit point of the sequence $Z^{i}$ is a solution to the optimization problem~\eqref{theopt}. To prove this, we will need the following lemma: 
\begin{lemma}
	\label{little5}
	Let $Z_{n_i}\rightarrow Z^{\infty}$ be a convergent subsequence of $Z_{i}$. 
	
	Let $p_{n_i}\in \partial  \sum_{k,l}\|(X^{(k)})^\top Z^iY^{(l)}\|_{*}$ be a sequence of subgradients of our regularizer $\sum_{k,l}\| M^{(k,l)}\|_{*}$ evaluated at $Z^i$. There exists a convergent subsequence of $p_{m_i}$ which converges to some $$p\in  \partial  \sum_{k,l}\|(X^{(k)})^\top Z^{\infty}Y^{(l)}\|_{*},$$
	a subgradient of our regularizer, evaluated at the limit $Z^{\infty}$.
\end{lemma}

\begin{proof}
	First, recall from~\cite{subdifnu} and~\cite{softimpute} that the set of subgradients of the nuclear norm of a matrix $A$ is given by $$\partial \|A\|_{*}=\left\{UV^\top+W,U^\top W=0=VW^\top,\|W\|_{\sigma}\leq 1    \right\},$$
	where $UDV^\top$ is the SVD of the matrix $A$. 
	Using the chain rule and the fact that the side information matrices $X^{(k)},Y^{(l)}$ are constant, we can calculate the set of subgradients of our regularizer evaluated at both $Z^i$ and $Z^\infty$ as follows: 
	\begin{align}
	\label{element}
	& \quad \quad \quad \quad \quad 	\partial \sum_{k,l}\|(X^{(k)})^\top Z^iY^{(l)}\|_{*} \\&= \Bigg\{  \sum_{k,l}U_{k,l}^i(V_{k,l}^i)^\top+W_{k,l}^i,(U_{k,l}^i)^\top W_{k,l}^i=0,\nonumber \\&\quad \quad\quad \quad  \quad\quad \quad \quad  \quad W_{k,l}^iV_{k,l}^i=0,\|W_{k,l}^i\|_{\sigma}\leq 1      \Bigg\} \nonumber 
	\end{align}
	and
	\samepage{
		\begin{align}
		\label{elementagain}
		&\quad \quad  \quad \quad \quad \quad \partial \sum_{k,l}\|(X^{(k)})^\top Z^iY^{(l)}\|_{*}\\&= \Bigg\{  \sum_{k,l}U_{k,l}V_{k,l}^\top+W_{k,l},U_{k,l}^\top W_{k,l}=0,\nonumber \\&\quad \quad\quad \quad  \quad\quad \quad \quad  \quad W_{k,l}V_{k,l}=0,\|W_{k,l}\|_{\sigma}\leq 1     \Bigg\} ,\nonumber 
		\end{align}
		where $U_{k,l}D_{k,l}V_{k,l}^\top$  (resp. $U_{k,l}^iD_{k,l}^i(V_{k,l}^i)^\top$) is the singular value decomposition of  $(X^{(k)})^\top Z^\infty Y^{(l)}$(resp. $(X^{(k)})^\top Z^iY^{(l)}$).}
	
	By compactness, there exists a subsequence $m_i$ of $n_i$ such that $W^{m_i}$ converges to a value $W$. By continuity of the spectral norm, we also have $\|W\|_{*}\leq 1$. Furthermore, it follows from the convergence of $Z^{n_i}$ (and in particular, of $Z^{m_i}$) to $Z^\infty$ that $\sum_{k,l}U_{k,l}^{m_i}(V_{k,l}^{m_i})^\top \rightarrow  \sum_{k,l}U_{k,l}(V_{k,l})^\top$. The result follows. 
	
\end{proof}

\begin{proposition}
	\label{lemma5soft}
	Every limit point of the sequence $(Z^{i})_{i\in \mathbb{N}}$  defined in~\eqref{defseq} is a stationary point of the loss function $\mathcal{L}(Z)=\frac{1}{2}\|\tilde{Z}-Z\|_{\Fr}^2+ \sum_{k=1}^K\sum_{l=1}^L\left\|(X^{(k)})^\top Z Y^{(l)}  \right\|_{*}$ defined in~\eqref{thenewprob1}. Hence, it is also a solution to the fixed point equation 
	\begin{align}
	Z=S_{\Lambda}\left( R_{\Omega}+P_{\Omega^\perp }(Z)  \right).
	\end{align}
\end{proposition}
\begin{proof}
	Let $Z^\infty$ be such a limit point. There exists a subsequence $Z^{n_i}$ such that 	$Z^{n_i}\rightarrow Z^\infty$.

	By Lemma~\ref{decreasefrob1} ,we have $Z^{n_i}-Z^{n_i-1}\rightarrow 0$, which by continuity of the operator $S_{\Lambda}$ implies that 
	\begin{align}
	\label{eqn:easystep}
	R_\Omega+P_{\Omega^\perp}(Z^{n_i-1})-Z^{n_i}\rightarrow R_{\Omega}-P_{\Omega}(Z^{\infty}).
	\end{align}

	Now, note that by definition of $Z^i$, \begin{align*}
	\forall i, \quad 0&\in \partial Q(Z|Z^{i-1})\\&=-(P_{\Omega}(R)+P_{\Omega}(Z^{i-1})-Z^i) \\&\quad \quad \quad + \partial  \sum_{k,l}\|(X^{(k)})^\top Z^iY^{(l)}\|_{*}.
	\end{align*}
	
	Thus, we can choose, for all $i$, a  $p_i\in \partial  \sum_{k,l}\|(X^{(k)})^\top Z^iY^{(l)}\|_{*}$ such that $p_i-(P_{\Omega}(R)+P_{\Omega}(Z^{i-1})-Z^i) =0$.
	Now, by Lemma~\ref{little5}, there exists a subsequence $Z^{m_i}$  of $Z^{n_i}$ such that $p_{m_i}\rightarrow p$ for some 
	\begin{align}
	\label{eqn:pingrad}
	p\in \partial  \sum_{k,l}\|(X^{(k)})^\top Z^\infty Y^{(l)}\|_{*}.
	\end{align} 
	
	Putting equations~\eqref{eqn:easystep} and~\eqref{eqn:pingrad} together, we obtain 
	\begin{align}
	0=p_{m_i}-(P_{\Omega}(R)+P_{\Omega}(Z^{m_i-1})-Z^{m_i}) \nonumber \\\rightarrow p-R_{\Omega}-P_{\Omega}(Z^{\infty}).
	\end{align}
	Thus, $0$ is a subgradient of $\mathcal{L}$ evaluated at $Z^\infty$. The first statement of the Proposition follows. 
	
	As for the second statement, note that \begin{align}
	\label{thissss}
	Z^{m_i}=S_{\Lambda}\left(R_{\Omega}+P_{\Omega^\perp}(Z^{m_i-1})\right).
	\end{align}
	
	Furthermore, by Lemma~\ref{decreasefrob1}, $Z^{m_i}- Z^{m_i-1}\rightarrow 0$, and therefore $Z^{m_i-1}\rightarrow Z^\infty$. Thus, using the continuity of the generalized singular value thresholding operator, we obtain by passing to the limits in~\eqref{thissss}:
	\begin{align}
	Z^\infty=S_{\Lambda}\left(   R_{\Omega}+P_{\Omega^\perp}(Z^{\infty})      \right),
	\end{align}
	as expected.
\end{proof}

\begin{proof}[Proof of Theorem~\ref{convergence}]
	In Proposition~\ref{lemma5soft}, we have already proved that any limit point of the sequence $(Z^{i})_{i\in \mathbb{N}}$  (defined in equation~\eqref{defseq}) is a stationary point of the loss function, and therefore a solution to the optimization problem~\eqref{theopt}. Thus, the only thing left to prove is that the sequence $(Z^{i})_{i\in \mathbb{N}}$ converges: indeed, if that is the case, its limit will be its (only) limit point, and will be a solution to problem~\eqref{theopt}.
	
	Let us first dispense with the following simple observation: by Lemma~\ref{monotone}, for any $i$, we have \begin{align}
	\label{compactnessargument}
	\mathcal{L}(Z^i)\leq \mathcal{L}(Z^0).
	\end{align}
	Since the objective function $\mathcal{L}$ is a continuous function of the matrix $Z$, the set of matrices $Z$ satisfying equation~\eqref{compactnessargument} is compact. Thus, by compactness, there exists at least one limit point $\bar{Z}$.

	Now, by the continuity of $S_{\Lambda}$ and the definition of $Z^i$, we have, for any $i$: 
	\begin{align}
	\label{lastkey}
	&\|\widebar{Z}-Z^i\|_{\Fr}^2\nonumber \\
	&=\left\|S_{\Lambda}\left(  R_\Omega +P_{\Omega^\perp}(\widebar{Z})\right)  -S_{\Lambda}\left(  R_\Omega +P_{\Omega^\perp}(Z^{i-1})\right)       \right\|_{\Fr}^2\nonumber \\
	&\leq \left\|\left(  R_\Omega +P_{\Omega^\perp}(\widebar{Z})\right)  -\left(  R_\Omega +P_{\Omega^\perp}(Z^{i-1})\right)       \right\|_{\Fr}^2\nonumber \\
	&=\|P_{\Omega^\perp}(\widebar{Z}-Z^{i-1})\|_{\Fr}^2\leq  \|\widebar{Z}-Z^{i-1}\|_{\Fr}^2,
	\end{align}
	where at the first line, we have used Proposition~\ref{lemma5soft} and the definition of $Z^i$.
	
	We will now show that the sequence $(Z^{i})_{i\in \mathbb{N}}$ actually converges to $\widebar{Z}$. To do this, we proceed by contradiction. Assume $Z^i$ doesn't converge towards $\bar{Z}$. By definition of convergence, this implies that there must exist an $\epsilon_*>0$ such that there exists an infinite subsequence $Z^{I_1},Z^{I_2},\ldots$ such that for all $i$, $\|Z^{I_i}-\widebar{Z}\|_{\Fr}\geq \epsilon_*$. Since the subsequence $Z^{I_1},Z^{I_2}, \ldots$ is contained in the compact set $$\mathcal{C}_{\epsilon_*}:=\left\{Z: \mathcal{L}(Z)\leq \mathcal{L}(Z^0) \> \land \> \|Z^{I_i}-\bar{Z}\|_{\Fr}\geq \epsilon_*\right\},$$ it must have a limit point $\widetilde{Z}\in \mathcal{C}_{\epsilon_*}$ inside that set. In particular, we have \begin{align}
	\label{eq:farapart}
	\|\widetilde{Z}-\widebar{Z}\|_{\Fr}\geq \epsilon_*
	\end{align}

	Set $\epsilon= \frac{\epsilon_*}{3}$. Since $\widebar{Z}$ is a limit point of $(Z^{i})_{i\in \mathbb{N}}$, there certainly exists an index $k$ such that 
	\begin{align}
	\label{eq:Zk}
	\|Z^{k}-\widebar{Z}\|_{\Fr}\leq \epsilon.
	\end{align}
	Since $\widetilde{Z}$ is also a limit point of $(Z^{i})_{i\in \mathbb{N}}$ (specifically, a limit point of the subsequence $(Z^{I_i})_{i\in \mathbb{N}}$), there exists an index $l$ such that $l>k$ and 
	\begin{align}
	\label{eq:Zltilde}
	\|Z^{l}-\widetilde{Z}\|_{\Fr}\leq \epsilon.
	\end{align}
	On the other hand, since $l>k$ by iteratively applying equation~\eqref{lastkey} $l-k$ times, we obtain: 
	\begin{align}
	\label{eq:Zlbar}
	\|\widebar{Z}-Z^l\|_{\Fr}\leq \|\widebar{Z}-Z^k\|_{\Fr}\leq \epsilon,
	\end{align}
	where the last inequality follows from equation~\eqref{eq:Zk}.
	
	Now, by equations~\eqref{eq:Zlbar} and~\eqref{eq:Zltilde} and the triangle inequality, we obtain: 
	\begin{align}
	\|\widetilde{Z}-\widebar{Z}\|_{\Fr}&\leq \|\widetilde{Z}-Z^{l}\|_{\Fr}+  \|Z^l-\widebar{Z}\|_{\Fr}\\&\leq \epsilon+\epsilon=2\epsilon=\frac{2\epsilon_*}{3} <\epsilon_*,
	\end{align}
	which is in contradiction with equation~\eqref{eq:farapart}. Thus, we deduce by contradiction that $Z^i$ indeed converges to its only limit point $\widebar{Z}$ (which we refer to as $Z^\infty$ in the rest of the appendix). As explained at the beginning of the proof, this, together with Proposition~\ref{lemma5soft}, implies Theorem~\ref{convergence}, as required.

\end{proof}

We can now proceed with the proof of our Theorem~\ref{rate} on the worst-case convergence.
\begin{proof}[Proof of Theorem~\ref{rate}]
	The proof is exactly the same as that of theorem 2 in~\cite{softimpute} (and also takes inspiration from~\cite{nesterov}), and we reformulate it into our notation here for the sake of completeness only. 
	
	For $\theta\in [0,1]$, we write $Z^i(\theta)$ for $(1-\theta)Z^i +\theta Z^\infty$.  Note that by convexity of our loss function $\mathcal{L}$, we have $\mathcal{L}(Z^i(\theta))\leq (1-\theta)\mathcal{L}(Z^i) +\theta \mathcal{L}(Z^\infty)$. 
	
	Note also that we have 
	\begin{align}
	&\|P_{\Omega^\perp}(Z^i-Z^i(\theta))\|_{\Fr}^2=\theta^2\|P_{\Omega^\perp}(Z^i-Z^\infty)\|_{\Fr}^2\nonumber \\ &\leq \theta^2 \|Z^i-Z^\infty\|_{\Fr}^2\leq \theta^2 \|Z^0-Z^\infty\|_{\Fr}^2,
	\end{align}
	where we have used Lemmas~\ref{decreasefrob1} and~\ref{contraction}.
	
	Using these facts and the definition in the construction of the sequence $Z^i$, we can derive the following key inequalities:  
	
	\begin{align}
	\label{thekey}
	&\mathcal{L}(Z^{i+1}) \nonumber\\&= \min_Z  \left[  \mathcal{L}(Z)+\frac{1}{2}\|Z-Z^i\|_{\Fr}^2    \right] \nonumber \\&\leq  \min_{\theta\in [0,1]}\left[  \mathcal{L}(Z^i(\theta))+\frac{1}{2}\|Z^i(\theta)-Z^i\|_{\Fr}^2      \right]\nonumber \\
	&\leq \min_{\theta\in [0,1]}\Bigg[     \mathcal{L}(Z^i)+\theta(\mathcal{L}(Z^\infty)-\mathcal{L}(Z^i)) \\& \quad \quad \quad \quad \quad \quad \quad \quad \quad \quad \quad +\frac{1}{2}\theta^2\|Z^0-Z^\infty\|_{\Fr}^2        \Bigg].\nonumber 
	\end{align}
	The last expression is minimised for $\theta=\theta^i$ where 
	\begin{align}
	\label{expression}
	\theta^i=\min\left(\frac{\mathcal{L}(Z^i)-\mathcal{L}(Z^\infty)}{\|Z^0-Z^\infty\|_{\Fr}^2},1\right).
	\end{align}
	
	(If $\|Z^0-Z^\infty\|_{\Fr}^2=0$, then $Z^i=Z^\infty\quad \forall i$ and there is nothing to prove.)
	
	Recall also that $\theta^i$ is a decreasing sequence (cf. Lemma~\ref{monotone}): if $\theta^i\leq 1$, then $\theta^j\leq 1 \quad \forall j>i$. Suppose $\theta^0= 1$. Then, plugging this back into equation~\eqref{thekey}, we obtain: 
	\begin{align}
	\label{oho}
	\mathcal{L}(Z^1)-\mathcal{L}(Z^\infty)\leq \frac{1}{2}\|Z^0-Z^\infty\|_{\Fr}^2,
	\end{align}
	and therefore $\theta^1\leq \frac{1}{2}$. Thus, in all cases, $\theta^i<1 \quad \forall i\geq 1$. Note also that if $\theta^0=1$, inequality~\eqref{BBB} is satisfied (this follows from inequality~\eqref{oho}).
	
	Now, for $i\geq 1$, we can just use the explicit expression~\eqref{expression} for $\theta$, which, plugged back into equation~\eqref{thekey}, gives: 
	
	\begin{align}
	\mathcal{L}(Z^{i+1})-\mathcal{L}(Z^{i})\leq -\frac{(\mathcal{L}(Z^{i})-\mathcal{L}(Z^{\infty}))^2}{2\|Z^0-Z^\infty\|_{\Fr}^2}.
	\end{align}
	Now, writing $\alpha_i$ for $(\mathcal{L}(Z^{i})-\mathcal{L}(Z^{\infty}))$ (which is a decreasing sequence, as shown by Lemma~\ref{decreasefrob1}) and using the above expression, we obtain
	\begin{align}
	\alpha_i&\geq \frac{\alpha_i^2}{2\|Z^0-Z^\infty\|_{\Fr}^2}+\alpha_{i+1}
	\nonumber\\& \geq \frac{\alpha_i\alpha_{i+1}}{2\|Z^0-Z^\infty\|_{\Fr}^2}+\alpha_{i+1},
	\end{align}
	which yields: 
	\begin{align}
	\alpha_{i+1}^{-1}\geq \frac{1}{2\|Z^0-Z^\infty\|_{\Fr}^2}+\alpha_{i}^{-1}.
	\end{align}
	Summing both sides for the index running from $1$ to $i-1$, we obtain: 
	\begin{align}
	\label{eq:alphaend}
	\alpha_{i}^{-1}\geq \frac{i-1}{2\|Z^0-Z^\infty\|_{\Fr}^2}+\alpha_{1}^{-1}.
	\end{align}
	Since $\theta_1<1$, by definition of $\theta_1$, we obtain $\frac{\alpha_1}{2\|Z^0-Z^\infty\|_{\Fr}^2}\leq \frac{1}{2}$. Plugging this back into equation~\eqref{eq:alphaend}, we obtain: 
	\begin{align}
	\alpha_i^{-1}&\geq  \frac{i-1}{2\|Z^0-Z^\infty\|_{\Fr}^2}+\alpha_{1}^{-1}\nonumber \\ &\geq \frac{i-1}{2\|Z^0-Z^\infty\|_{\Fr}^2}+\frac{1}{\|Z^0-Z^\infty\|_{\Fr}^2}\nonumber \\
	&= \frac{i+1}{2\|Z^0-Z^\infty\|_{\Fr}^2}, \nonumber 
	\end{align}
	which yields inequality~\eqref{BBB} after inverting both sides.

\end{proof}

\textbf{Remark:}
We use the notation $Z^\infty$ to refer to the limit of the sequence of iterates, instead of referring to 'the solution $Z^*$ of the optimization problem~\eqref{theopt}' because the solution is not necessarily unique and $Z^\infty$ may actually depend on the initialization. Indeed, the optimization problem~\eqref{theopt} is \textit{convex} but \textit{not strongly convex}. Thus, there can be several solutions, but each of them corresponds to the same value of the objective function. 
To check that the specific problem~\eqref{theopt} can indeed have several equivalent solutions, consider the following example. $X^{(1)}=Y^{(1)}=\Id$ (so that our algorithm coincides with Softimpute), $m=n=2$, and let the observed entries be $R_{2,1}=R_{1,2}=1$ (thus, $\Omega=\{(2,1);(1,2)\}$. Here are three equivalent solutions to the optimization problem with regularising parameter $\lambda$: 
\begin{align*}
A_-&=\left( \begin{array}{cc}
\lambda-1  & 1-\lambda  \\
1-\lambda  & \lambda -1
\end{array}\right);\\ 
A_+& =\left( \begin{array}{cc}
1-\lambda  & 1-\lambda  \\
1-\lambda   & 1-\lambda 
\end{array}\right); \quad \quad  \text{and} \\
A_0 &=\left( \begin{array}{cc}
0 & 1-\lambda \\
1-\lambda   & 0
\end{array}\right).
\end{align*}

In all three cases, the nuclear norm is $2-2\lambda$, and the value of the objective function is $\lambda(2-2\lambda)+\frac{1}{2}(\lambda^2+\lambda^2)$. It is also easy to check that those are actually solutions of~\eqref{theopt} because performing any extra iteration of the algorithm yields the same matrix: 
for $A_0$, after the imputation step we get the target \begin{align*}
&\left( \begin{array}{cc}
0 & 1 \\
1  & 0
\end{array}\right)\\&=  1\times \left( \begin{array}{c}
0  \\
1 
\end{array}\right) \left( \begin{array}{cc}
1  &0
\end{array}\right)+1 \times \left( \begin{array}{c}
1  \\
0 
\end{array}\right) \left( \begin{array}{cc}
0 &1
\end{array}\right),
\end{align*}
where the second line is the SVD. It is clear from the SVD that applying the singular value thresholding operator will return the matrix $A_0$. Thus, the algorithm converges exactly to $A_0$ in one (zero) iteration(s). 

For $A_+$, note that after the imputation step we get the target

\begin{align*}
& \left( \begin{array}{cc}
1-\lambda  & 1  \\
1   & 1-\lambda 
\end{array}\right)\\
&=(2-\lambda)\times  \left( \begin{array}{c}
1/\sqrt{2}  \\
1/\sqrt{2}
\end{array}\right) \left( \begin{array}{cc}
1/\sqrt{2}  &1/\sqrt{2}
\end{array}\right)\\&+\lambda \times  \left( \begin{array}{c}
-1/\sqrt{2}  \\
1/\sqrt{2}
\end{array}\right) \left( \begin{array}{cc}
1/\sqrt{2}  &- 1/\sqrt{2}
\end{array}\right),
\end{align*}
and it is clear that after applying the SVT operator we obtain 
\begin{align*}
&(2-2\lambda)\times  \left( \begin{array}{c}
1/\sqrt{2}  \\
1/\sqrt{2}
\end{array}\right) \left( \begin{array}{cc}
1/\sqrt{2}  &1/\sqrt{2}
\end{array}\right)\\&
=\left( \begin{array}{cc}
1-\lambda  & 1-\lambda  \\
1-\lambda   & 1-\lambda 
\end{array}\right) =A_+,
\end{align*}
as expected. A similar calculation shows that $A_-$ is also a solution. 

\section{Complexity and runtime analysis}
\label{complexity}

\noindent \textbf{Runtime analysis of SVT operations in Alg.~\ref{OMICimpute}:}
Whilst Algorithms~\ref{OMICimpute} and~\ref{BOMICLUSTERS2} theoretically require $KL$ SVD operations, in many instances of our model class (including BOMIC, OMIC+ and BOMIC+) most of the svd calculations are actually \textit{trivial}. Indeed, consider the target matrix $T=P_{\Omega^\top}(Z^{\old})+R_\Omega$. Note that $M^{(k,l)}:=(X^{(k)})^\top T Y^{(l)}$ is a $d_1^{(k)}\times d_2^{(l)}$ matrix. For instance, if $d_1^{(k)}=1$ or $d_2^{(l)}=1$ (cf. BOMIC+, $k=1$ or $l=1$), this matrix is a vector, making the computation of an SVD unnecessary. In fact, for any combination $(k,l)$ such that $d_1^{(k)}+d_2^{(l)}$ is small, it is easy to compute the small matrix $(X^{(k)})^\top T Y^{(l)}$ and perform its SVD through standard methods.

Figure~\ref{fig:my_label} is a graph which compares the runtimes of SVD calculations in Softimpute and our algorithm~\ref{OMICimpute}.
For each datapoint, we randomly selected a matrix with the following parameters $\text{rank} \in \{5,6,\cdots,10\};
m \in \{100,101,\cdots,1000\}; 
d^{(1)}_1=d^{(1)}_2 \in \{2,3,\cdots,\lceil 0.1 m \rceil\}$. More specifically, the users and items were each divided into $d^{(1)}_1=d^{(1)}_2$ communities and the matrices $X^{(1)},X^{(2)},X^{(3)},Y^{(1)},Y^{(2)},Y^{(3)}$ were constructed according to the standard procedure for BOMIC+ (see. Subsection~\ref{bomic+}). The matrices $M^{(k,l)}$ were chosen with iid Gaussian entries.

We then compared, on the one hand (left part of the figure) the following two operations: 
\begin{itemize}
	\item[1] Performing the SVD of the full matrix $R=\sum_{k,l}X^{(k)}M^{(k,l)}(Y^{(k)})^\top$,
	\item[2] Performing  the SVDs of \textit{all nine matrices} $M^{(k,l)}$ for $k,l\leq 3$;	\end{itemize}
and on the other hand (right part of the figure), the following two operations:

\begin{itemize}
	\item[1] One full iteration of the Softimpute algorithm, including imputation and singular value thresholding operator. 
	\item[2] One full iteration of our algorithm, including the imputation and the application of the generalized singular value thresholding operator. 
\end{itemize}

\begin{figure}
	\centering
	\includegraphics[width=0.97 \linewidth]{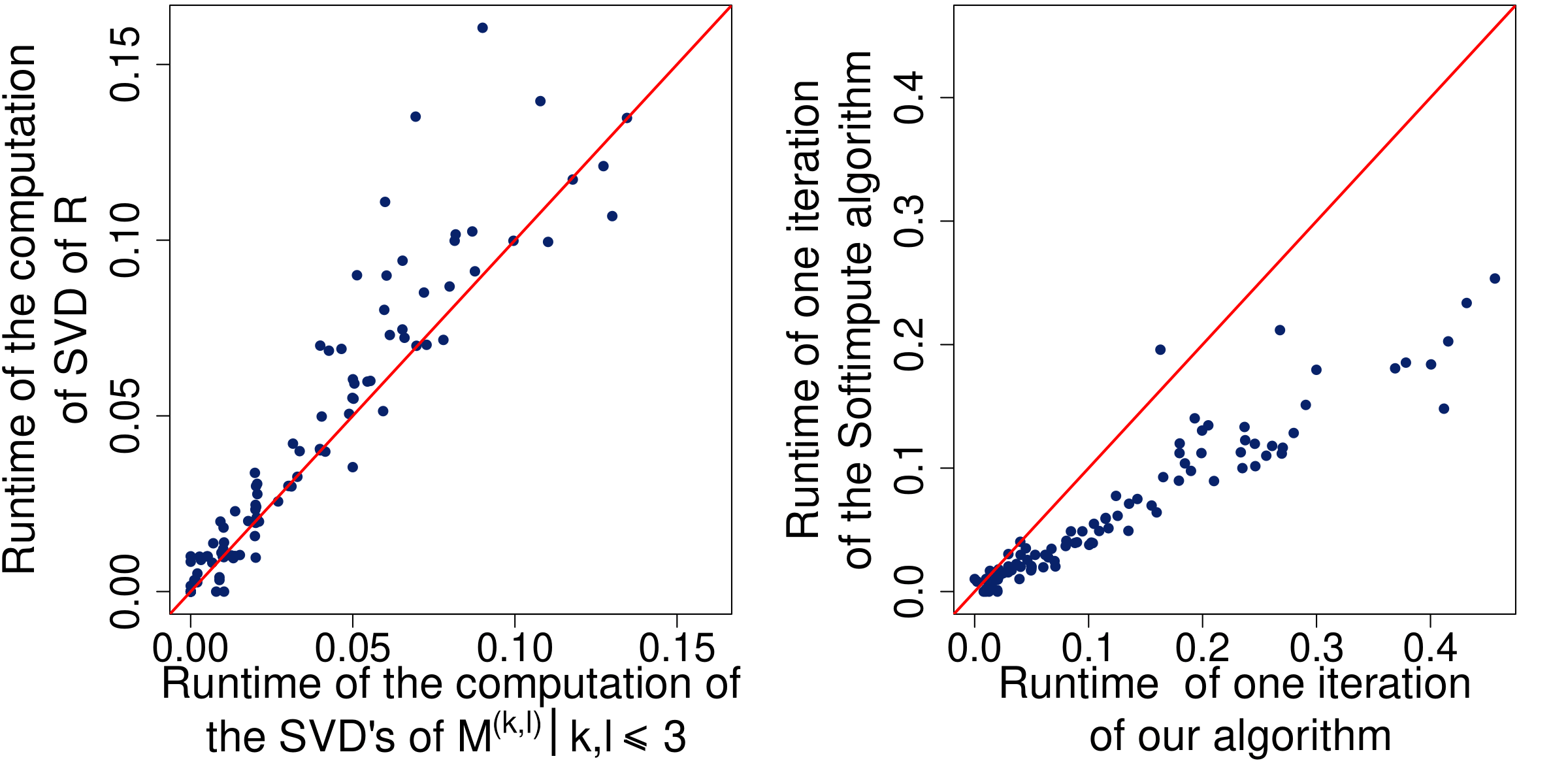}
	\caption{Runtime comparison of main computations required for one iteration of Softimpute and OMIC. The red line is the identity.}
	\label{fig:my_label}
\end{figure}
As we can see from the figure, the computational burden of all 9  SVDs required in our algorithm is not significantly bigger than that of the single SVD required in the SoftImpute implementation. Furthermore, although one full iteration of our algorithm is slower than one full iteration of the softimpute algorithm due to extra multiplication steps, this appears to be the case only by a small constant factor.

\noindent \textbf{Formal complexity analysis:}
We provide an efficient implementation for the special cases BOMIC, OMIC+, and BOMIC+. In those three cases, the number of iterations required at each step of both  Algorithm~\ref{BOMICLUSTERS2} as well as the SVD calculation (Algorithm~\ref{SVDCalc}) depend on the many practicalities related to various warm starts applied in both cases. However, it is possible to write down the complexity of performing one iteration.

At each iteration of Algorithm~\ref{BOMICLUSTERS2}, the key step is the SVT operation using Algorithm~\ref{SVDCalc} (the only other operation being an assignment of $O(|\Omega|)$ entries). For each iteration inside Algorithm~\ref{SVDCalc}, the complexity can be computed from the following operations which are each required a fixed number of times (here as usual, $r$ is the fixed maximum rank set as a hyperparameter):
\begin{itemize}
	\item Multiplying each column of a matrix in $\R^{m\times r}$ or $\R^{n\times r}$ by a different constant (e.g. line 12,17). Cost: $O((m+n)r).$
	\item Computing the SVD of a $\R^{m\times r}$ or $\R^{n\times r}$ matrix (e.g., lines 16,26). Cost: $O(r^3+(m+n)r^2)$.
	\item Performing projections onto the spaces corresponding to $X^{(k)}$ or $Y^{(l)}$ via the procedure from line 1. Cost: $O(m+n)$.
	\item Multiplying $r$ vectors by the current target (see lines 14,19, 24). Cost: $O(|\Omega|r+(m+n)r^2)$.
\end{itemize}
Since $r\leq m+n$, this yields an overall complexity of $O(KL[|\Omega|r+(m+n)r^2])$. Note that $KL\leq 9$, so that the complexity is $O(9[|\Omega|r+(m+n)r^2])=O(|\Omega|r+(m+n)r^2)$, the same as the SoftImpute algorithm~\cite{softimpute}.

\section{Proofs of generalization bounds}
\textbf{Notation:} In this section, we assume the entries are sampled with i.i.d. noise, so that observations of entry $R_{i,j}$ are of the form $R_{i,j}+\delta_{i,j}$ for $\delta_{i,j}\sim \Delta_{i,j}$ for some noise distribution $\Delta_{i,j}$\footnote{Furthermore, $R_{i,j}$  and $\Delta_{i,j}$ are defined so that $R_{i,j}=\argmin_y \mathbb{E}_{\delta_{i,j}}(\ell(R_{i,j}+\delta_{i,j},y))$. }. Thus, $N$ i.i.d. observations indexed by $\alpha\in \{1,2,\ldots,N\}$ are denoted by $R_{i_\alpha,j_\alpha}+\delta_\alpha$ where $(i_\alpha,j_\alpha)$ is the $\alpha$'th i.i.d. choice of entry, and each $\delta_\alpha$ is drawn from $\Delta_{i_\alpha,j_\alpha}$ independently. The loss function $\ell:\mathbb{R}\times \mathbb{R}\rightarrow \mathbb{R}^+$ is bounded by a constant $B$, with Lipschitz constant bounded by $L_{\ell}$\footnote{These conditions are satisfied, for instance, for a hinge loss which could be used to estimate the probability of predicting within a given accuracy, or for the squared loss if one additionally assumes a fixed upper bound on all entries and predictions (which is a reasonable assumption in practice).}. For all $k\leq K, l\leq L, i\leq m,j\leq n$, we will write $\mathbf{x}^k_i$ (resp. $\mathbf{y}^k_j$ ) for the $i$th row (resp. $j$th column) of the matrix $X^{(k)}$ (resp. $Y^{(l)}$ ),   $\mathcal{X}^{(k)}$ for $\max_{i=1}^m \|\mathbf{x}^k_i\|_{2}=\max_{i=1}^m \|X^{(k)}_{i,\nbull}\|_{2}$ and $\mathcal{Y}^{(l)}$  for $\max_{i=1}^n \|\mathbf{y}^l_i\|_{2}$. 
For a predictor $f:\{1,2,\ldots,m\}\times \{1,2,\ldots,n\}\rightarrow \mathbb{R}$, we will write $\mathcal{R}(f)$ for the expected risk $\mathbb{E}_{(i,j)\sim\mathcal{D}}(\ell(f(i,j),R_{i,j}+\delta_{i,j})$ and $\hat{\mathcal{R}}(f)$ for the empirical risk  $(1/N)\sum_{\alpha=1}^N\ell(f(i,j),R_{i_\alpha,j_\alpha}+\delta_\alpha)$. Table~\eqref{thiss} summarizes all the notations used in this appendix and in the paper.

\subsection{Proof of bounds in the distribution-free case}

First, let us recall the following lemma from~\cite{mostrelated,ReallyUniform1}.
\begin{lemma}
	\label{mostcomplex}
	Let $\ell$ be a loss function bounded by $B$ and with Lipschitz constant bounded by $L$. Suppose we are given a matrix $R\in \mathbb{R}^{m\times n}$, which is observed with i.i.d. noise, so that observing entry $(i,j)$ results in an output of $R_{i,j}+\delta$ where $\delta\sim \Delta_{i,j}$ where the $\Delta_{i,j}$ are distributions. Let $F_{\mathcal{M}}$ be the set of matrices $\tilde{R}$ with $\|\tilde{R}\|_{*}\leq \mathcal{M}$. Let us write the data-dependent Rademacher complexity for $N$ samples indexed by $\alpha\in\{1,2,\ldots,N\}$ by  $\rad_N(F_{\mathcal{M}})=\mathbb{E}\left(\sup_{\tilde{R}\in F_{\mathcal{M}}}\frac{1}{N}\sum_{\alpha=1}^N\sigma_{\alpha}\ell(R_{(i_\alpha,j_\alpha)}+\delta_\alpha,\tilde{R})\right)$, where the $\sigma_\alpha$ are independent Rademacher random variables, and the $(i_\alpha,j_\alpha)$ are entries sampled independently. 
	
	We have the following bound on the expected complexity $\rad=\mathbb{E}_{\Omega}(\rad_N(F_{\mathcal{M}}))$: 
	\begin{align}
	&\rad=\mathbb{E}_{\Omega}(\rad_N(F_{\mathcal{M}}))\nonumber\leq \sqrt{\frac{9\mathcal{M}B\mathcal{C}L(\sqrt{d_1}+\sqrt{d_2})}{N}}.
	\end{align}
	Here, $\mathcal{C}$ is the universal constant from~\cite{Rafal2005}.
\end{lemma}

Using this, we can show the following lemma for side information by absorbing the variation between entries corresponding to the same communities into the "noise" of an auxiliary problem to which we apply Lemma~\ref{mostcomplex}.

\begin{lemma}
	\label{lem:absorbasnoise}
	Let $X\in \mathbb{R}^{m\times A}$ and  $Y\in \mathbb{R}^{n\times B}$ be auxiliary matrices whose columns are indicator functions of distinct sets forming partitions $\{c_1,c_2,\ldots,c_A\}$  and $\{s_1,\ldots,s_B\}$ of $\{1,2,\ldots,n\}$ and $\{1,2,\ldots,m\}$ respectively. Set $t>0$ and consider the function class $\mathcal{F}_t:=\left\{ XMY^\top \big| \|M\|_*\leq t \right\}$. The Rademacher complexity $\mathcal{F}_N(\mathcal{F}_t)$ satisfies $$\mathcal{R}_N(\mathcal{F}_t)\leq \sqrt{\frac{9tC\mathcal{C}L\left[\sqrt{a}+\sqrt{b}\right]}{N}}.$$
	In particular, if we consider instead $\mathcal{G}_t:=\left\{ XMY^\top \big| \rank(M)\leq r \right\}$, we obtain: 
	$$\mathcal{R}_N(\mathcal{G}_t)\leq \sqrt{\frac{9\sqrt{Cabr}B\mathcal{C}L\left[\sqrt{a}+\sqrt{b}\right]}{N}},$$
	where $C$ is a bound on the predicted entries. 
	
\end{lemma}

\begin{proof}
	This follows from Lemma~\ref{mostcomplex} applied to the modified problem where for $u\leq A,v\leq B$, the observations $R_{u,v}$ are distributed according to the distribution of $R_{i,j}+\delta_{i,j}$ where $(i,j)$ is drawn from $\mathcal{D}$ conditioned on $i\in c_u, j\in c_v$ (note that Lemma~\ref{mostcomplex} allows for the observations of $R_{i,j}$ to be perturbed by random variables with distributions $\Delta_{i,j}$ conditioned on $(i,j)$, the differences between the values of the ground truth matrix at different pairs $(i,j)$ where the communities of $i$ and $j$ are fixed can be absorbed into this perturbation).
\end{proof}

\begin{proposition}
	\label{prop:hadamardetc}
	Let $A\in \mathbb{R}^{m\times n}$ be a matrix and let $v\in \R^{m}$ and $w\in \R^{n}$ be two vectors. We have 
	\begin{align}
	\|vw^\top \odot A\|_*\leq  \max_{i,j} |v_i||w_j|\|A\|_*
	\end{align}
	where $\odot$ denotes the Hadamard (entry wise) product.
\end{proposition}
\begin{proof}
	
	By an equivalent formulation of the nuclear norm~\ref{nuclearfro} (see. also ~\cite{Srebro})
	\begin{align*}
	&\|vw^\top \odot A\|_*\\&= \min\big (\|B\|_{\Fr}\|C\|_{\Fr}\big | \\ & \quad\quad\quad \quad \quad \quad BC^\top =vw^\top \odot A\big) \\
	&\leq \min\big (\|\diag(v)B\|_{\Fr}\|\diag(w)C\|_{\Fr}\big |\\ & \quad \quad\quad\quad \quad \quad  \diag(v)B(\diag(w)C)^\top=vw^\top \odot A\big)\\
	&=\min\big (\|\diag(v)B\|_{\Fr}\|\diag(w)C\|_{\Fr}\big |\\ & \quad \quad\quad\quad \quad \quad (vw^\top) \odot (BC^\top)=vw^\top \odot A\big)\\
	&\leq \min\big (\|\diag(v)B\|_{\Fr}\|\diag(w)C\|_{\Fr}\big | \\ & \quad \quad\quad\quad \quad \quad  BC^\top= A\big)\\
	&\leq \min\big (\max_i |v_i|\|B\|_{\Fr}\max_j |w_j|\|\diag(w)C\|_{\Fr}\big | \\ & \quad\quad\quad\quad \quad \quad  BC^\top= A\big)\\
	&=\max_{i,j} |v_i||w_j|\|A\|_*,
	\end{align*}
	as expected.

\end{proof}

We are now in a position to present the main result.

\begin{theorem}
	Consider the OMIC+ setting from Section~\ref{omic+} (in particular, $K=L=2$).  Let $\mathcal{K}$ denote the maximum ratio between the sizes of any two user or item communities. 
	Choose some $\mathcal{M}_{k,l}$  and $C_{k,l}$ such that $
	\|R^{(k,l)}\|_{*}\leq \mathcal{M}_{k,l}$  and $\max_{i,j}|R^{(k,l)}_{i,j}|\leq C_{k,l}$ for all $(k,l)$. 
	Let $\hat{f}$ be the solution to the optimization problem \begin{align}
	\label{optim}
	&\minimize \quad \hat{\mathcal{R}}(f) \> \text{s.t.}  \>\forall  k,l,\>  \nonumber \\&f=\sum_{k,l} (X^{(k)}M^{(k,l)}(Y^{(l)})^{\top}); \>  \|M^{(k,l)}\|_{*} \leq \mathcal{M}_{k,l}; \>\nonumber \\&\text{and}\> \|X^{(k)}M^{(k,l)}(Y^{(l)})^{\top}\|_\infty \leq C_{k,l}\quad \forall k,l,j,i
	\end{align}
	
	\label{CommunityPrecise2}
	With probability $\geq 1-\delta$ over the draw of the training set, the solution to the optimization problem~\eqref{optim}  satisfies 
	\begin{align}
	\label{eq:finalbounds1}
	&	\mathcal{R}(\hat{f})\leq 2L_\ell\sqrt{\frac{9\mathcal{C} }{N}}\bigg(\frac{1}{\sqrt{c}}\sqrt{C_{1,1}\mathcal{M}_{1,1}}\left[\sqrt{a}+\sqrt{b}\right]^{1/2}\nonumber \\& +\frac{1}{\sqrt[4]{c}}\sqrt{C_{1,2}\mathcal{M}_{1,2}}\left[\sqrt{a}+\sqrt{n}\right]^{1/2}\nonumber \\&+\frac{1}{\sqrt[4]{c}}\sqrt{C_{2,1}\mathcal{M}_{2,1}}\left[\sqrt{m}+\sqrt{b}\right]^{1/2}\nonumber \\
	&+\sqrt{C_{2,2}\mathcal{M}_{2,2}}\left[\sqrt{m}+\sqrt{n}\right]^{1/2}\bigg)\nonumber \\
	&+2B\sqrt{\frac{\log(1/\delta)}{2M}} +\mathcal{E}.
	\end{align}
	
	Expressed in terms of matrix ranks instead, we obtain: 
	\begin{align}
	\label{eq:finalbounds}
	&	\mathcal{R}(\hat{f})\leq 2L_\ell\sqrt{\frac{9\mathcal{C} }{N}}\bigg(\sqrt{\mathcal{K}}C_{1,1}\sqrt[4]{abr_{1,1}}\left[\sqrt{a}+\sqrt{b}\right]^{1/2}\nonumber \\& \quad \quad \quad \quad \quad  +\sqrt[4]{\mathcal{K}}C_{1,2}\sqrt[4]{anr_{1,2}}\left[\sqrt{a}+\sqrt{n}\right]^{1/2}\nonumber \\&\quad \quad \quad \quad  \quad +\sqrt[4]{\mathcal{K}}C_{2,1}\sqrt[4]{mbr_{2,1}}\left[\sqrt{m}+\sqrt{b}\right]^{1/2}\nonumber \\&\quad \quad \quad \quad \quad 
	+C_{2,2}\sqrt[4]{mnr_{2,2}}\left[\sqrt{m}+\sqrt{n}\right]^{1/2}\bigg)\nonumber \\
	&\quad \quad \quad \quad \quad +2B\sqrt{\frac{\log(1/\delta)}{2M}} +\mathcal{E}.
	\end{align}
	(Here $\mathcal{C}$ is an absolute constant and $c$ is the number of elements in the smallest community. )
\end{theorem}

\begin{proof}
	
	For convenience we prove a slightly more general result where $c_1$ (resp. $c_2$) is the size of the smallest community of users (resp. items).
	Let $X'$ and $Y'$ denote matrices whose columns are the (non-normalised) indicator functions of the communities. By Lemma~\ref{mostcomplex}, the Rademacher complexity of the function class $=\{ X'M(Y')^\top |\|M\|_{*}\leq \mathcal{M}_{1,1} \land \|X'M(Y')^\top\|_\infty \leq C_{1,1} \}$ is bounded by 
	$$\sqrt{C_{1,1}\frac{9\mathcal{M}_{1,1}\mathcal{C}(\sqrt{a}+\sqrt{b})}{N}}.$$
	Now observe that by Lemma~\ref{prop:hadamardetc}, the function class \begin{align*}
	\mathcal{F}_{1,1}&:=\{ XM(Y)^\top |\|M\|_{*}\leq \mathcal{M}_{1,1}  \\ & \quad \quad \quad \quad \land \|X^{(1)}M(Y^{(1)})^\top\|_\infty \leq C_{1,1} \}
	\end{align*} satisfies  \begin{align*}
	&\mathcal{F}_{1,1}\subset \{ X'M'(Y')^\top |\|M'\|_{*}\leq \mathcal{M}_{1,1}c_1^{-1/2}c_2^{-1/2} \\ & \quad \quad \quad \quad \quad \quad \quad \quad \quad  \land \|X'M'(Y')^\top\|_\infty \leq C_{1,1} \},\end{align*}
	where $c_1$ (resp. $c_2$) is the size of the smallest community of users (resp. items). It follows that $$\rad(\mathcal{F}_{1,1})\leq \frac{1}{\sqrt[4]{c_1c_2}}\sqrt{C_{1,1}\frac{9\mathcal{M}_{1,1}\mathcal{C}(\sqrt{a}+\sqrt{b})}{N}}.$$
	
	By the same argument applied to the two situations where each user or item is a single community, we obtain the following results for $\mathcal{F}_{1,2}:=\{ X^{(1)}M(Y^{(2)})^\top |\|M\|_{*}\leq \mathcal{M}_{1,2} \land \|X^{(1)}M(Y^{(2)})^\top\|_\infty \leq C_{1,2} \}$,  $\mathcal{F}_{2,1}:=\{ X^{2}M(Y^{1})^\top |\|M\|_{*}\leq \mathcal{M}_{2,1} \land \|X^{(2)}M(Y^{(1)})^\top\|_\infty \leq C_{2,1} \}$ and $\mathcal{F}_{2,2}:=\{ X^{(2)}M(Y^{(2)})^\top |\|M\|_{*}\leq \mathcal{M}_{2,2} \land \|X^{(2)}M(Y^{2})^\top\|_\infty \leq C_{2,2} \}$:

	\begin{align*}
	&\rad(\mathcal{F}_{1,2})\leq\\  &\rad(\widetilde{\mathcal{F}_{1,2}})\leq  \frac{1}{\sqrt[4]{c_1}}\sqrt{C_{1,2}\frac{9\mathcal{M}_{1,2}\mathcal{C}(\sqrt{a}+\sqrt{n})}{N}};
	\end{align*}
	
	\begin{align*}
	&\rad(\mathcal{F}_{2,1})\leq\\ &\rad(\widetilde{\mathcal{F}_{2,1}})\leq  \frac{1}{\sqrt[4]{c_2}}\sqrt{C_{2,1}\frac{9\mathcal{M}_{2,1}\mathcal{C}(\sqrt{m}+\sqrt{b})}{N}};
	\end{align*}
	
	and 
	\begin{align*}
	& \rad(\mathcal{F}_{2,2})\leq \\& \rad(\widetilde{\mathcal{F}_{2,2}})\leq  \sqrt{C_{2,2}\frac{9\mathcal{M}_{2,2}\mathcal{C}(\sqrt{m}+\sqrt{n})}{N}};
	\end{align*}
	
	after noting that $\mathcal{F}_{1,2}\subset \widetilde{\mathcal{F}}_{1,2}:= \{ X^{(1)}M(\tilde{Y}^{(2)})^\top |\|M\|_{*}\leq \mathcal{M}_{1,2} \land \|X^{(1)}M(\tilde{Y}^{(2)})^\top\|_\infty \leq C_{1,2} \}=\{ X^{(1)}M |\|M\|_{*}\leq \mathcal{M}_{1,2} \land \|X^{(1)}M\|_\infty \leq C_{1,2} \}$; 
	$\mathcal{F}_{2,1}\subset \widetilde{\mathcal{F}}_{2,1}:= \{ \tilde{X}^{(2)}M(Y^{(2)})^\top |\|M\|_{*}\leq \mathcal{M}_{2,1} \land \|\tilde{X}^{(2)}M(Y^{(2)})^\top\|_\infty \leq C_{2,1} \}=\{ M(Y^{2})^\top |\|M\|_{*}\leq \mathcal{M}_{2,1} \land \|M(Y^{(2)})^\top\|_\infty \leq C_{2,1} \}$; and $\mathcal{F}_{2,2}\subset \widetilde{\mathcal{F}}_{2,2}:=\{ M |\|M\|_{*}\leq \mathcal{M}_{2,2} \land \|M\|_\infty \leq C_{2,2} \}$.
	Here $\tilde{X}^{(1)}=X^{(1)}$ is a matrix whose columns are the (non normalised) indicator functions of the communities, and $\tilde{X}^{(2)}$ and $\tilde{Y}^{(2)}$ are identity matrices.
	
	By the subadditivity of Rademacher complexity, Talagrand's lemma and the classic Rademacher theorem then immediately yield the first result: 
	\begin{align}
	\mathcal{R}(\hat{f})&\leq 2B\sqrt{\frac{\log(1/\delta)}{2M}} +\mathcal{E}+\nonumber \\& 2L_{\ell}\left[\rad(\mathcal{F}_{1,1})+\rad(\mathcal{F}_{1,2})+\rad(\mathcal{F}_{2,1})+\rad(\mathcal{F}_{2,2})  \right]\nonumber \\
	&\leq 2L_\ell\sqrt{\frac{9\mathcal{C} }{N}}\bigg(\frac{1}{\sqrt{c}}\sqrt{C_{1,1}\mathcal{M}_{1,1}}\left[\sqrt{a}+\sqrt{b}\right]^{1/2}\nonumber \\& +\frac{1}{\sqrt[4]{c}}\sqrt{C_{1,2}\mathcal{M}_{1,2}}\left[\sqrt{a}+\sqrt{n}\right]^{1/2}\nonumber \\&+\frac{1}{\sqrt[4]{c}}\sqrt{C_{2,1}\mathcal{M}_{2,1}}\left[\sqrt{m}+\sqrt{b}\right]^{1/2}\nonumber \\
	&+\sqrt{C_{2,2}\mathcal{M}_{2,2}}\left[\sqrt{m}+\sqrt{n}\right]^{1/2}\bigg)\nonumber \\
	&+2B\sqrt{\frac{\log(1/\delta)}{2M}} +\mathcal{E}.
	\end{align}

	Regarding the second result, note that since the entries of $M_{1,1}$ are bounded by $C_{1,1}\bar{c}$ where $\bar{c}$ is the size of the largest community, we have  $\|M_{1,1}\|\leq \bar{c} C_{1,1}\sqrt{abr_{1,1}}$, $\|M_{1,2}\|\leq \sqrt{\bar{c}} C_{1,2}\sqrt{anr_{1,2}}$, $\|M_{2,1}\|\leq \sqrt{\bar{c}}C_{2,1} \sqrt{mbr_{2,1}}$ and $\|M_{2,2}\|_*\leq C_{2,2}\sqrt{mnr_{2,2}}$. Plugging this back into the first result yields the second result, as expected.

\end{proof}

If we set the number of user and item communities to one, we obtain the BOMIC model. Furthermore, results can also be similarly extended to the BOMIC+ model, obtaining the full theorem below. 
\begin{theorem}
	\label{Bomicbound}
	For the BOMIC algorithm from Subsection~\ref{bias}, with probability $\geq 1-\delta$ over the draw of the training set, we have 
	\begin{align}
	&	\mathcal{R}(\hat{f})\leq 2L_\ell\sqrt{\frac{9\mathcal{C} }{N}}\bigg(C_{1,1} +C_{1,2}\sqrt[4]{n}\left[1+\sqrt{n}\right]^{1/2}\nonumber \\
	&+C_{2,1}\sqrt[4]{m}\left[\sqrt{m}+1\right]^{1/2}\nonumber \\&+C_{2,2}\sqrt[4]{mnr_{2,2}}\left[\sqrt{m}+\sqrt{n}\right]^{1/2}\bigg)\nonumber \\
	&+2B\sqrt{\frac{\log(1/\delta)}{2M}} +\mathcal{E}.
	\end{align}

	For the BOMIC+ algorithm from Subsection~\ref{bias}, with probability $\geq 1-\delta$ over the draw of the training set, we have 
	\begin{align}
	\label{bomic+bound}
	&\mathcal{R}(\hat{f})\leq 2L_\ell\sqrt{\frac{9\mathcal{C} }{N}}\bigg(C_{1,1} \nonumber \\
	&+C_{1,2}\sqrt[4]{b}\left[1+\sqrt{b}\right]^{1/2}+C_{2,1}\sqrt[4]{a}\left[1+\sqrt{a}\right]^{1/2}\nonumber \\
	&+C_{2,2}\sqrt[4]{abr_{2,2}}\left[\sqrt{a}+\sqrt{b}\right]^{1/2}\nonumber  \nonumber \\
	&+C_{1,3}\sqrt[4]{n}\left[\sqrt{n}+1\right]^{1/2}+C_{3,1}\sqrt[4]{m}\left[\sqrt{m}+1\right]^{1/2}\nonumber \\
	&+C_{2,3}\sqrt[4]{anr_{2,3}}\left[\sqrt{a}+\sqrt{n}\right]^{1/2}\nonumber \nonumber \\
	&+C_{3,2}\sqrt[4]{bmr_{3,2}}\left[\sqrt{m}+\sqrt{b}\right]^{1/2}\nonumber \nonumber \\
	&+C_{3,3}\sqrt[4]{mnr_{3,3}}\left[\sqrt{m}+\sqrt{n}\right]^{1/2}\nonumber \bigg)\\
	&+2B\sqrt{\frac{\log(1/\delta)}{2M}} +\mathcal{E}.
	\end{align}
\end{theorem}
\subsection{Bounds under the assumption of uniform marginals}
\label{unimarg}
In this section, we prove bounds assuming that the sampling distribution has the probability that each row and each column each have equal probability of being sampled. 

\begin{proposition}
	\label{oneterm}
	Let $\mathcal{F}_{\mathcal{M}}$ be the class of functions $R\in \mathbb{R}^{m\times n}: \|R\|_{*}\leq \mathcal{M}$, where as usual wlog $m\geq n$ and we also assume $m\geq 3$. Further assume that the sampling distribution has uniform marginals. We have the following bound on the expected Rademacher complexity of $\mathcal{F}$: 
	\begin{align}
	\E(\rad_N(\mathcal{F}))\leq 20 (\mathcal{M}/\sqrt{n})\sqrt{\frac{\log(m)}{N}}.
	\end{align}
\end{proposition}
\begin{proof}
	Writing $x_1,\ldots,x_N$ for the iid samples from $\{1,2,\ldots,m\}\times \{1,2,\ldots,n\}$, $\sigma_1,\ldots,\sigma_N$ for the Rademacher variables  and $\Sigma$ for the matrix with $\Sigma_{i,j}=\sum_{k}\sigma_k 1_{x_k=(i,j)}$, we have by the duality of the trace norm
	\begin{align}
	\E(\rad_N(\mathcal{F}))&=\E_{x,\sigma}\frac{1}{N}\sup_{R\in \mathcal{F}}(\langle R,\Sigma \rangle) \nonumber \\
	&\leq \frac{\mathcal{M}}{N}\E(\|\Sigma\|).
	\end{align}
	
	Note that $\sigma=\sum_{k=1}^NX_k$ where $X_k=\sigma_k 1_{x_k}$. Note also that the  $\E(X_kX_k^\top)$ (resp. $\E(X_k^\top X_k)$) is a diagonal matrix whose ith entry is the sampling probability of the ith row (resp. column). Thus, by our uniform marginal assumption, we have using the notation from proposition~\ref{bernsteinexp} $\rho_k=\sqrt{1/n}$ and $\sigma=\sqrt{N/n}$.

	Thus, by the Bernstein inequality in expectation (proposition~\ref{bernsteinexp}), we can continue, assuming without loss of generality that $m\geq 2$ and $N\geq m\log(m)$:
	\begin{align}
	&\E(\rad_N(\mathcal{F}))\nonumber \\&\leq \frac{\mathcal{M}}{N} \big[\sqrt{8/3}\sigma (1+\sqrt{\log(m+n)})\nonumber \\& \quad \quad \quad \quad\quad \quad+ \frac{8}{3}(1+\log(m+n))\big]\nonumber \\
	&\leq \frac{\mathcal{M}}{N} \big[\sqrt{8/3}\sigma (1+\sqrt{2\log(m)})\nonumber \\& \quad \quad\quad \quad\quad \quad+ \frac{8}{3}(1+2\log(m))\big]\nonumber \\
	&\leq \frac{\mathcal{M}}{N} \big[\sqrt{8/3}\sigma \sqrt{\log(m)}(1/\sqrt{\log(2)}+\sqrt{2})\nonumber \\&\quad  \quad \quad \quad + \frac{8}{3}(1/\sqrt{\log(2)}+2)\sigma \sqrt{\log(m)}\big]\nonumber \\
	&\leq 20 \sigma  (\mathcal{M}/N)\sqrt{\log(m)}\nonumber \\
	&\leq \frac{20(\mathcal{M}/\sqrt{n})\sqrt{\log(m)}}{\sqrt{N}},
	\end{align}
	as expected.
\end{proof}

\begin{theorem}
	\label{thm:omic+bound}
	Consider the community side information setting of Section~\ref{omic+} (OMIC+, in particular, $K=L=2$).  Let $\mathcal{K}$ denote the maximum ratio between the sizes of any two user or item communities. Assume that the sampling distribution has uniform marginals. 
	Choose some $\mathcal{M}_{k,l}$ such that $
	\|R^{(k,l)}\|_{*}\leq \mathcal{M}_{k,l}$  for all $(k,l)$. 
	Let $\hat{f}$ be the solution to the optimization problem \begin{align}
	\label{optim2}
	&\minimize \quad \hat{\mathcal{R}}(f) \> \text{s.t.}  \>\forall  k,l,\>  \nonumber \\&f=\sum_{k,l} (X^{(k)}M^{(k,l)}(Y^{(l)})^{\top}); \>  \|M^{(k,l)}\|_{*} \leq \mathcal{M}_{k,l}.
	\end{align}

	In expectation over the draw of the training set, the solution to the optimization problem~\eqref{optim2}  satisfies 
	\begin{align}
	&	\E(\mathcal{R}(\hat{f}))  \\&\leq \frac{40L_\ell}{\sqrt{N}}\Bigg[ \frac{1}{c}\frac{\mathcal{M}_{1,1}}{\sqrt{a}}\sqrt{\log(b)}+\frac{1}{\sqrt{c}}\frac{\mathcal{M}_{1,2}}{\sqrt{a}}\sqrt{\log(m)}\nonumber \\ & \quad \quad  +\frac{1}{\sqrt{c}}\frac{\mathcal{M}_{2,1}}{\sqrt{b}}\sqrt{\log(n)}+\frac{\mathcal{M}_{2,1}}{\sqrt{n}}\sqrt{\log(m)} \Bigg]+\mathcal{E}.\nonumber 
	\end{align}
	Expressed in terms of the ranks and maximum sizes of the entries $r_{k,l}, C_{k,l}$, chosen to satisfy the feasibility condition $\rank(R_{k,l})\leq r_{k,l}$ and $\|R^{(k,l)}\|_{\infty}\leq C_{k,l}$, we have  
	\begin{align}
	&	\E(\mathcal{R}(\hat{f}))\\ \nonumber &\leq \frac{40L_\ell}{\sqrt{N}}\Bigg[ \mathcal{K}C_{1,1}\sqrt{br_{1,1}\log(b)}\\\nonumber &\quad \quad \quad \quad+\sqrt{\mathcal{K}}C_{1,2}\sqrt{mr_{1,2}\log(m)}\\\nonumber &\quad \quad \quad \quad +\sqrt{\mathcal{K}}C_{2,1}\sqrt{nr_{2,1}\log(n)}\\\nonumber &\quad \quad \quad \quad+C_{2,2}\sqrt{mr_{2,2}\log(m)} \Bigg]+\mathcal{E}.
	\end{align}
\end{theorem}
\begin{proof}
	
	The proof is similar to that of Theorem~\ref{CommunityPrecise2}.

	Let $\ell\circ \mathcal{F}$ be the loss function class associated with the first problem. 
	
	By the Talagrand Lemma and the subadditivity of Rademacher complexity, we have 
	\begin{align}
	\label{ethate}
	& \rad(\ell\circ \mathcal{F})\\&\nonumber \leq 2L_{\ell}\left[\rad(\mathcal{F}_{1,1})+\rad(\mathcal{F}_{1,2})+\rad(\mathcal{F}_{2,1})+\rad(\mathcal{F}_{2,2}) \right],
	\end{align}
	where $\mathcal{F}_{k,l}:=\{ XM(Y)^\top |\|M\|_{*}\leq \mathcal{M}_{k,l} \}$.
	
	Next, note, similarly to the proof of Theorem~\ref{CommunityPrecise2}, that 
	\begin{align*}
	&\mathcal{F}_{1,1}\subset \widetilde{\mathcal{F}_{1,1}}:=\\\nonumber &\{ \widetilde{X}^1M'(\widetilde{Y}^1)^\top |\|M'\|_{*}\leq \mathcal{M}_{1,1}c_1^{-1/2}c_2^{-1/2} \},
	\end{align*}
	and thus, by Proposition~\ref{oneterm}, 
	\begin{align}
	\rad(\mathcal{F}_{1,1})\leq  \rad(\widetilde{\mathcal{F}_{1,1}})\leq  \frac{20}{\sqrt{N}\sqrt{c_1c_2}}\frac{\mathcal{M}_{1,1}}{\sqrt{a}}\sqrt{\log(b)},
	\end{align}
	with similar results for the other terms.

	Plugging this into equation~\eqref{ethate} yields the first result. The final result then follows after noting that $\|M_{1,1}\|_*\leq \sqrt{\bar{c}_1\bar{c}_2} C_{1,1}\sqrt{abr_{1,1}} $, $\|M_{1,2}\|_*\leq \sqrt{\bar{c}_1} C_{1,2}\sqrt{anr_{1,2}} $, $\|M_{2,1}\|_*\leq \sqrt{\bar{c}_2} C_{2,1}\sqrt{mbr_{1,1}} $, and $\|M_{2,2}\|_*\leq  C_{2,2}\sqrt{mnr_{1,1}} $.

\end{proof}

\subsection{In-depth discussion of bounds}
\label{absorb}
In the case of OMIC+ and for a fixed tolerance threshold $\epsilon$, our distribution-free sample complexity bounds scale as \begin{align*}
\mathcal{K}C_{1,1}^2b\sqrt{ar_{1,1}}+\sqrt{\mathcal{K}}C_{1,2}^2 n\sqrt{ar_{1,2}}\\+\sqrt{\mathcal{K}}  C_{2,1}^2 m \sqrt{br_{2,1}}+C_{2,2}^2 m\sqrt{r_{2,2}n}, 
\end{align*}
whilst in the case of uniform marginals they scale like

\begin{align*}
\mathcal{K}^2C_{1,1}^2br_{1,1}\log(b)+\mathcal{K}C_{1,2}^2mr_{1,2}\log(m)\\+\mathcal{K}C_{2,1}^2nr_{2,1}\log(n)+C^2_{2,2}mr_{2,2}\log(m).
\end{align*}

In both cases, if $C_{1,1}$ is much larger than $C_{1,2},C_{2,1}$ and $C_{2,2}$, the bound behaves similarly to the situation where the users and items are identified with their category. Whilst this is not very surprising, the bound does show that this remains true for small but non zero values of $C_{1,2},C_{2,1}$ and $C_{2,2}$: the model can effectively learn a combination of community behaviour and user behaviour with no further difficulty than if it were learning both problems independently. 
Note also conversely that if $a,b$ are very small (as in the particular case BOMIC where $a=b=1$), the first three terms are very small and the bound essentially tells us that the model is about as hard as learning the low-rank residual alone.

Note that the bounds show that prior knowledge of the community structure helps more than knowledge of the generic low-rank structure. Indeed, consider a typical situation where the communities are of equal and small size,   $C_{1,2}=C_{2,1}=0$, $a=b$, $m=n$ and $r_{1,1}=a$, and assume that the absolute values of the maximum and minimum entries of $R^{(2,2)}$ are of the same order so that the maximum absolute value of an entry of $R$ is $\simeq C_{1,1}+C_{2,2}$. With knowledge of the community structure, our model requires in the distribution-free case $O(C_{1,1}^2 a^2+C_{2,2}^2 m^{3/2}\sqrt{r_{2,2}})$ entries. If we were to apply a generic low rank method instead, the number of required entries would then be  $O((C_{2,2}+C_{1,1})^2m\sqrt{(r_{2,2}+a)m})$. More generally, we see that ignoring the change in the $C_{k,l}$'s, absorbing ground truth community component of order $r$ into the generic low-rank term results in an increase of at least $O([\sqrt{r_{2,2}+r}-\sqrt{r_{2,2}}]m^{3/2})$ in the number of required entries, whilst the sample complexity only grows by $O(a^{3/2}[\sqrt{r_{1,1}+r}-\sqrt{r_{1,1}}])$ instead when the community side information is duly considered. Similar results hold for the case of uniform marginals, where absorbing a community component of rank $r$ into the generic low rank component costs $O(rm\log(m))$ in sample complexity, compared to $O(ra\log(a))$ when the side information is duly considered.

Admittedly, absorbing cross terms ($R^{(2,1)}$ or $R^{(1,2)}$) into the generic low-rank component does not cause the bound for uniform marginals to change at the asymptotic level. On the other hand, the \textit{distribution-free bounds} still show a significant advantage in exploiting the community side information even when it comes to cross terms. Indeed, consider a similar situation to above with $a=b=1$ (BOMIC) and $C_{2,1}=C_{1,1}=0$. Absorbing the cross term $R^{(1,2)}$ into the generic low rank component results in a sample complexity bound of order $O((C_{2,2}+C_{1,2})^2\sqrt{r_{2,2}+1}n^{3/2})$. Ignoring again the effects on the $C_{k,l}$ for simplicity, this represents an increase of at least $O([\sqrt{r_{2,2}+1}-\sqrt{r_{2,2}}]n^{3/2})$ to the fourth term, compared to a contribution of order $O(n)$ when the side information is taken into account. This means that (at least for small values of $r_{2,2}$), the knowledge of the specific singular vector (i.e. the singular vector $(1/\sqrt{n})_{i\leq n}$) helps our model more than the simple knowledge of the equivalent restriction in the rank.

\textbf{Remarks on proof techniques and comparison to IMC bounds}

The bounds admittedly follow reasonably straightforwardly from similar techniques as those used for standard matrix completion, applied to auxiliary problems corresponding to each of the four terms ($X^{(1)}M^{(1,1)}(Y^{1})^\top $, $X^{(1)}M^{(1,2)}(Y^{2})^\top $, $X^{(2)}M^{(2,1)}(Y^{1})^\top $ and  $X^{(2)}M^{(2,2)}(Y^{2})^\top $) independently and then merged via the subadditivity of Rademacher complexity. Indeed, in the distribution-free case state-of-the-art bounds take the form (cf.~\cite{ReallyUniform1,mostrelated, espain} etc.) \begin{align}
O\left(\sqrt{\frac{\mathcal{M}(\sqrt{n}+\sqrt{m})}{N}}\right)
\label{noninductiveisdone}
\end{align} where $\mathcal{M}$ is a bound on the nuclear norm and in the case of uniform marginals, bounds of the form 
\begin{align}
\label{noninductiveisdoneuniform}
O\left( \sqrt{\frac{(\mathcal{M}^2/n)\log(m)}{N}}\right)
\end{align}
were proved in~\cite{foygel2011learning}.

However, we note that the state-of-the-art bounds for \textit{inductive matrix completion} (whose predictors take the form $XMY^{\top}$ for some fixed side information $X$, with nuclear norm minimization at work on $M$), applied to any of the first three terms in question do \textbf{not} yield bounds as tight as ours. 

Indeed, there is no suitable equivalent to~\eqref{noninductiveisdone} or~\eqref{noninductiveisdoneuniform} in the inductive case, and the state-of-the-art bounds for inductive matrix completion (cf.~\cite{espain,PIMC}, etc.),  actually scale like $O\left(\mathbf{x} \mathbf{y}\mathcal{M}\sqrt{\frac{1}{N}}\right)$ (or $rd_1d_2$ where $d_1$ and $d_2$ are the dimensions of the side information) instead, where $\mathbf{x}$ (resp. $\mathbf{y}$) stands for the maximum norm of a row of $X$ (resp. $Y$). 

In the representative case $C_{1,2}=C_{2,1}=C_{2,2}=0$, our distribution-free bound~\eqref{eq:finalbounds} takes the form 
\begin{align}
& O\left( \sqrt{\frac{C_{1,1}\mathcal{M}_{1,1} \left[\sqrt{a}+\sqrt{b}\right]  }{cN}} \right)\nonumber \\ &\simeq O\left( C_{1,1} \sqrt{\frac{\mathcal{K}\sqrt{abr_{1,1}}\left[\sqrt{a}+\sqrt{b}\right]}{N}}\right), 
\end{align}
whereas the bound~\eqref{noninductiveisdone} takes the form 
\begin{align}
\frac{1}{c\sqrt{N}}\mathcal{M}_{1,1}\simeq \frac{\sqrt{r_{1,1}}\sqrt{ab}\bar{c}}{\sqrt{N}c}=\frac{\sqrt{r_{1,1}}\sqrt{ab}\mathcal{K}}{\sqrt{N}}.
\end{align}
In terms of sample complexity, this corresponds to a required number of entries of the order of $abr_{1,1}$, in line with the results of~\cite{PIMC}. This sample complexity bound makes an excellent use of the side information in the sense that the bound is \textit{independent of the size of the matrix}, but further refining  the dependence on the \textit{dimensions of the side information}, which is relevant in our case,  was not one of the aims of that paper.  In the case where the side information is composed of identity functions, this result is clearly vacuous, contrary to our own, which instead then scales as the state-of-the-art bounds for matrix completion (from which it was derived).

\subsection{Experimental verification of the bounds under uniform marginals}
\label{sec:experiments_uniform}
In this section, we aim to experimentally validate our generalisation bounds. We focus on the case of uniform marginals for the following reasons:
even without side information, for the case of traditional matrix completion for a square $n\times n$ matrix, it is not clear in what sense the bound $O(n^{3/2}\sqrt{r})$ is tight. Indeed, it certainly is tight for a full rank matrix, but this is not very informative, since in that case, it simply says that the required number of entries is of the same order as the number of entries in the matrix.  The case of a lower rank constraint is less clear. This makes it difficult to design a sampling regime and a way to evaluate the bound.

To verify and experimentally explore the behaviour of our bound which applies to the case of uniform marginals, we construct some ground truth matrices with $C_{1,2}=C_{2,1}=0$, $C_{1,1}=1$, $a=b$, $m=n$. Set $r_{1,1}=a$ and we vary the values of $a$, $C=C_{2,2}$ and $r_{2,2}$. 

\textit{Ground truth matrix generation} In each case, the matrix $A:=R^{(1,1)}=\frac{aM^{(1,1)}}{m}$ is constructed with iid Rademacher entries, i.e., the entries of the community$\times$community component of the matrix are iid Rademacher variables. For each rank $r=r_{2,2}$, we construct a matrix $B$ as follows: generate $r$ iid column vectors $V$ (resp. $W$) in $\mathbb{R}^{m-a}$ (resp. $\mathbb{R}^{n-b}$) whose entries are $N(0,1/(m-a))$ (resp. $N(0,1/(n-b))$). We then form the matrix $\tilde{B}:=X^{(2)}VW^\top (Y^{(2)})^\top$, and finally obtain the matrix $B$ by dividing by the maximum entry in absolute value: $B=\tilde{B}/\{\max_{i,j}|\tilde{B}|\}$. The final matrix is constructed as $A+CB$. 

\textit{Sampling regime:}
we evaluate two sampling regimes with uniform marginals: 
\begin{itemize}
	\item Uniform sampling
	\item Checkerboard sampling: uniformly sample from entries $(i,j)$ with $i = j \mod 2$. 
\end{itemize}

\textit{Training procedure:} we train the bias-free part of OMIC+ via our algorithm. To isolate variables and for simplicity, $\lambda_{1,2}$ and $\lambda_{2,1}$ are both set to $\infty$.
To evaluate the performance of our algorithm on many sparsity regimes, we pick an ordering $O=\{O_1,\ldots,O_{(mn)!}\}$ (or $O=\{O_1,\ldots,O_{(mn/4)!}\}$ if in the other sampling regime) of all entries (here  $O_1\subset \ldots\subset O_{(mn)!}=\{1,2,\ldots m\}\times \{1,2,\ldots,n\}$ are increasing subsets of the set of entries. We train our algorithm for $\Omega$ taking each value in $O$, using the previous value as a warm start to dramatically reduce computation
For simplicity, hyperparameters $\lambda_{1,1}$ and $\lambda_{2,2}$ are determined through cross validation \textit{on a fixed, reasonably performing sparsity regime} once and for all. 

We set a tolerance threshold $\epsilon=0.1$ based on convergence analysis, and for each configuration of $a,r,C$, we compute the minimum $N_{\epsilon}$ such that the algorithm achieves RMSE $\leq \epsilon$ for the set $O_{N_{\epsilon}}$. 

Figure~\ref{fig:bound} is a graph of $N_{\epsilon}$ versus the bound in~\ref{corspecific2} for various values of $a,C,r,n$. We set $\epsilon=0.1$, $a=b$. We explored the following values: $a \in \{2,3,\cdots,8\}$
$r \in \{2,3,\cdots,8\}$
$C \in [0.5,2]$
$m \in \{100,101,\cdots,400\}$. For each datapoint, we choose a random combination of the above parameters, generate a random matrix accordingly, and compare the $N_{\epsilon}$ to our bound. 
\begin{centering}
	\begin{figure}
		\centering
		\includegraphics[width=0.97 \linewidth]{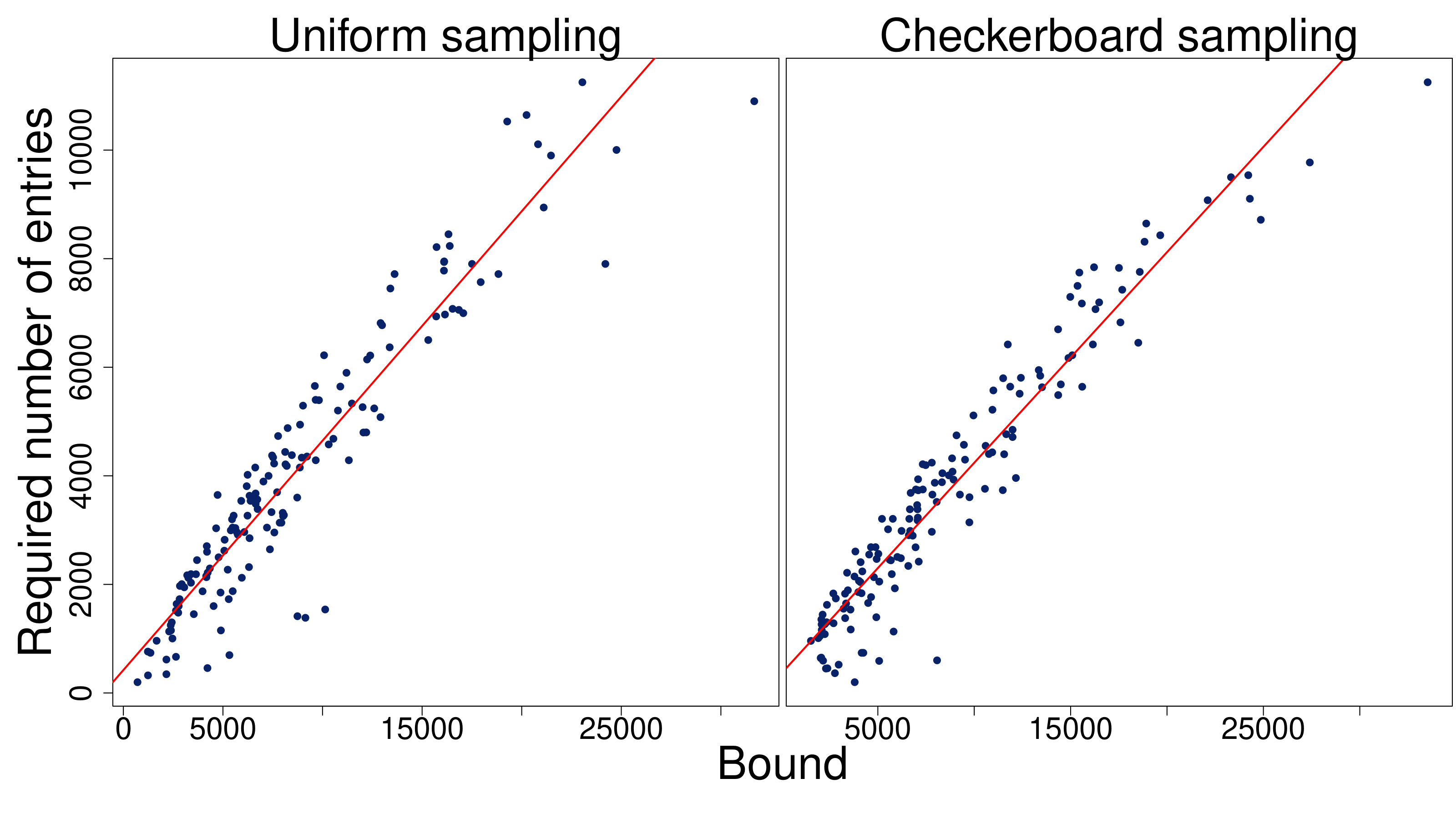}
		\caption{Comparison of our bound in~\ref{corspecific2} and the observed required number of entries to reach $0.1$ validation RMSE. The red lines are obtained through linear regression.}
		\label{fig:bound}
	\end{figure}
\end{centering}
We observe a very good match between the bound and the observed de facto sample complexity.

\section{Miscellaneous Lemmas}
Recall the definition of the Rademacher complexity of a function class $\mathcal{F}$:
\begin{definition}
	Let $\mathcal{F}$ be a class of real-valued functions with range $X$. Let also $S=(x_1,x_2,\ldots,x_n)\subset X$ be $n$ samples from the domain of the functions in $\mathcal{F}$. The empirical Rademacher complexity $\rad_S(\mathcal{F})$ of $\mathcal{F}$ with respect to $x_1,x_2,\ldots,x_n$ is defined by
	\begin{align}
	\rad_S(\mathcal{F}):=\mathbb{E}_{\delta}\sup_{f\in\mathcal{F}}\frac{1}{n}\sum_{i=1}^n  \delta_if(\mathbf{x}_i),
	\end{align}
	where $\delta=(\delta_1,\delta_2,\ldots,\delta_n)\in\{\pm 1\}^n$ is a set of $n$ i.i.d. Rademacher random variables (which take values $1$ or $-1$ with probability $0.5$ each).
\end{definition}
Recall the following classic theorem~\cite{rademach,contraction,Bartlettmend}:
\begin{theorem}
	\label{rademachh}
	Let $Z,Z_1,\ldots,Z_n$ be i.i.d. random variables taking values in a set $\mathcal{Z}$, and let $a<b$. Consider a set of functions $\mathcal{F}\in[a,b]^{\mathcal{Z}}$. $\forall \delta>0$, we have with probability $\geq 1-\delta$ over the draw of the sample $S$ that \begin{align}
	\label{almostthere}
	&\forall f \in \mathcal{F}, \quad \mathbb{E}(f(Z)) \\&\leq \frac{1}{n}\sum_{i=1}^nf(z_i)+2\mathbb{E}(\rad_{S}(\mathcal{F}))+(b-a)\sqrt{\frac{\log(2/\delta)}{2n}}. \nonumber 
	\end{align}
\end{theorem}

With a simple extra integration argument, we obtain the following version in expectation: 
\begin{theorem}
	\label{rademachhexp}
	Let $Z,Z_1,\ldots,Z_n$ be i.i.d. random variables taking values in a set $\mathcal{Z}$, and let $a<b$. Consider a set of functions $\mathcal{F}\in[a,b]^{\mathcal{Z}}$. $\forall \delta>0$, we have in expectation over the draw of the sample $S$ that \begin{align}
	&\inf_{f \in \mathcal{F}}  \left (\mathbb{E}(f(Z)) -\frac{1}{n}\sum_{i=1}^nf(z_i)\right) \nonumber \\&\leq 2\mathbb{E}(\rad_{S}(\mathcal{F}))+3(b-a)\sqrt{\frac{1}{n}}.
	\end{align}
\end{theorem}

\begin{proof}
	Let $X=\inf_{f \in \mathcal{F}}  \left (\mathbb{E}(f(Z)) -\frac{1}{n}\sum_{i=1}^nf(z_i)\right)-2\mathbb{E}(\rad_{S}(\mathcal{F}))$. Let us also write $\phi(\delta)$ for $(b-a)\sqrt{\frac{\log(2/\delta)}{2n}}$. By~\eqref{almostthere}, we have \begin{align}
	\mathbb{P}\left(X \geq \phi(\delta)\right)\leq \delta
	\end{align}
	For all $i\geq 1$, let us write $A_i$ for the event $\{X \leq \phi(\delta_i)\}$ where $\delta_i=2^{-i}$. 
	Let us also write $\widetilde{A}_1:=A_1$ and for $i\geq 2$, $\widetilde{A}_i:=A_i\setminus A_{i-1}$. We have, for $i\geq 2$, $\mathbb{P}(\widetilde{A}_i)\leq \mathbb{P}(A_{i-1}^c)\leq \delta_{i-1}$, and for $i=1$, $\mathbb{P}(\widetilde{A}_1)\leq 1= \delta_{i-1}$.
	
	Thus, 
	\begin{align*}
	\E(X)&=\sum_{i=1}^\infty \E(X|\widetilde{A}_i)\P(\widetilde{A}_i)\\
	&\leq \sum_{i=1}^\infty \E(X|\widetilde{A}_i)\delta_{i-1}\\
	&=\sum_{i=1}^\infty \phi(\delta_{i})\frac{1}{2^{i-1}}\\
	&=\sum_{i=1}^\infty (b-a)\sqrt{\frac{1+i}{2n}}\frac{1}{2^{i-1}}\\&\leq (b-a)\sqrt{\frac{2}{n}}\sum_{i=1}^\infty  2^{-i/2}\\
	&\leq (b-a)\sqrt{\frac{2}{n}} \frac{1/\sqrt{2}}{1-1/\sqrt{2}}\\
	&=  (b-a)\sqrt{\frac{2}{n}} \frac{1}{\sqrt{2}-1}\\
	&\leq 5(b-a)\sqrt{\frac{1}{n}},
	\end{align*}
	as expected.

\end{proof}

\begin{proposition}[Cf.~\cite{SimplerMC}]
	\label{bernstein}
	Let $X_1,\ldots,X_S$ be independent, zero mean random matrices of dimension $m\times n$. For all $k$, assume $\|X_k\|\leq M$ almost surely, and denote $\rho_k^2=\max(\|\E(X_kX_k^\top) \|,\|\E(X^\top_kX_k) \|)$. For any $\tau>0$, 
	\begin{align}
	&\mathbb{P}\left(\left \|\sum_{k=1}^S X_k \right \|\geq \tau \right)\nonumber \\& \leq (m+n)\exp\left(  -\frac     {\tau^2/2}     {\sum_{k=1}^S\rho_k^2 +M\tau/3}      \right).
	\end{align}
	
\end{proposition}

\begin{proposition}
	\label{bernsteinexp}
	Under the assumptions of Proposition~\ref{bernstein},  writing $\sigma^2=\sum_{k=1}^S\rho_k^2$, we have 
	\begin{align}
	\E\left( \left \|\sum_{k=1}^S X_k \right \| \right)&\leq \sqrt{8/3}\sigma (1+\sqrt{\log(m+n)})\nonumber \\&+ \frac{8M}{3}(1+\log(m+n)).
	\end{align}
\end{proposition}
\begin{proof}
	The result in O notation is an exercise from~\cite{bookhighprob}, and a similar result is also mentioned in both~\cite{foygel2011learning} and~\cite{tropp}. 
	
	For completeness and to get the exact constants, we include a proof as follows. 
	
	Let $Y=\left  \|\sum_{k=1}^S X_k \right \|$. By Proposition~\ref{bernstein}, splitting into two cases depending on whether $\tau M\leq \sigma^2$ or $\tau M\geq \sigma^2$ we have 
	\begin{align}
	\label{thesource}
	\mathbb{P}(Y\geq \tau)\leq \min\left(1,(m+n)\exp\left[-\frac{3\tau^2}{8\sigma^2}\right]\right)\nonumber \\+\min\left(1,(m+n)\exp\left[-\frac{3\tau}{8M}\right]\right).
	\end{align}
	
	Now note that writing $\kappa$ for $\log(m+n)8M/3$, we have
	\begin{align}
	\label{easy1}
	&\int_{0}^\infty 1\land (m+n)\exp\left(-\frac{3\tau}{8M}\right) d     \tau \nonumber \\&\leq \int_{0}^\kappa  1\land (m+n)\exp\left(-\frac{3\tau}{8M}\right) d     \tau \nonumber\\& \quad \quad \quad + \int_{\kappa}^\infty (m+n)\exp\left(-\frac{3\tau}{8M}\right)  d     \tau\nonumber\\
	&\leq \kappa +\left[\frac{-8M}{3} (m+n)\exp\left(-\frac{3\tau}{8M}\right)   \right]^{\infty}_{\kappa}\nonumber\\
	&=\kappa+  \frac{8M(m+n)}{3} \exp\left(-\frac{3\kappa}{8M}\right) \nonumber\\
	&=\kappa + \frac{8M(m+n)}{3}=\frac{8M}{3}(1+\log(m+n)).
	\end{align}

	We also have, writing $\psi$ for $\sigma \sqrt{\log(m+n)8/3}$, 
	
	\begin{align}
	\label{hard1}
	&\int_{0}^\infty 1\land (m+n)\exp\left(-\frac{3\tau^2}{8\sigma^2}\right) d\tau \nonumber\\&\leq  \int_{0}^\psi  1d\tau+\int_{\psi}^\infty  (m+n)\exp\left(-\frac{3\tau^2}{8\sigma^2}\right) d\tau\nonumber\\
	&\leq \psi  + \int_{\psi}^\infty \exp\left(-\frac{3(\tau^2-\psi^2)}{8\sigma^2}\right) d\tau  \nonumber\\
	&\leq \psi  + \int_{\psi}^\infty \exp\left(-\frac{3(\tau-\psi)^2}{8\sigma^2}\right) d\tau \nonumber \\
	&\leq \psi +\sigma \sqrt{2\pi/3}\nonumber\\&=\sigma \left[ \sqrt{\log(m+n)8/3} +\sqrt{2\pi/3} \right]\nonumber\\&\leq \sqrt{8/3}\sigma (1+\sqrt{\log(m+n)}). 
	\end{align}
	
	Plugging inequalities~\eqref{easy1} and~\eqref{hard1} into equation~\eqref{thesource}, we obtain: 
	\begin{align}
	\mathbb{E}(Y)&\leq \int_{0}^\infty \mathbb{P}(Y\geq \tau)d\tau \nonumber \\
	&\leq \sqrt{8/3}\sigma (1+\sqrt{\log(m+n)})\\ \nonumber& \quad \quad \quad \quad + \frac{8M}{3}(1+\log(m+n)),
	\end{align}
	as expected.

\end{proof}

\section{Details of the matrix generation procedure for the synthetic data experiments}

Our generation procedure can be described as follows: let $\tilde{a}$ be the vector with components $\tilde{a}_{i}=i-\frac{m+1}{2}$ and let $\tilde{b}$ be the vector with components $\tilde{a}_{j}=i-\frac{n+1}{2}$. Let $a=\frac{\tilde{a}}{\|\tilde{a}\|}$ and $b=\frac{\tilde{b}}{\|\tilde{b}\|}$.  We also write $v_1\in \mathbb{R}^m$ for $(\frac{1}{\sqrt{m}},\frac{1}{\sqrt{m}},\ldots,\frac{1}{\sqrt{m}})^\top$ and $v_2\in \mathbb{R}^n$ for $(\frac{1}{\sqrt{n}},\frac{1}{\sqrt{n}},\ldots,\frac{1}{\sqrt{n}})^\top$ Then we define  $G=\frac{1}{2}av_2^\top +\frac{1}{2}v_1 b^\top$ and $S \in \R^{m \times n}$ where $S_{i,j} = (1 / mn)$, if $(i,j) \in \{1,\cdots,m/2\} \times \{1, \cdots,n/2\} \cup \{m/2+1,\cdots,m\} \times \{n/2+1,\cdots,n\}$, and $S_{i,j} = -(1 / mn)$ otherwise. Therefore, we can generate a matrix $R \in \R^{m \times n}$ as 
\begin{align}
\label{genMAt}
R(\alpha) = \alpha cG + (1-\alpha) cS,
\end{align}
where $\alpha \in [0,1]$ is a parameter that controls the relative intensities of the user/item biases and the non-inductive component, and $c$ is a scaling constant. Note that $G$ is composed of the sum of two terms. The first term is a matrix with all rows being equal, whilst the second term's columns are all equal. Thus $G$ is made up of user and item biases. On the other hand, the $S$ matrix can be divided in four blocks of equal sizes. The top left and bottom right blocks entries have a constant value of $(1/mn)$. The remaining block has entries with the value $-(1/mn)$.

To perform the experiments we needed to select the parameters $m, n, c$ and $\alpha$. We chose\footnote{For a small number of incoherent eigenvectors, which is the situation in the case described here, the choice $c=\sqrt{mn}$ ensure entries of size close to one.} $m=n=c=100$. The parameter $\alpha \in \{0, 0.25, 0.5, 0.75, 1\}$ was empirically selected in such a way that the expected intensity of the biases' component varied. Note that in the extremes of the $\alpha$ interval the generated matrix is just composed of one of the components.

To determine the number of observed entries and the sampling distribution, we considered two extra parameters: the percentage of observed entries $p_{\Omega}$ and a parameter $\gamma \in \mathbb{N}$ that manages the sparsity distribution. Given a fixed $p_{\Omega}$ we randomly selected $\gamma(p_{\Omega}mn/(\gamma+1))$ entries in the first $m/2$ rows and $(p_{\Omega}mn/(\gamma+1))$ in the remaining ones. The parameter $p_{\Omega}$ was varied in $\{0.15,0.30,0.50\}$  and the parameter $\gamma$ as varied in the range $\{1,4\}$ ($\gamma=1$ indicates uniformly sampled observations).

Note that for SoftImpute we need an extra post-processing step to estimate the biases. In this case, we calculate the matrix bias as the average of the SI-predicted matrix entries. After subtracting the SI matrix bias, we calculated the users and the items bias by averaging the columns and the rows, respectively.

\section{In-depth literature review}
A major step signaling the beginning of the construction of a formal theory of matrix completion was the introduction of the SoftImpute algorithm~\cite{softimpute}, which uses the nuclear norm as a regularizer. 
Around the same time, the field witnessed a series of breakthroughs in the study of how many entries are required to recover a low-rank matrix exactly~\cite{genius,CandesRecht} or approximately from noisy entries~\cite{noisy,Kolchinski}. Those works assume that the entries of the matrix are sampled uniformly. A simpler and more complete approach to the same results was provided in both~\cite{SimplerMC} and~\cite{physics}. The conclusion of the works on exact recovery is that if the entries are sampled uniformly, it is possible to recover the matrix with high probability assuming $O(\mu r n \log(n)^2)$ entries are observed, where $n$ is the dimension of the matrix, and $\mu$ is some notion of \textit{coherence}, which is $O(1)$ if the singular vectors have roughly equal components. 

Of course, there is a huge branch of literature focusing on the optimization aspect of matrix completion~\cite{optimization1,optimization2,optimization3,optimization4,optimization5,optimization6,optimizationCore,optimizationBest}.

In~\cite{netflixwin}, user biases were trained jointly with other methods, including methods taking time dependence into account, but no nuclear-norm regularization was used. 

Other works~\cite{wain,completingprovably,nonuniform,nonunif} have focused on the case of non uniformly sampled entries. The general form of the results obtained is (similarly to the uniform case) that $O(r n \log(n)^2)$ observed entries are sufficient. However, these results come at the cost of either making strong explicit assumptions on the distributions, sometimes with relevant constants showing up in the bounds, or strong modifications of the algorithm. 

The case of non-uniform entries with absolutely no assumption on the sampling distribution is an interesting one that commands a completely different approach. It was studied in~\cite{ReallyUniform1,ReallyUniform2}.
The most related work to ours is~\cite{mostrelated}, where the authors study, and provide generalization bounds for a model composed of a sum of an IMC term and a standard SoftImpute model. This model is a particular case of ours. Note  we require to adapt proofs to obtain bounds with a tighter dependence on the dimensions of both left and right side information for the bounds to be non trivial in case of user biases. Furthermore, no notion of interpretability or orthogonality was presented in~\cite{mostrelated}.

Inductive matrix completion~\cite{IMC,IMC1,IMC2,IMC3} is the problem of solving matrix completion with some side information: given some features $X\in \mathbb{R}^{m\times d_1}$ and $Y\in \mathbb{R}^{n \times d_2}$, it tries to find a low-rank matrix $M$ such that $R=XMY^{\top}$ approximates the observed matrix well. It has found many successful applications in recent years~\cite{IMCforDrug,IMCGenes,blogrec}. Theoretical guarantees were provided in~\cite{optimization2,IMCtheory1,IMCtheory2}. 
Note that in the basic model, successful IMC requires that the columns of $X$ (resp. $Y$) span the left (resp. right) singular vectors of the SVD of the ground truth matrix (this case is often referred to as "perfect" side information). In~\cite{mostrelated} the extended model $R=XMY^{\top}+N$, with nuclear-norm regularization applied to both $M$ and $N$ was proposed. 
Recently, progress was made in the direction of matrix completion with side information with the need to extract features jointly~\cite{ProvableIMC}.

The idea of nuclear-norm minimization was also extended to tensors in various ways~\cite{tensor,tensor2,Tensor3} as there is no unique approach which provides all the benefits the nuclear norm enjoys in the matrix case.

\subsection{Models with similarities to ours}
\label{sec:modelswithsim}
In this section, we explain the particulars of some recently proposed models which may be understood as using combinations of side information and generic low rank constraints, or other variations of parts of the ideas we propose. We note that in each case, substantial differences remain between the works in question and the present paper.

In~\cite{NICE}, the authors introduce a model with similarities (and differences) to both~\cite{mostrelated} and the present work: assuming one is given side information matrices $X$ and $Y$, the authors present the following model: first, the matrices $X$ and $Y$ are augmented by a columns of ones resulting in the matrices $\bar{X}=[X,\textbf{1}]$ and $\bar{Y}=[Y,\textbf{1}]$. Predictors then take the form $E=\bar{X}M(\bar{Y})^\top+\Delta$, with nuclear norm regularisation imposed on $E$ and $L^1$ (or nuclear norm) regularisation on $M$, and Frobenius norm regularization imposed on $\Delta$, with the constraint that $P_{\Omega}(E)=R_\Omega$ where $R_\Omega$ denotes the observed entries. Thus, similarly to~\cite{mostrelated}, the authors allow hyperparameters to decide how much to trust the side information by employing nuclear norm regularisation on \textit{both} the $M$ inside the IMC term \textit{and} nuclear norm regularisation on the whole predictor. However, the strategy remains different from~\cite{mostrelated}, which applies nuclear norm regularisation to a residual term instead. The authors prove generalisation bounds in the uniform sampling regime which scale as the bounds in~\cite{IMCtheory1} in the case where the corresponding realizability assumptions are satisfied and scale as the state-of-the-art bounds in~\cite{ReallyUniform1,mostrelated} in the case where the side information is not good and the model must rely purely on the generic nuclear norm regularisation. A similarity between that model and ours is the augmentation of the matrices $X$ and $Y$ by a vector of ones (which also modifies the corresponding realizability assumptions). As such the proposed predictors involve a combination of lower order terms (which depend on one of the side information terms but not the others: $xv^\top+yw^\top$ where $x,y$ are the side information vectors and $u,w$ are trainable parameters) and higher order terms ($xMy^\top$ where $M$ is trainable). Thus if the side information is categorical, the model predictors will be similar to the predictors in our model OMIC+. However, the extra interpretability and optimization benefits we reap from the cross-term orthogonality constraints are absent in that work. Further, the optimization problem remains different and the generalization bounds do not explicitly exploit community structure. Contrary to our models BOMIC and BOMIC+ (but similarly to our model OMIC+, to which it is more closely related), there are also actually no user or item biases in that work (unless the matrices $X$ and $Y$ include identity matrices as submatrices), although there is a matrix-wise bias.

In~\cite{Rev2.2}, the authors solve an explicit rank minimization problem under linear constraints on the matrix (this problem is now commonly referred to as 'Matrix Regression' (see also~\cite{Rechtmatrixequations})). In that formulation, instead of observing entries of the matrix, iid measurements of the form $v^\top Mw$ are made where $w,v$ are Gaussian vectors. As such, the problem is related to both inductive matrix completion (because the $v$s and $w$s play a similar role as side information vectors in IMC) and classic matrix completion (because different observations never correspond to the same $v$ or $w$ and the dimensions of $v$ and $w$ are the same as the ground truth matrix). In this context and with an explicit low-rank assumption the authors show a sample complexity bound of $O(r^3n\log(n))$, which is in line with bounds for classic matrix completion under a uniform sampling assumption. A similar problem with nuclear norm regularisation was studied in~\cite{IMCtheory2}, together with its link to IMC. As explained in our Theory section, none of the bounds in either of those works give tight bounds for the community setting. 
Later in~\cite{Rev2.1}, an ingenious new algorithm was proposed, reducing the sample complexity to $O((m+n)r)$.

In~\cite{Rev2.3}, the authors propose a very general optimization framework that encompasses both the matrix regression problem mentioned above and low rank matrix completion, as well as one-bit matrix completion.

\subsection{Matrix completion with graph side information}
In~\cite{Koren1,Koren2}, the authors propose a model based on user biases combined with neighborhood-based models.  In~\cite{SocialReg1,SocialReg2,SocialReg3,SocialReg4,SocialReg5,SocialReg6}, the authors construct various low-rank matrix completion problems with regularizers inspired from the graph side information (typically, the feature vectors of adjacent nodes in the graph are regularised to be close to each other).  In~\cite{Combined}, the authors ingeniously combine this idea with user biases. Notably, in~\cite{TheoryReg}, some generalization guarantees were provided for such regularization strategies. 

\subsection{Community discovery}

In~\cite{BinaryNIPS}, the authors propose a probabilistic model to solve binary matrix completion with graph side information based on the assumption that the users form communities: the assumption is that each user's rating is a noisy measurement of the preference of the cluster. The clusters are recovered from the graph information via the Stochastic Block Model (SBM), and the cluster preferences are then recovered from the observed data. Generalisation bounds and an asymptotic analysis are provided for this model. 
In~\cite{Vincent} a similar model with further twists such as the existence of atypical users and items is introduced, and a thorough and impressive complexity and generalization analysis is performed.

In~\cite{extra1}, the authors consider the problem of simultaneously clustering users and items in an efficient way based on \textit{a single fully observed matrix}. In~\cite{extra2}, another similar matrix factorization approach was provided to cluster users based on side information. 

Those works rely heavily on the more general problem of community discovery, which is concerned with recovering "groups" or clusters of users given some side information such as a graph of interaction~\cite{sbm1,sbm2,sbm3,sbm4,sbm5,sbm6,sbm7,sbm8}. These models typically rely on the assumption that the graph we observe is generated under the Stochastic Block Model, i.e., each edge in the graph is present or absent with a given probability that depends (only) on the cluster assignments of the two relevant nodes. We refer the reader to~\cite{sbm1} for a survey.  

We note that the above approaches are crucially different from ours in that \emph{they do not allow for non random user-specific behaviour within each cluster}: in all of these works (except~\cite{Vincent}), the behaviour of users/items is an independent noisy measurement of cluster behaviour, whilst in our model, users exhibit their own behaviour on top of the cluster-specific behaviour. In particular, in the works above, there is no difference between predicting the matrix and predicting the clusters, whilst in our setting, we usually assume the clusters are given and recover the matrix from them. In that respect, our setting is more similar to the regularization-based techniques~\cite{SocialReg1,SocialReg2,SocialReg3,SocialReg4,SocialReg5,SocialReg6}, but our method is different. The paper~\cite{Vincent} is to the best of our knowledge the only work that incorporates item specific behavior in a community detection context. They do so in a discrete fashion with the concept of "atypical" movies and users, whilst our approach is a continuous one, which includes the possibility of representing any matrix (at a regularization cost).

We note in passing that a different approach to extracting community information from graphs is offered by graph neural networks~\cite{gcnn1,gcnn2,gcnn3,gcnn4,gcnn5}.

\subsection{On some variants of the matrix completion problem} In~\cite{NewTheoryMC}, the authors study a different problem, closely related to matrix completion, which assumes the existence of an exact dictionary, then introduce a much weaker condition than uniform sampling and incoherence, and show that typical optimization algorithms can recover the matrix. Before that, in~\cite{EarlierNewTheory}, similar results were shown under different assumptions.
In~\cite{Rechtmatrixequations}, the author gives recovery guarantees for the more general problem of linear matrix equations. 
In~\cite{multiview}, the authors study a multi-view model for image data where a common low-rank representation of several different views of the data points is constructed, to be later fed to a matrix completion algorithm. 

Recently, \cite{Transfer1} and \cite{Transfer2} studied a form of transfer learning problem, where the same users rank different media such as movies, music, series etc. In~\cite{chen2017learning}, the authors perform \textit{non negative matrix completion} with multiple sources of side information through a regularization-based approach different form the IMC setting.

Very recently, progress was made in the direction of establishing theoretical guarantees for matrix completion with "side information" where features are extracted jointly through a shallow network~\cite{ProvableIMC}.  This opens up an interesting new avenue of research for matrix completion methods such as ours, which could perhaps be further combined with these techniques in the future. We note that this preliminary work~\cite{ProvableIMC} only deals with the case of fixed rank constraints.

\newpage

\begin{table*}[h!]
	\centering
	\begin{tabular}{c|c}
		Notation & Meaning  \\
		\midrule
		$\mathcal{D}$ & Sampling distribution over entries\\
		$N$ & Sample size\\
		$\{(i_1,j_1),(i_2,j_2),\ldots,(i_N,j_N)\}$ &  Set of observed entries \\
		$P_{\Omega}$ & Projection on the set of observed entries\\
		\small ($\{(i_\alpha,j_\alpha):\alpha\leq N\}=\Omega$)  & \\
		$P_{\Omega^\top}$ & Projection on complement of set of observed entries\\
		$X^{(k)}\in \mathbb{R}^{m\times d_1^{(k)}}$ ($k\leq K$) & $k$'th left side information matrix\\
		$Y^{(l)}\in \mathbb{R}^{n\times d_2^{(l)}}$ ($l\leq L$) & $l$'th right side information matrix\\
		$d_1^{(k)}$ & Dimension of $k$'th left side information \\
		$d_2^{(l)}$ & Dimension of $l$'th right side information \\
		$M^{(k,l)}\in \mathbb{R}^{d_1^{(k)}\times d_2^{(l)}}$ &  Matrices to be trained in our model $f_{i,j}=\sum_{k,l=1}^{K,L} X^{(k)}M^{(k,l)}(Y^{(l)})^\top$\\
		$\mathcal{M}_{k,l}$ &  Upper bound on $\|M^{(k,l)}\|_{*}$\\
		$R_{\Omega}$ & Matrix of observed entries \\
		$\|\nbull\|_{*}$ & Nuclear norm\\
		$\|\nbull\|_{\sigma}$ & Spectral Norm\\
		$f:\{1,\ldots,m\}\times \{1,\ldots,n\}\rightarrow \mathbb{R}$ & Prediction function \\
		$f_M$ & Prediction function $f_{i,j}=\sum_{k,l=1}^{K,L} X^{(k)}M^{(k,l)}(Y^{(l)})^\top$\\
		& ($M=(M_1,\ldots,M_{K,L})$)\\
		$\mathcal{R}(f)$ & Expected risk $\mathbb{E}_{(i,j)\sim\mathcal{D}}(\ell(f(i,j),R_{i,j}+\delta_{i,j}))$\\
		$\hat{\mathcal{R}}(f)$ & Empirical risk $\frac{\sum_{(i,j)\in \Omega}\ell(f(i,j),R_{i,j}+\delta_{i,j}))}{\#(\Omega)}=\sum_{\alpha=1}^{N}\frac{\ell(f_{i_\alpha,j_\alpha},R_{i_\alpha,j_\alpha}+\delta_{i,j})}{N}$\\
		$\hat{f}$ &  Solution to optimisation problem~\eqref{optim}\\
		$\ell$ & Loss function\\
		$B$  & Upper bound on  value of the loss function $\ell$ \\
		$L_{\ell}$ & Upper bound on the Lipschitz constant of the loss function  $\ell$ \\
		$R\in \mathbb{R}^{m\times n}$ & Ground truth matrix (can be observed with noise)\\
		$\delta_{i,j}, \delta_\alpha$ (if $(i_\alpha,j_\alpha)=(i,j)$) & Sample from noise distribution of entry $R_{i,j}$, \\ &(follows distribution $\Delta_{i,j}$ independently at each draw )\\
		$R^{(k,l)}=(X^{(k)})^\top R Y^{(l)}$ & Core matrix in $(k,l)$ component of ground truth matrix\\
		$X^{(k)}R^{(k,l)}(Y^{(l)})^\top$ & $(k,l)$ component of ground truth matrix\\
		$C_{k,l}$ & Upper bound on entries of   $X^{(k)}R^{(k,l)}(Y^{(l)})^\top$ (or $R$)\\
		$r_{k,l}$ & Upper bound on rank of $X^{(k)}R^{(k,l)}(Y^{(l)})^\top$ (equivalently $R^{(k,l)}$)\\
		$C$ & Constant upper bound on the entries of the ground truth matrix $R$\\
		$\mathbf{x}^k_i$ & $i$th row of the matrix $X^{(k))}$\\
		$\mathbf{y}^k_i$ & $i$th row of the matrix $Y^{(k))}$\\
		$\mathcal{X}_k$ & $\max_{i=1}^m \|\mathbf{x}^k_i\|_{2}=\max_{i=1}^m \|X^{(k)}_{i,\nbull}\|_{2}$\\
		$\mathcal{Y}_l$ & $\max_{i=1}^n \|\mathbf{y}^l_i\|_{2}=\max_{i=1}^m \|Y^{(l)}_{i,\nbull}\|_{2}$\\
		$\mathcal{E}$ & Optimal risk $\mathbb{E}_{(i,j)\sim \mathcal{D}}(\ell(R_{i,j}+\delta_{i,j},R_{i,j}))$\\
		$\tilde{X}^{(k)}$ & Matrix whose columns come from completing $X^{(k)}$ into an orthonormal basis\\
		$\tilde{Y}^{(l)}$ & Matrix whose columns come from completing $Y^{(l)}$ into an orthonormal basis\\
		$S_{\lambda}$ & Singular value thresholding operator from earlier work\\
		$ \Lambda=(\lambda_{k,l})_{k\leq K; l\leq  L}$ & A set of hyperparameters \\
		$S_{\Lambda}$ & Generalised SVTO defined in~\eqref{SingThresh}
	\end{tabular}
	\caption{\small Table of notations for quick reference}
	\label{thiss}
\end{table*}

\ifCLASSOPTIONcaptionsoff
\newpage
\fi




\bibliographystyle{IEEEtran}

\bibliography{TNNLS-2020-P-15336.R2-bibliography}



%


%





\end{document}